\documentclass[journal]{IEEEtran}
\ifCLASSINFOpdf
\else
\fi
\usepackage{cite}
\usepackage{amsmath,amssymb,amsfonts,amsthm}
\usepackage{algorithmic}
\usepackage{algorithm}
\usepackage{graphicx}
\usepackage{subfigure}
\usepackage{textcomp}
\usepackage{makecell}
\usepackage{color}

\usepackage{mathtools}
\usepackage{mathrsfs}
\usepackage{braket}
\usepackage{arydshln}
\usepackage{indentfirst}
\usepackage{blkarray}
\usepackage{multirow}
\usepackage{bbding}
\usepackage{booktabs}
\usepackage{diagbox}
\usepackage{hyperref}
\hypersetup{hypertex=true,
            colorlinks=true,
            linkcolor=blue,
            anchorcolor=blue,
            citecolor=blue}

\allowdisplaybreaks[4]

\newtheorem{definition}{Definition}
\newtheorem{theorem}[definition]{Theorem}
\newtheorem{lemma}[definition]{Lemma}

\begin{document}

% \tableofcontents
\title{Toward Consistent and Efficient Map-based Visual-inertial Localization: Theory Framework and Filter Design}

\author{\IEEEauthorblockN{Zhuqing Zhang\IEEEauthorrefmark{1}, Yang Song\IEEEauthorrefmark{2}, Shoudong Huang\IEEEauthorrefmark{2}, Rong Xiong\IEEEauthorrefmark{1}, and Yue Wang\IEEEauthorrefmark{1}\\}
\IEEEauthorblockA{\IEEEauthorrefmark{1}State Key Laboratory of Industrial Control Technology, \\Zhejiang University, Hangzhou, Zhejiang, China.\\}
\IEEEauthorblockA{\IEEEauthorrefmark{2}Robotics Institute, Faculty
of Engineering and Information Technology,\\ University of Technology
Sydney, Ultimo, NSW 2007, Australia}
% <-this % stops an unwanted space
\thanks{Corresponding author: Yue Wang (email: wangyue@iipc.zju.edu.cn)}}

\maketitle

\begin{abstract}

This paper focuses on designing a consistent and efficient filter for map-based visual-inertial localization. First, we propose a new Lie group with its algebra, based on which a novel invariant extended Kalman filter (invariant EKF) is designed. We theoretically prove that, when we do not consider the uncertainty of the map information, the proposed invariant EKF can naturally maintain the correct observability properties of the system. To consider the uncertainty of the map information, we introduce a Schmidt filter. With the Schmidt filter, the uncertainty of the map information can be taken into consideration to avoid over-confident estimation while the computation cost only increases linearly with the size of the map keyframes. In addition, we introduce an easily implemented observability-constrained technique because directly combining the invariant EKF with the Schmidt filter cannot maintain the correct observability properties of the system that considers the uncertainty of the map information. Finally, we validate our proposed system's high consistency, accuracy, and efficiency via extensive simulations and real-world experiments.             
\end{abstract}


% \fi
\IEEEpeerreviewmaketitle

\section{Introduction}

\IEEEPARstart{L}{ocalization} is a basic module for intelligent robots. Recently, cameras and inertial measurement units (IMUs) have been widely used in localization for their lightweight and low cost, giving birth to visual-inertial odometry (VIO). Under the efforts of researchers from different groups, there already exist lots of mature VIO algorithms with properties of high accuracy and high efficiency\cite{msckf,openvins,s-msckf,ukf-msckf,vinsmono,rovio,SVO2,okvis}. However, VIO inevitably suffers from drift for a long-term running. In practice, it is more appreciated to perform localization with global information, so that on the one hand, the drift of the estimator can be bounded; on the other hand, further tasks like navigation can be implemented by acquiring the global position. A typical implementation of drift-free localization is employing a pre-built map. 

In general, a good map-based visual-inertial localization (VIL) algorithm should not only be error-bounded, but be consistent and computational efficient while fusing the pre-built map information. 
% Generally speaking, an estimator is consistent if the estimated error is zero mean and the estimated covariance is equal to or bigger than the true covariance \cite{oc}.
If a system is inconsistent, then the uncertainty of the estimated value is usually too optimistic, which will lead to unreliable estimation\cite{oc}. According to previous works, there are two main reasons for the inconsistency of the system: One is that the estimator is failed to maintain the correct observability properties of the system\cite{converg and consist,consistVIN}; another is that the uncertainty of the fused information (e.g., the pre-built map) is not properly used by the system\cite{consistent}. 

One theoretically sound method to consistently fuse a pre-built map with VIO is to regard the problem as an extension of visual-inertial simultaneous localization and mapping (VI-SLAM), which takes all measurements from both the local odometry session and the mapping session to formulate a batch optimization problem. Unfortunately, even if this scheme yields consistent localization, the ever-growing size of the state prohibits online computation. A popular solution makes a compromise to improve the computational efficiency by only maintaining a local odometry and the relative transformation between the odometry frame and the map frame, without considering the uncertainty of the map, i.e., assuming the map is \textit{perfect}\cite{maplab,vinsgps}. However, as mentioned in \cite{consistent}, \textit{perfect} map assumption will lead to overconfident estimation. In particular, Lee et al. \cite{gpsvio} point out that the formulation like \cite{maplab,vinsgps} may cause the observability-deficient issue, which leads to inconsistency. Therefore, they propose a solution by representing the local odometry state directly in the global frame, discarding the relative transformation between the local odometry and the global frames, which is only suitable for map agnostic gravity-aligned global measurements, e.g., global positioning system (GPS).

We consider a more general situation: 1) the pre-built map has uncertainties, i.e., \textit{imperfect map}; 2) the map-based measurements (global measurements) have uncertainties; 3) the transformation between the map frame and the inertial frame is 6 degress of freedom (DoF), i.e., gravity-unaligned. Under this formulation, the problem in \cite{gpsvio} can be regarded as a special case, where GPS acts like a \textit{perfect map} providing global measurements with uncertainies and the transformation between the map and the inertial frame is 4 DoF. We argue that this general situation is common in practice, where the maps are obtained indoors or via structure from motion (SFM)\cite{colmap}. How to consistently and efficiently fuse such general map information is a non-trivial problem, and the theoretical consistency analysis remains vague, naturally leaving the corresponding solution very limited.

% Nevertheless, for a scenario that employs a gravity-unaligned pre-built map with uncertainties (i.e., \textit{imperfect} map) to provide global measurements, the theoretical consistency analysis remains vague, naturally leaving the corresponding solution very limited.  

In this work, we aim to fill the gap by proposing observability-based consistency theories and developing an observability-constrained filter for consistent and efficient map-based VIL. Specifically, we employ a visual pre-built map, which contains map keyframe poses and map features positions with uncertainties. We regard the relative transformation that connects the local inertial frame and the map (global) frame as the \textit{augmented variable}, and the whole VIL system that maintains a local VIO state and the \textit{augmented variable} simultaneously as the \textit{augmented system}. Further, two kinds of \textit{augmented systems} are analyzed --- the \textit{perfect augmented system} that assumes the map is \textit{perfect} and the \textit{imperfect augmented system} that considers the uncertainty of the map. According to the theoretical analyses toward the observability properties of the \textit{(im)perfect augmented system}, we get the following conclusions: \textbf{a)} The ideal \textit{perfect augmented system} for our map-based VIL problem has four dimensions that are unobservable; \textbf{b)} The ideal \textit{imperfect augmented system} for our map-based VIL problem has ten (four plus six) dimensions that are unobservable; \textbf{c)} The ever-changing estimation of \textit{augmented variable} is the root that results in the deficiency of the original unobservable directions, only three unobservable directions of the \textit{(im)perfect augmented system} being preserved. 

Based on the above conclusions, we design a consistent and efficient map-based VIL. First of all, we design an invariant EKF based on our proposed novel Lie group and Lie algebra, which can naturally maintain the correct observability properties of the \textit{perfect augmented system}. Then, inspired by \cite{consistent, schmidtekf}, we introduce the Schmidt EKF to take the uncertainty of the map into consideration while keeping the computation at a low level. Moreover, to keep the correct observability properties of the \textit{imperfect augmented system}, we follow the idea of \cite{oc} to constrain the updating of the \textit{augmented variable}. Finally, a multi-state constraint is introduced to improve the accuracy of the localization. With our final multi-state observability constrained Schmidt invariant Kalman filter (MSOC-S-IKF), the system can correctly maintain the ten unobservable dimensions of the \textit{imperfect augmented system} while satisfying the real-time demand.

In summary, the contributions of this paper are as follows which have been verified by extensive simulations and real-world experiments:
\begin{itemize}
    \item Theoretically analyze the observability properties of the \textit{perfect/imperfect augmented system}, and ground these theories to physical explanations.
    \item Introduce a new Lie group and algebra, and propose a novel invariant EKF for the \textit{perfect augmented system} based on this new Lie group and algebra, which is proven to maintain the consistency of the \textit{perfect augmented system} naturally.
    \item Propose a consistent and efficient filter that combines the invariant EKF, Schmidt filter, multi-state constraint, and observability constraint, such that the system can consider the uncertainty of the map information, keep the cost of computation and storage at a low level, maintain the correct observability properties of the \textit{imperfect augmented system} and perform accurate localization.
\end{itemize}
\vspace{-2mm}

\section{Related works}
\subsection{Visual-inertial odometry}
Visual-inertial odometry (VIO) can be divided into filter-based and optimization-based in general. For the filter-based VIO, the IMU information is modeled in the propagation progress, while the relative motion calculated by the vision sensor is employed to update the state\cite{combine vi}. The famous frameworks include loosely-coupled multi-sensor fusion (MSF) \cite{msf}, which fuses the results from IMU and the visual odometry in an outer EKF framework. i.e., the estimations from the IMU are independent of those from the visual odometry. A more popular tightly-coupled framework is multi-state constrained Kalman filter (MSCKF)\cite{msckf,openvins,s-msckf,ukf-msckf}, which maintains a sliding window so that multiple states can be used to constrain the updating, and therefore results in more accurate estimation. Besides, benefiting from the multi-state constraint, MSCKF does not maintain feature points in the state vector, realizing a very efficient VIO. Another tightly-coupled filter-based VIO employs a robocentric formulation \cite{rovio}. And it also utilizes iterative EKF to update the state such that the nonlinear system can be relinearized to obtain accurate estimation. However, keeping features in the state and utilizing iterative EKF make it less efficient compared with MSCKF.        

For the optimization-based VIO, the measurements are formalized as a graph and optimized by iterative algorithms. OKVIS \cite{okvis} is a famous early implementation of the optimization-based VIO, where IMU measurements and visual measurements are tightly coupled in a factor graph. To reduce the optimization variables, preintegration \cite{preintegration} is proposed as a factor between the two keyframes. Following this framework, a popular VIO VINS-Mono \cite{vinsmono} is proposed and utilized in many aerial and ground moving vehicles. Forster et al. \cite{SVO2} propose a semi-direct-based VIO system, where both pixel intensities and features are used to perform robust and efficient estimation.

Compared with the filter-based VIO, the optimization-based VIO requires much more computation. Considering accuracy and efficiency, we use Open-VINS\cite{openvins}, a variation of MSCKF, as the backbone of our proposed algorithm. 
\vspace{-2mm}

\subsection{Drift-free visual-inertial localization}
Although the research on VIO has gained much progress, VIO inevitably suffers from drift. For the filter-based VIL, to bound the drift in VIO, GPS \cite{gpsvio} and map-based measurements \cite{consistent,maplab} are introduced using the augmented state formulation for fusion. These works make tradeoffs between the computation and the consideration of uncertainties of the map. One challenge is to design a filter that can be consistent and efficient simultaneously. In \cite{consistent}, the Schmidt filter is introduced to consider the uncertainty of the map. However, the system maintains all the map features, whose size (denoted as $n$) can be tens of thousands, in the state vector, not suitable for large scenes. In this paper, we propose to only maintain the map keyframe poses in the state, whose size (denoted as $m$) is much smaller than $n$, improving the computational efficiency. Moreover, as mentioned in \cite{converg and consist, gpsvio}, the system with the formulation like \cite{consistent}, i.e. the \textit{augmented system}, will suffer from inconsistency due to the linearization toward nonlinear systems, which is not addressed in \cite{consistent}.

For the optimization-based VIL, to reduce the drift of the VIO, global measurements can also be integrated into the system. In VINS-Fusion\cite{vinsmono,vinsgps,vinslocal,vinscalib}, which is the extension of VINS-Mono \cite{vinsmono}, map-based measurements are fused in an outer-loop pose-graph optimization, i.e., loosely-coupled. In \cite{tcgp}, a tightly-coupled approach is proposed to optimize all measurements together with respect to all states. When both odometry and map-based measurements are integrated, a degeneracy analysis is performed in \cite{degeneration}. In general, optimization-based VIL methods have superior accuracy than filtering-based methods because the nonlinear system can be re-linearized with a cost of more expensive resources. However, the uncertainties of the map are ignored in all these methods, which will make the estimation unrealiable.
\vspace{-2mm}

\subsection{Observability analysis and consistency maintaining}\label{sec:related works observability}
The inconsistency problem of point feature based EKF SLAM is firstly demonstrated in \cite{first consist}. Its authors find that when a stationary robot observes a new feature multiple times, the uncertainty along the robot's orientation will decrease, which is against our intuition since the observation toward a new feature should not provide information about the ego state of the robot. A further analysis of inconsistency problem is discussed in \cite{converg and consist}, where the authors also consider the case that a robot observes a feature from two positions. In \cite{fej}, the authors firstly theoretically analyze the observability properties of the EKF-SLAM and argue that due to the linearization toward the nonlinear system, the unobservable directions become observable. Based on this analysis, first-estimated Jacobian (FEJ) is proposed to maintain the correct observability properties of the system, which turns out to be effective for improving the system's consistency. Furthermore, in their follow-up work, an observability constrained method is introduced \cite{oc}. Different from designing constraints to maintain the correct observability, a series of manifold-based filters (invariant filters) are proposed\cite{ukf-msckf,consistent1,c1,c2,c3,c4}. For example, in \cite{consistent1}, the authors formulate the state on a special Lie group and define nonlinear errors through Lie algebra instead of using standard errors in EKF. With this kind of formulation, the proposed invariant EKF can automatically maintain the correct observability properties of the system. 

The observability analyses and the methods for improving the consistency of the system mentioned above have already been extended and applied into visual-inertial localization systems \cite{consistent2,consistent3,consistent4}. Unfortunately, since map-based measurements are introduced for updating via the augmented state formulation, these works on observability analyses of visual-inertial odometry cannot reflect the observability of the augmented state. In \cite{gpsvio}, the observability for visual-inertial odometry fusing GPS is analyzed, which shows the insights for the observability of the augmented system. \textcolor{black}{However, it only points out that the augmented system suffers from inconsistency and proposes a consistent estimation algorithm that the state is estimated in a gravity-aligned map frame, i.e., the map has 4 DoF. To be specific, after initializing the relative transformation between the VIO frame and the 4 DoF map frame, the variables in the state vector are transformed from the 4 DoF VIO frame to the 4 DoF map frame, and the state vector is propagated by IMU measurements in the 4 DoF map frame. Nevertheless, for a more general gravity-unaligned map (a 6 DoF map), the state cannot be correctly propagated by the IMU measurements in the map frame because the gravity in the map frame is unknown. In this paper, we propose a general framework that can be used for the case of 6 DoF maps and extend the observability analysis in \cite{gpsvio} to the case that the uncertainty of the map information is considered.}

% However, their solution relies on the gravity alignment between odometry frame and the GPS measurements, which, may not be true in map based measurements, leaving the consistency of augmented state in map based visual inertial localization still in vague.

\section{Preliminary and problem statement}
In this section, some preliminaries about Lie group, Lie algebra and invariant error are introduced first. Then, the problem description toward map-based visual-inertial localization system is provided, which could give the readers a general idea of the problem addressed in the following sections. 
\vspace{-2mm}

\subsection{Theoretical background and preliminaries}\label{sec:preliminaries}
\textbf{Matrix Lie group and Lie algebra\cite{state estimation}:} A matrix Lie group $\mathfrak{G} \subseteq \mathbb{R}^{N \times N}$ is a subset of invertible square matrices with the following properties:
\begin{equation}\label{eq:group property}
    \begin{aligned}
        &\mathbf{I} \in \mathfrak{G}\\
        &\forall \mathbf{X} \in \mathfrak{G}, \mathbf{X}^{-1} \in \mathfrak{G}\\
        &\forall \mathbf{X}_{1},\mathbf{X}_{2} \in \mathfrak{G}, \mathbf{X}_{1}\mathbf{X}_{2} \in \mathfrak{G},\\
    \end{aligned}
\end{equation}
where $\mathbf{I}$ is the identity matrix, $N$ is the order of the square matrices. 

Denote $\mathfrak{g} \cong \mathbb{R}^{D}$ be the Lie algebra of $\mathfrak{G}$. An exponential map $\mathrm{exp}: \mathbb{R}^{D} \xrightarrow{} \mathfrak{G}$ can be defined by
\begin{equation}\label{eq:exp}
    \mathrm{exp}(\boldsymbol{\xi}) = \mathrm{exp_{m}}(\mathfrak{L}_{\mathfrak{g}}(\boldsymbol{\xi})),
\end{equation}
where $\boldsymbol{\xi} \in \mathbb{R}^{D}$, $\mathrm{exp_{m}}$ is the standard matrix exponential map, and $\mathfrak{L}_{\mathfrak{g}}$ is the isomorphism of $\mathbb{R}^{D}$ and $\mathfrak{g}$. For simplicity, we denote $\mathfrak{L}_{\mathfrak{g}}(\cdot) \triangleq (\cdot)^{\curlywedge}$. Denote $\mathrm{log}$ be the inverse operation of $\mathrm{exp}$.

The usually used Lie group (and its algebra) includes special orthogonal group ${SO}(3)$($\mathfrak{so}(3)$), which represents rotation matrix; special Euclidean group ${SE}(3)$($\mathfrak{se}(3)$), which is the combination of rotation and translation; ${SE}_{2+K}(3)$($\mathfrak{se}_{2+K}(3)$), which is the extension of special Euclidean group and consists of one rotation matrix and $2+K$ vectors\cite{consistent1}.    

\textbf{Left and right invariant error:} Suppose $\mathbf{X}_{t},\hat{\mathbf{X}}_{t} \in \mathfrak{G}$ are the real and the estimated states at the time step $t$, respectively\footnote{Throughout this paper, the symbol with/without $\hat{\cdot}$ denotes the estimated/true value.}.
Then, we have the following definition:

\begin{definition}
\label{def:invariant error}
The left and right invariant errors between the two states are\cite{c2}:
\begin{equation}\label{eq:left_invariant}
    \boldsymbol{\eta}_{t}^{l} = \mathbf{X}_{t}^{-1}\hat{\mathbf{X}}_{t},\quad(\mathrm{Left}\ \mathrm{invariant}\ \mathrm{error})
\end{equation}
\begin{equation}\label{eq:right_invariant}
    \boldsymbol{\eta}_{t}^{r} = \hat{\mathbf{X}}_{t} \mathbf{X}_{t}^{-1},\quad(\mathrm{Right}\ \mathrm{invariant}\ \mathrm{error})
\end{equation}
\end{definition}
The reason why (\ref{eq:left_invariant})/(\ref{eq:right_invariant}) is called left/right invariant error is that when we left/right multiply an arbitrary element in $\mathfrak{G}$ to both trajectories $\mathbf{X}_{t},\hat{\mathbf{X}}_{t}$, the error is invariant. In the rest of the paper, right invariant error, which is widely used in the localization problem\cite{ukf-msckf,consistent1,c2,consistent3}, will be used to formulate the problem.
\vspace{-2mm}

\begin{figure}[t!]
    \centering
    \setlength{\abovecaptionskip}{0cm}
    \includegraphics[width=0.8\linewidth]{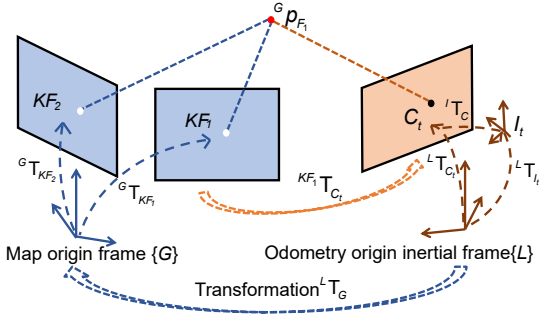}
     \caption{The different reference frames}
     \label{fig:coordinate}
     \vspace{-0.6cm}
\end{figure}

\subsection{Map-based visual-inertial localization problem}\label{sec: problem statement}
There are five kinds of reference frames in the map-based visual-inertial localization problem as shown in Fig. \ref{fig:coordinate}, i.e., $L$, $I_t$, $C_t$, $G$, and $KF$, where $L$, $I_t$, and $C_t$ are local VIO-related frames while $G$ and $KF$ are global map-related frames. The frames with the subscript $t$ mean they changes with time, whereas the frames without the subscript $t$ are fixed.

For the local VIO, it estimates the pose of current IMU body frame $I_t$ or camera frame $C_t$ in the fixed local inertial frame $L$ (normally, the frame $I_0$). The IMU outputs linear acceleration and angular velocity in ${I_t}$, and the camera observes features in ${C_t}$.

For the pre-built map, which can be built by visual SLAM or SFM algorithms, it contains the fixed known poses of map keyframes $\{KF_i\}$ and positions of map features $\{F_j\}$ in the fixed global map frame $G$.

It is worth noting that $L$ is an inertial frame where the IMU is used to propagate the state $^{L}\mathbf{T}_{I_t}$\footnote{In this paper, we use $\mathbf{T}$ to represent a pose or a transformation. $^{L}\mathbf{T}_{I_t}$ means the pose of frame $I_t$ in frame $L$, or transforming a state from frame $I_t$ to frame $L$.}. If we use the solution in \cite{gpsvio} to perform map-based localization, $G$ should also be an inertial (gravity-aligned) frame, such that IMU measurements can be utilized to propagate $^{G}\mathbf{T}_{I_t}$ in $G$. However, in practice, $G$ may not be an inertial frame. This encourages us to formulate the problem in a more general way by introducing the so-called \textit{augmented system}, which maintains the local pose $^{L}\mathbf{T}_{I_t}$ and the \textit{augmented variable} $^{L}\mathbf{T}_{G}$ simultaneously. In this way, without requiring $G$ to be an inertial frame, we can obtain $^{G}\mathbf{T}_{I_t}$ through $^{L}\mathbf{T}_{G}$ and $^{L}\mathbf{T}_{I_t}$. Specifically, when there are matches between the current view and the features in the pre-built map, the current view measurements of these features will bridge the frames ${G}$, ${KF}$, ${L}$ and ${C_t}$. With these measurements, we can not only estimate the \textit{augmented variable} $^{L}\mathbf{T}_{G}$, but also update the local pose of current body frame $^{L}\mathbf{T}_{I_t}$ to reduce the drift of the local VIO.

% \textcolor{black}{It is worth noting that as $L$ is an inertial frame where the IMU is used to propagate the state $^{L}\mathbf{T}_{I_t}$\footnote{In this paper, we use $\mathbf{T}$ to represent a pose or a transformation. $^{L}\mathbf{T}_{I_t}$ means the pose of frame $I_t$ in frame $L$, or transforming a state from frame $I_t$ to frame $L$.}, if we want to localize the body frame in the map frame $G$ like the problem formulation in \cite{gpsvio}, $G$ should also be an inertial (gravity aligned) frame such that IMU measurements can be utilized to propagate the pose of the body $^{G}\mathbf{T}_{I_t}$ in $G$. However, in practice, the map frame $G$ may not be an inertial frame, which encourages us to formulate the problem as estimating the local pose $^{L}\mathbf{T}_{I_t}$ and the \textit{augmented variable} $^{L}\mathbf{T}_{G}$ at the same time.}

Therefore, such \textit{augmented system} is required to \textit{consistenly} and \textit{efficiently} estimate $^{L}\mathbf{T}_{G}$ and $^{L}\mathbf{T}_{I_t}$, given a sequence of IMU measurements in frame $I_t$ and images in frame $C_t$, and pre-built map information consisting of keyframes $\{KF_i\}$ and features $\{F_j\}$ in frame $G$. 

% such that given a sequence of IMU measurements in frame $I_t$ and images in frame $C_t$, and pre-built map information consists of keyframes $\{KF_i\}$ and features $\{F_j\}$ in frame $G$, we can estimate the \textit{augmented variable} (i.e., $^{L}\mathbf{T}_{G}$) along with odometry-related variables (e.g., $^{L}\mathbf{T}_{I_t}$) \textit{consistently} and \textit{efficiently}. 
\vspace{-2mm}

\section{Invariant EKF for perfect map based visual-inertial localization}
\label{sec:IEKF perfect}
In this section, we assume the map components are perfect without uncertainty, i.e. \textit{perfect map}. With this assumption, we propose an invariant EKF maintaining the \textit{augmented variable} for the \textit{perfect map}. First, a new Lie group with its algebra is introduced to construct the kinematics of the \textit{augmented system}. Then, the state of the system, propagating procedure, and updating procedure is illustrated step by step. 
This kind of \textit{perfect map} based visual-inertial localization system is called the \textit{perfect augmented system}.
\vspace{-2mm}

\subsection{A novel Lie group and algebra} \label{sec:novel Lie group}

Although the invariant EKF based on $SE_{2+K}(3)$ is successfully applied to SLAM problem to achieve consistent estimator\cite{consistent1,contact,ukf_lg}, $SE_{2+K}(3)$ is not adaptable to the \textit{augmented system} mentioned above. This is because $SE_{2+K}(3)$ contains only one rotation matrix, whereas another rotation matrix in the \textit{augmented variable} also needs to be considered.

Therefore, we introduce a novel Lie group $SE_{2+K}^{M}(3)$ with its algebra $\mathfrak{se}_{2+K}^{M}(3)$.
\begin{definition}
\label{def:new Lie group}
With $SE_{2+K}(3)$ and $SO(3)$, a new Lie group denoted as $SE_{2+K}^{M}(3)$ is defined as,
\begin{equation} \label{eq:new Lie group}
    \begin{aligned}
         SE_{2+K}^{M}(3)&\triangleq\{\mathcal{T}=\begin{bmatrix}\mathcal{T}_{1}&\mathbf{0}_{(5+K+M)\times(3M)}\\
         \mathbf{0}_{(5+K+M)\times(3M)}^{\top}&\mathcal{T}_{2}\end{bmatrix},\\
         &\mathcal{T}_{1} \in SE_{2+K+M}(3), \mathcal{T}_{2}\triangleq \mathrm{diag}(\begin{array}{ccc}\mathbf{R}_{1},\!\!&\!\!\cdots\!\!&\!\!,\mathbf{R}_{M}
         \end{array}), \\ &\mathbf{R}_{i} \in SO(3), i=1,\cdots,M\},
        %  &\mathcal{T}_{1} \in SE_{2+K+M}(3), \mathcal{T}_{2}=\begin{bmatrix}\mathbf{R}_{1}&&\\&\ddots&\\&&\mathbf{R}_{M}\end{bmatrix}, \\ &\mathbf{R}_{i} \in SO(3), i=1,\cdots,M\}.
    \end{aligned}
\end{equation}
where $\mathrm{diag}(\cdot)$ represents a diagonal block matrix.

The corresponding Lie algebra $\mathfrak{se}_{2+K}^{M}(3)$ is defined by
\begin{equation}\nonumber 
    \begin{aligned}
    \mathfrak{se}_{2+K}^{M}(3)&\triangleq \{\mathfrak{L}_{\mathfrak{se}_{2+K}^{M}}(\boldsymbol{\phi}), \boldsymbol{\phi}=\left[\boldsymbol{\theta}_{0}^{\top}\ \boldsymbol{\phi}_{1}^{\top}\ \cdots \ \boldsymbol{\phi}_{2+K+M}^{\top}\right.\\&\left.\boldsymbol{\theta}_{1}^{\top}\ \cdots\ \boldsymbol{\theta}_{M}^{\top} \right]^{\top} \in \mathbb{R}^{9+3K+3M}\},\\
    \end{aligned}
\end{equation}
\begin{equation}\label{eq:new Lie algebra}
     \begin{aligned}
    &\mathfrak{L}_{\mathfrak{se}_{2+K}^{M}}(\boldsymbol{\phi})= \boldsymbol{\phi}^{\curlywedge}\triangleq\\ &\left[\begin{array}{c;{2pt/2pt}c;{2pt/2pt}c}
    (\boldsymbol{\theta}_{0})_{\times}&\boldsymbol{\phi}_{1}\ \cdots \ \boldsymbol{\phi}_{2+K+M}&\mathbf{0}_{3\times3M}\\
    \hdashline
    \mathbf{0}_{(2+K+M)\times3}&\mathbf{0}_{2+K+M}& \mathbf{0}_{(2+K+M)\times(3M)}\\
    \hdashline
     \mathbf{0}_{3M\times3}&\mathbf{0}_{3M\times(2+K+M)}&\boldsymbol{\Theta}\\
    \end{array}\right],\\
    &\qquad \qquad \qquad \boldsymbol{\Theta}\triangleq \mathrm{diag}(\begin{array}{ccc}(\boldsymbol{\theta}_{1})_{\times}
    &,\cdots,&(\boldsymbol{\theta}_{M})_{\times}\end{array}).
    \end{aligned}
\end{equation}
\end{definition}

It is clear that $SE_{2+K}^{M}(3)$ satisfies the properties in (\ref{eq:group property}). We can regard $SE_{2+K}^{M}(3)$ as the combination of one $SE_{2+K}(3)$ and $M$ $SO(3)\times \mathbb{R}^{3}$, and we insert the $M$ $\mathbb{R}^{3}$ into $SE_{2+K}(3)$ to get $SE_{2+K+M}(3)$ while the $M$ $SO(3)$ are arranged diagonally. As presented in the following part, $SE_{2+K}^{1}(3)$ can be used for the \textit{augmented system}.

\vspace{-2mm}

\subsection{State for the \textit{perfect augmented system}}\label{sec:state prefect}
 We define the state of the \textit{augmented system} at timestamp $t$ as $\mathbf{X}_{t} \in SE_{2+1}^{1}(3) \times \mathbb{R}^{6}$:
\begin{equation} \label{eq:state}
\begin{aligned}
     \mathbf{X}_{t}&=(\mathbf{X}_{A_t},\mathbf{B}_{t}),\;\;\;\;    \mathbf{B}_{t}=\begin{bmatrix}\mathbf{b}_{g_t}^{\top}&\mathbf{b}_{a_t}^{\top}\end{bmatrix}^{\top} \in \mathbb{R}^{6},\\
     \mathbf{X}_{A_t}&=\left[\begin{matrix}^{L}\mathbf{R}_{I_t}&^{L}\mathbf{v}_{I_t}&^{L}\mathbf{p}_{I_t}&^{L}\mathbf{p}_{f_t}&^{L}\mathbf{p}_{G_t}&\mathbf{0}_{3}\\\mathbf{0}_{1\times3}&1&0&0&0&\mathbf{0}_{3}\\\mathbf{0}_{1\times3}&0&1&0&0&\mathbf{0}_{3}\\\mathbf{0}_{1\times3}&0&0&1&0&\mathbf{0}_{3}\\
    \mathbf{0}_{1\times3}&0&0&0&1&\mathbf{0}_{3}\\ \mathbf{0}_{3}&\mathbf{0}_{3\times1}&\mathbf{0}_{3\times1}&\mathbf{0}_{3\times1}&\mathbf{0}_{3\times1}&^{L}\mathbf{R}_{G_t}\end{matrix}\right],
\end{aligned}
\end{equation}
where ${}^{L}\mathbf{R}_{I}$ represents the rotation of the body frame $I$ in $L$, $^{L}\mathbf{v}_{I}$ and $^{L}\mathbf{p}_{I}$ are the velocity and the position of the body in $L$, respectively, $^{L}\mathbf{p}_{f}$ is the position of a (local) feature in $L$, $^{L}\mathbf{p}_{G}$ and $^{L}\mathbf{R}_{G}$ are the translation and the rotation parts of the relative transformation between $G$ and $L$ (\textit{augmented variable}), and $\mathbf{b}_{g}$ and $\mathbf{b}_{a}$ represent the gyroscope and accelerometer bias. Note that we only employ one local feature to formulate the state for simplicity. It can be easily extended to multiple features (say $K$ features), and in this case, $\mathbf{X}_{A_t}$ will belong to $SE_{2+K}^{1}(3)$.

\textbf{Error definition:} With the state defined above, we formulate the error state as
\begin{equation}\label{eq:error}
    \mathbf{e}_t=(\hat{\mathbf{X}}_{A_t}\mathbf{X}_{A_t}^{-1},\hat{\mathbf{B}}_{t}-\mathbf{B}_{t}) \triangleq (\boldsymbol{\eta}_{A_t},\boldsymbol{\xi}_{B_t}).
\end{equation}
Note that $\boldsymbol{\eta}_{A_t} \in SE_{2+1}^{1}(3)$. Applying (\ref{eq:exp}) and (\ref{eq:new Lie algebra}) to $\boldsymbol{\eta}_{A_t}$, the following approximation holds up to the first order:
\begin{equation}\label{eq:approximate_error}
    \boldsymbol{\eta}_{A_t} = \mathrm{exp_{m}}(\mathfrak{L}_{\mathfrak{se}_{2+1}^{1}}(\boldsymbol{\xi}_{A_t})) \approx \mathbf{I}_{10} + \mathfrak{L}_{\mathfrak{se}_{2+1}^{1}}(\boldsymbol{\xi}_{A_t}),
\end{equation}
where
\begin{equation}
\boldsymbol{\xi}_{A_t}\triangleq \begin{bmatrix}\boldsymbol{\xi}_{\boldsymbol{\theta}_{LI_t}}^{\top}&\boldsymbol{\xi}_{\mathbf{v}_{LI_t}}^{\top}&\boldsymbol{\xi}_{\mathbf{p}_{LI_t}}^{\top}&\boldsymbol{\xi}_{\mathbf{p}_{Lf_t}}^{\top}&\boldsymbol{\xi}_{\mathbf{p}_{LG_t}}^{\top}&\boldsymbol{\xi}_{\boldsymbol{\theta}_{LG_t}}^{\top}\end{bmatrix}^{\top},    
\end{equation}
and
\begin{equation}\label{eq:nonlinear_error}
    \begin{aligned}
        \boldsymbol{\xi}_{\boldsymbol{\theta}_{LI_t}}&=\tilde{\boldsymbol{\theta}}_{LI_t},\\
        \boldsymbol{\xi}_{\mathbf{v}_{LI_t}}&={}^{L}\hat{\mathbf{v}}_{I_t}-(\mathbf{I}_{3}+(\tilde{\boldsymbol{\theta}}_{LI_t})_{\times}){}^{L}\mathbf{v}_{I_t},\\
         \boldsymbol{\xi}_{\mathbf{p}_{LI_t}}&={}^{L}\hat{\mathbf{p}}_{I_t}-(\mathbf{I}_{3}+(\tilde{\boldsymbol{\theta}}_{LI_t})_{\times}){}^{L}\mathbf{p}_{I_t},\\
         \boldsymbol{\xi}_{\mathbf{p}_{Lf_t}}&={}^{L}\hat{\mathbf{p}}_{f_t}-(\mathbf{I}_{3}+(\tilde{\boldsymbol{\theta}}_{LI_t})_{\times}){}^{L}\mathbf{p}_{f_t},\\
         \boldsymbol{\xi}_{\mathbf{p}_{LG_t}}&={}^{L}\hat{\mathbf{p}}_{G_t}-(\mathbf{I}_{3}+(\tilde{\boldsymbol{\theta}}_{LI_t})_{\times}){}^{L}\mathbf{p}_{G_t},\\
         \boldsymbol{\xi}_{\boldsymbol{\theta}_{LG_t}}&=\tilde{\boldsymbol{\theta}}_{LG_t},\\
         \boldsymbol{\xi}_{B_t}&=\begin{bmatrix}(\hat{\mathbf{b}}_{g_t}-\mathbf{b}_{g_t})^{\top}&(\hat{\mathbf{b}}_{a_t}-\mathbf{b}_{a_t})^{\top}\end{bmatrix}^{\top}.
    \end{aligned}
\end{equation}
In the above equations, $\tilde{\boldsymbol{\theta}}$ represents the error of the rotation matrix, which is defined as $\tilde{\boldsymbol{\theta}}=\mathrm{log}({\hat{\mathbf{R}}\mathbf{R}^{-1}})$, and $\boldsymbol{\eta}_{R}\triangleq{}^{L}\hat{\mathbf{R}}_{I}{}^{L}{\mathbf{R}}_{I}^{-1}= \mathbf{I}_{3}+(\tilde{\boldsymbol{\theta}}_{LI})_{\times}$ up to the first order. $\boldsymbol{\epsilon}_{t}\triangleq\begin{bmatrix}\boldsymbol{\xi}_{A_t}^{\top}&\boldsymbol{\xi}_{B_t}^{\top}\end{bmatrix}^{\top}$ defines the \textit{right invariant error} which is used to formulate the error propagation and update equations.
\vspace{-2mm}

\subsection{Propagation for the \textit{perfect augmented system}}\label{sec:prop perfect}
With the state defined by (\ref{eq:state}), the system kinematics can be given as:
\begin{equation} \label{eq:kinematics}
    \left\{\begin{aligned} ^{L}\dot{\mathbf{R}}_{I_t}&=^{L}\mathbf{R}_{I_t}(\boldsymbol{\omega}_{t}-\mathbf{b}_{g_t}-\mathbf{w}_{g_t})_{\times}\\
    ^{L}\dot{\mathbf{v}}_{I_t}&=^{L}\mathbf{R}_{I_t}(\mathbf{a}_{t}-\mathbf{b}_{a_t}-\mathbf{w}_{a_t})+\mathbf{g}\\
    ^{L}\dot{\mathbf{p}}_{I_t}&=^{L}\mathbf{v}_{I_t}\;\;\;\;^{L}\dot{\mathbf{p}}_{f_t}=\mathbf{0}_{3\times1}\\
    ^{L}\dot{\mathbf{p}}_{G_t}&=\mathbf{0}_{3\times1}\;\;^{L}\dot{\mathbf{R}}_{G_t}=\mathbf{0}_{3}\\
    \dot{\mathbf{b}}_{g_t}&=\mathbf{w}_{bg_t}\;\;\;\;\;\;\dot{\mathbf{b}}_{a_t}=\mathbf{w}_{ba_t}\\
    \end{aligned}
    \right.
\end{equation}
where $\boldsymbol{\omega}$ and $\mathbf{a}$ are the measurements of angular velocity and linear acceleration from the IMU, $\mathbf{w}_{g}$ and $\mathbf{w}_{a}$ are the IMU measurement noises, $\mathbf{w}_{bg}$ and $\mathbf{w}_{ba}$ are the Gaussian random walk noises, and $\mathbf{g}=\begin{bmatrix}
0&0&-9.8
\end{bmatrix}^{\top}$ is the gravitational acceleration in the local inertial frame $L$.

Different from the standard EKF error propagation function, we need to formulate the error propagation function with the \textit{right invariant error} defined by (\ref{eq:nonlinear_error}), which leads to
\begin{equation} \label{eq:linearized_propagation}
\begin{aligned}
     \frac{d}{dt}\boldsymbol{\epsilon}&\!\!=\!\!\frac{d}{dt}\begin{bmatrix}
    \boldsymbol{\xi}_{A_t}\\\boldsymbol{\xi}_{B_t}\end{bmatrix}
    \!\!\triangleq\!\!
    \left[\begin{matrix}\begin{array}{c;{2pt/2pt}c}
    \mathbf{A}_{\mathbf{X}_{A_t}}\!\!&\!\!\mathbf{A}_{\mathbf{X}_{A}\mathbf{B}_t}\\
    \hdashline
    \mathbf{0}_{6\times18}\!\!&\!\!\mathbf{A}_{\mathbf{B}_t}\end{array}
    \end{matrix}\right]\!\!\!
    \begin{bmatrix}
    \boldsymbol{\xi}_{A_t}\\\boldsymbol{\xi}_{B_t}
    \end{bmatrix}
    \!\!+\!\!
    \begin{bmatrix}
    Ad_{\hat{\mathbf{X}}_{A_t}}\!\!\!&\!\!\!\mathbf{0}_{18\times6}\\
    \mathbf{0}_{6\times18}\!\!\!&\!\!\!\mathbf{I}_{6\times6}
    \end{bmatrix}\!\! \mathbf{w}_t\\
    &\!\!\triangleq\!\!
     \mathbf{A}_{t}\boldsymbol{\epsilon}_{t}+\mathbf{W}_{t}\mathbf{w}_{t},
\end{aligned}
\end{equation}
% where
% \begin{equation}
%     \mathbf{A}_{\mathbf{X}_{A_t}}=\begin{bmatrix}\mathbf{0}&\mathbf{0}&\mathbf{0}&\mathbf{0}&\mathbf{0}&\mathbf{0}\\
%     (\mathbf{g})_{\times}&\mathbf{0}&\mathbf{0}&\mathbf{0}&\mathbf{0}&\mathbf{0}\\
%     \mathbf{0}&\mathbf{I}&\mathbf{0}&\mathbf{0}&\mathbf{0}&\mathbf{0}\\
%     \mathbf{0}&\mathbf{0}&\mathbf{0}&\mathbf{0}&\mathbf{0}&\mathbf{0}\\
%     \mathbf{0}&\mathbf{0}&\mathbf{0}&\mathbf{0}&\mathbf{0}&\mathbf{0}\end{bmatrix}
% \end{equation}
where
\begin{equation}\label{eq:At}
    \begin{aligned}
    &\mathbf{A}_{t}=\\&
    \left[\begin{matrix}\begin{array}{cccccc;{2pt/2pt}cc}\mathbf{0}&\mathbf{0}&\mathbf{0}&\mathbf{0}&\mathbf{0}&\mathbf{0}&-^{L}\hat{\mathbf{R}}_{I_t}&\mathbf{0}\\
    (\mathbf{g})_{\times}&\mathbf{0}&\mathbf{0}&\mathbf{0}&\mathbf{0}&\mathbf{0}&-(^{L}\hat{\mathbf{v}}_{I_t})_{\times}{}^{L}\hat{\mathbf{R}}_{I_t}&-^{L}\hat{\mathbf{R}}_{I_t}\\
    \mathbf{0}&\mathbf{I}_3&\mathbf{0}&\mathbf{0}&\mathbf{0}&\mathbf{0}&-(^{L}\hat{\mathbf{p}}_{I_t})_{\times}{}^{L}\hat{\mathbf{R}}_{I_t}&\mathbf{0}\\
    \mathbf{0}&\mathbf{0}&\mathbf{0}&\mathbf{0}&\mathbf{0}&\mathbf{0}&-(^{L}\hat{\mathbf{p}}_{f_t})_{\times}{}^{L}\hat{\mathbf{R}}_{I_t}&\mathbf{0}\\
    \mathbf{0}&\mathbf{0}&\mathbf{0}&\mathbf{0}&\mathbf{0}&\mathbf{0}&-(^{L}\hat{\mathbf{p}}_{G_t})_{\times}{}^{L}\hat{\mathbf{R}}_{I_t}&\mathbf{0}\\
    \mathbf{0}&\mathbf{0}&\mathbf{0}&\mathbf{0}&\mathbf{0}&\mathbf{0}&\mathbf{0}&\mathbf{0}\\
    \hdashline
    \mathbf{0}&\mathbf{0}&\mathbf{0}&\mathbf{0}&\mathbf{0}&\mathbf{0}&\mathbf{0}&\mathbf{0}\\
    \mathbf{0}&\mathbf{0}&\mathbf{0}&\mathbf{0}&\mathbf{0}&\mathbf{0}&\mathbf{0}&\mathbf{0}\end{array}\end{matrix}\right],  
    \end{aligned}
\end{equation}
$Ad_{\hat{\mathbf{X}}_{A_t}}$ is the adjoint operator of $SE_{2+1}^{1}(3)$ applied to $\hat{\mathbf{X}}_{A_t}$, which is given by:
\begin{equation}
    \begin{bmatrix}
    ^{L}\hat{\mathbf{R}}_{I_t}&\mathbf{0}&\mathbf{0}&\mathbf{0}&\mathbf{0}&\mathbf{0}\\
    (^{L}\hat{\mathbf{v}}_{I_t})_{\times}{}^{L}\hat{\mathbf{R}}_{I_t}&^{L}\hat{\mathbf{R}}_{I_t}&\mathbf{0}&\mathbf{0}&\mathbf{0}&\mathbf{0}\\
    (^{L}\hat{\mathbf{p}}_{I_t})_{\times}{}^{L}\hat{\mathbf{R}}_{I_t}&\mathbf{0}&^{L}\hat{\mathbf{R}}_{I_t}&\mathbf{0}&\mathbf{0}&\mathbf{0}\\
     (^{L}\hat{\mathbf{p}}_{f_t})_{\times}{}^{L}\hat{\mathbf{R}}_{I_t}&\mathbf{0}&\mathbf{0}&^{L}\hat{\mathbf{R}}_{I_t}&\mathbf{0}&\mathbf{0}\\
     (^{L}\hat{\mathbf{p}}_{G_t})_{\times}{}^{L}\hat{\mathbf{R}}_{I_t}&\mathbf{0}&\mathbf{0}&\mathbf{0}&^{L}\hat{\mathbf{R}}_{I_t}&\mathbf{0}\\
     \mathbf{0}&\mathbf{0}&\mathbf{0}&\mathbf{0}&\mathbf{0}&^{L}\hat{\mathbf{R}}_{G_t}\\
    \end{bmatrix},
\end{equation}
and $\mathbf{w}_{t}\triangleq \begin{bmatrix}
\mathbf{w}_{g_t}^{\top}&\mathbf{w}_{a_t}^{\top}&\mathbf{0}_{12\times1}^{\top}&\mathbf{w}_{bg_t}^{\top}&\mathbf{w}_{ba_t}^{\top}
\end{bmatrix}^{\top}$. All the $\mathbf{0}$ in above representations are with the order of $3$.
The detailed derivation is given in \textbf{Appendix A} of the supplementary material \cite{supplementary}.

Actually, from (\ref{eq:linearized_propagation}) and (\ref{eq:At}), we can find that $\mathbf{A}_{\mathbf{X}_{A_t}}$ consists of constants. This is because the right invariant error defined by (\ref{eq:nonlinear_error}) can be propagated independently of the system state, according to \textbf{Theorem 1} in \cite{c2}.

Combining $\mathbf{X}_{A}$ with IMU bias parameters $\mathbf{B}$, we can design an ``imperfect invariant EKF" as mentioned in \cite{D,contact}. The linearized error kinematics of (\ref{eq:kinematics}) is given by (\ref{eq:linearized_propagation}). The covariance of the right-invariant error is computed by solving the Riccati equation
\begin{equation}
    \label{eq:Riccati}
    \frac{d}{dt}\mathbf{P}_{t}=\mathbf{A}_{t}\mathbf{P}_{t}+\mathbf{P}_{t}\mathbf{A}_{t}^{\top}+\hat{\mathbf{Q}}_{t},
\end{equation}
where $\hat{\mathbf{Q}}_{t}\triangleq \mathbf{W}_{t}\text{Cov}(\mathbf{w}_{t})\mathbf{W}_{t}^{\top}$, and $\text{Cov}(\mathbf{w}_{t})$ is the covariance matrix of the noise vector $\mathbf{w}_{t}$.

\subsection{Updating for the \textit{perfect augmented system}}\label{update perfect}
Updating procedure is closely related with observation functions. As indicated by \cite{consistent1}, for the general observation function $\mathbf{y}_{t}=h(\mathbf{X}_{A_t},\mathbf{B}_{t})+\boldsymbol{\gamma}_t$, where $\mathbf{y}_{t}$ is the observation of the visual sensor and  $\boldsymbol{\gamma}_t$ is the zero-mean Gaussian noise with the covariance $\mathbf{V}_t$. The error of observation (innovation term) can be written as 
\begin{equation}
    \begin{aligned}
         \mathbf{r}_{t} &\triangleq \mathbf{y}_{t}- h(\hat{\mathbf{X}}_{A_t},\hat{\mathbf{B}}_{t})\\
         &=h(\mathbf{X}_{A_t},\mathbf{B}_{t})-h(\hat{\mathbf{X}}_{A_t},\hat{\mathbf{B}}_{t})+\boldsymbol{\gamma}_t\\
         &=h(\mathrm{exp}(-\boldsymbol{\xi}_{A_t})\hat{\mathbf{X}}_{A_t},\hat{\mathbf{B}}_{t}-\boldsymbol{\xi}_{B_t})-h(\hat{\mathbf{X}}_{A_t},\hat{\mathbf{B}}_{t})+\boldsymbol{\gamma}_t.
    \end{aligned}
\end{equation}
 Toward this function, we can perform the first-order Taylor expansion. However, unlike the standard Taylor expansion represented with the standard vector error, we need to formulate the Taylor expansion with the \textit{right invariant error} $\boldsymbol{\epsilon}\triangleq \begin{bmatrix}\boldsymbol{\xi}_{A_t}^{\top}&\boldsymbol{\xi}_{B_t}^{\top}\end{bmatrix}^{\top}$ as
\begin{equation}
    \label{eq:general_linear_ob_function}
    \mathbf{r}_{t} =
     \mathbf{y}_{t}-h(\hat{\mathbf{X}}_{A_t},\hat{\mathbf{B}}_{t})=-\mathbf{H}_t\boldsymbol{\epsilon}_{t}+\boldsymbol{\gamma}_t+O(\Vert \boldsymbol{\epsilon}_t \Vert^{2}).
\end{equation}

The observation function in this paper can be classified as local feature based observation function and map feature based observation function.

\textbf{Local feature based observation function:} Suppose we have a feature in the local odometry frame $L$, $^{L}\mathbf{p}_{f}$, which is observed by the robot at the current timestamp $t$. For simplicity, we assume the extrinsic between the camera and the IMU is the identity matrix, so that the observation function can be given as
\begin{equation}\label{eq:ob_local}
 \mathbf{y}_{L_t}=h(^{L}\mathbf{R}_{I_t}^{\top}(^{L}\mathbf{p}_{f_t}-^{L}\mathbf{p}_{I_t}))+\boldsymbol{\gamma}_{L_t},
\end{equation}
where $h(\cdot)$ is the projection function of the camera. The subscript $L$ of $\mathbf{y}_{L_t}$ and $\boldsymbol{\gamma}_{L_t}$ indicates the local observation. The Jacobian matrix of (\ref{eq:ob_local}) is given by
\begin{equation}\label{eq:local_H}
    \mathbf{H}_{L_t}=-\nabla h |_{\hat{\mathfrak{q}}} {}^{L}{\hat{\mathbf{R}}}_{I_t}^{\top} \begin{bmatrix}
    \mathbf{0}_3&\mathbf{0}_3&\mathbf{I}_3&-\mathbf{I}_3&\mathbf{0}_3&\mathbf{0}_3&\mathbf{0}_3&\mathbf{0}_3
    \end{bmatrix},
\end{equation}
where $\mathfrak{q} \triangleq {}^{L}\mathbf{R}_{I_t}^{\top}(^{L}\mathbf{p}_{f_t}-^{L}\mathbf{p}_{I_t})$, $\nabla h \triangleq \frac{dh}{d\mathfrak{q}} $. The detailed derivation is in \textbf{Appendix B} of our supplementary material\cite{supplementary}.

\textbf{Map feature based observation function:} Suppose we have a feature provided by the map, $^{G}\mathbf{p}_{F}$, which is observed by the robot again at the current timestamp $t$ (cf. Fig. \ref{fig:coordinate}). Then, with the \textit{augmented variable} $^{L}\mathbf{T}_{G}$ as a bridge, the observation function can be formulated as
\begin{equation} \label{eq:ob_global}
    \mathbf{y}_{G_t}=h(^{L}\mathbf{R}_{I_t}^{\top}(^{L}\mathbf{R}_{G_t}{}^{G}\mathbf{p}_{F}+^{L}\mathbf{p}_{G_t}-^{L}\mathbf{p}_{I_t}))+\boldsymbol{\gamma}_{G_t}.
\end{equation}
The subscript $G$ of $\mathbf{y}_{G_t}$ and $\boldsymbol{\gamma}_{G_t}$ indicates the map-based (global) observation. The Jacobian matrix of (\ref{eq:ob_global}) is given by
\begin{equation}\label{eq:global_H}
\begin{aligned}
     \mathbf{H}_{G_t}=-\nabla h^{\prime}|_{\hat{\mathfrak{q}}^{\prime}}{}^{L}\hat{\mathbf{R}}_{I_t}^{\top}
    &\left[\begin{matrix}
    -(^{L}\hat{\mathbf{R}}_{G_t}{}^{G}\mathbf{p}_{F})_{\times}&\mathbf{0}_3&\mathbf{I}_3&\mathbf{0}_3&-\mathbf{I}_3\end{matrix}\right.\\
    &\left.\begin{matrix}(^{L}\hat{\mathbf{R}}_{G_t}{}^{G}\mathbf{p}_{F})_{\times}&\mathbf{0}_3&\mathbf{0}_3
    \end{matrix}\right],
\end{aligned}
\end{equation}
where $\mathfrak{q}^{\prime} \triangleq {}^{L}\mathbf{R}_{I_t}^{\top}(^{L}\mathbf{R}_{G_t}{}^{G}\mathbf{p}_{F}+^{L}\mathbf{p}_{G_t}-^{L}\mathbf{p}_{I_t})$, $\nabla h^{\prime} \triangleq \frac{dh}{d\mathfrak{q}^{\prime}} $. The detailed derivation is in \textbf{Appendix B} of our supplementary material\cite{supplementary}.

With (\ref{eq:linearized_propagation}), (\ref{eq:ob_local}) and (\ref{eq:ob_global}), the state of the \textit{augmented system} can be propagated and updated with the procedure like EKF. The whole process of invariant EKF for the \textit{perfect augmented system} is summarized in \textbf{Algorithm \ref{alg:invariant_EKF}}. 

\renewcommand{\algorithmicrequire}{\textbf{Input:}}
\renewcommand{\algorithmicensure}{\textbf{Output:}}

\begin{algorithm}[t] \small
\caption{Invariant EKF for the \textit{perfect augmented system}}
\label{alg:invariant_EKF}
\begin{algorithmic}[1] 
\REQUIRE
the posterior of the state at timestamp $t-1$, $\hat{\mathbf{X}}_{t-1}$, the covariance of the state $\mathbf{P}_{t-1}$, the IMU measurements $\boldsymbol{\omega}_{t}$, $\mathbf{a}_{t}$.
\ENSURE
$\hat{\mathbf{X}}_{t}$, $\mathbf{P}_{t}$.
\STATE  Propagate the state with (\ref{eq:kinematics}) to get the prediction $\hat{\mathbf{X}}_{t|t-1}$, propagate the covariance with (\ref{eq:linearized_propagation}) and (\ref{eq:Riccati}) to get $\mathbf{P}_{t|t-1}$.
\IF {have tracked local features}
    \STATE compute the innovation term $\mathbf{r}_{L_t}$ and $\mathbf{H}_{L_t}$ with (\ref{eq:general_linear_ob_function}), (\ref{eq:ob_local}) and (\ref{eq:local_H}).
\ENDIF
\IF {have matched features from the map}
    \STATE compute the innovation term $\mathbf{r}_{G_t}$ and $\mathbf{H}_{G_t}$ with (\ref{eq:general_linear_ob_function}), (\ref{eq:ob_global}) and (\ref{eq:global_H}).
\ENDIF
\STATE Stack $\mathbf{r}_{L_t}$ and $\mathbf{r}_{G_t}$ as $\mathbf{r}_{t}$, $\mathbf{H}_{L_t}$ and $\mathbf{H}_{G_t}$ as $\mathbf{H}_{t}$.
\STATE Update the state and the covariance:\\
     $\mathbf{S}_t=\mathbf{H}_{t}\mathbf{P}_{t|t-1}\mathbf{H}_{t}^{\top}+\mathbf{V}_t$,\\
     $\mathbf{K}_{t}\triangleq \begin{bmatrix}\mathbf{K}_{\mathbf{X}_{A_t}}\\
     \mathbf{K}_{\mathbf{B}_t}\end{bmatrix}=\mathbf{P}_{t|t-1}\mathbf{H}^{\top}_{t}\mathbf{S}_{t}^{-1}$,\\
     $\hat{\mathbf{X}}_{t}=\begin{bmatrix}
     \hat{\mathbf{X}}_{A_{t}}\\
     \hat{\mathbf{B}}_{t}
     \end{bmatrix}=\begin{bmatrix}
     \mathrm{exp}(\mathbf{K}_{\mathbf{X}_{A_t}}\mathbf{r}_{t})\hat{\mathbf{X}}_{A_{t|t-1}}\\
     \hat{\mathbf{B}}_{t|t-1}+\mathbf{K}_{\mathbf{B}_{t}}\mathbf{r}_{t}
     \end{bmatrix}$,\\
     $\mathbf{P}_{t}=\mathbf{P}_{t|t-1}-\mathbf{K}_{t}\mathbf{H}_{t}\mathbf{P}_{t|t-1}$.
\RETURN $\hat{\mathbf{X}}_{t}$, $\mathbf{P}_{t}$.
\end{algorithmic} 
\end{algorithm}
% \vspace{-2mm}

\section{Invariant EKF for imperfect map based visual-inertial localization}
\label{sec:IEKF imperfect}
In Sec. \ref{sec:IEKF perfect}, we design an invariant EKF for the \textit{perfect augmented system} by employing a novel Lie group and its algebra. However, in practice, the pre-built map is actually imperfect and has uncertainty. Neglecting the uncertainty of the map information will lead to overconfident estimation, which encourages us to consider the uncertainty of the map information. This kind of \textit{imperfect map} based visual-inertial localization system is called the \textit{imperfect augmented system}.  

\subsection{State for the imperfect augmented system} \label{sec:state imperfect}
In order to consider the uncertainty of the map information, we need to add the map-related variables into the system state, so that the state of the \textit{imperfect augmented system} at timestamp $t$ is defined as
\begin{equation} \label{eq:imperfect state}
    \mathbf{X}_{t}^{*} = (\mathbf{X}_{A_t},\mathbf{B}_{t},\mathbf{X}_{M_t}) = (\mathbf{X}_{t},\mathbf{X}_{M_t}),
\end{equation} 
where $\mathbf{X}_{M_t}$ contains $m$ map keyframe poses $\{\mathbf{X}_{{KF_1},t},\cdots,\mathbf{X}_{{KF_m},t}\}$ and $n$ map feature positions $\{{}^{G}\mathbf{p}_{{F_1},t},\cdots,{}^{G}\mathbf{p}_{{F_n},t}\}$. For each map keyframe pose $\{\mathbf{X}_{KF_i,t}\}, i=1 \cdots m$, it consists of $\{^{G}\mathbf{R}_{{KF_i},t},^{G}\mathbf{p}_{{KF_i},t}\}$.

Following the definition of (\ref{eq:nonlinear_error}), the \textit{right-invariant error} of the map keyframe pose, $\boldsymbol{\epsilon}_{KF_i,t} \triangleq \begin{bmatrix}\boldsymbol{\xi}_{\boldsymbol{\theta}_{KF_i,t}}^{\top}&\boldsymbol{\xi}_{\mathbf{p}_{KF_i,t}}^{\top}\end{bmatrix}^{\top}$, is given as
\begin{equation}
\begin{aligned}\label{eq:error_map_kf}
   \boldsymbol{\xi}_{\boldsymbol{\theta}_{KF_i,t}} &= \tilde{\boldsymbol{\theta}}_{KF_i,t},\\
   \boldsymbol{\xi}_{\mathbf{p}_{KF_i,t}} &= ^{G}\hat{\mathbf{p}}_{KF_i,t} - (\mathbf{I}_{3} + (\tilde{\boldsymbol{\theta}}_{KF_i,t})_{\times}){}^{G}\mathbf{p}_{KF_i,t},
\end{aligned}
\end{equation}
while for the map feature position, we define the error in the Euclidean vector space instead of the Lie group space for the sake of convenience of the Jacobian derivation:
\begin{equation}\label{eq:error_map_landmark}
   \boldsymbol{\epsilon}_{F_i,t}\triangleq \boldsymbol{\xi}_{F_i,t} = {}^{G}\hat{\mathbf{p}}_{F_i,t}-{}^{G}\mathbf{p}_{F_i,t}.
\end{equation}
We define the nonlinear error of $\mathbf{X}_{t}^{*}$ as $\boldsymbol{\epsilon}_{t}^{*}\triangleq \begin{bmatrix}\boldsymbol{\epsilon}_{t}^{\top}&\cdots&\boldsymbol{\epsilon}_{KF_i,t}^{\top}&\cdots&\boldsymbol{\epsilon}_{F_j,t}^{\top}&\cdots\end{bmatrix}^{\top}$.

\subsection{Propagation for the imperfect augmented system}\label{sec:prop imperfect}
The kinematics of the $\mathbf{X}_{t}$ part is identical to (\ref{eq:kinematics}). For the $\mathbf{X}_{M_t}$ part, we assume the map-related variables should remain unchanged:
\begin{equation}
    \left\{\begin{aligned}
      ^{G}\dot{\mathbf{R}}_{KF_i,t} &= \mathbf{0}_{3}\\
      ^{G}\dot{\mathbf{p}}_{KF_i,t} &= \mathbf{0}_{3\times1}\\
      ^{G}\dot{\mathbf{p}}_{F_i,t} &= \mathbf{0}_{3\times1}
    \end{aligned}
    \right. 
\end{equation}

\subsection{Updating for the imperfect augmented system}\label{sec:update imperfect}
Since the local feature based observation function is not related with the map information, in this part, we only introduce the map feature based observation function. As shown in Fig. \ref{fig:coordinate}, supposing a map feature $^{G}\mathbf{p}_{F_j,t}$ is observed by the current frame $C_t$ and the map keyframe $KF_i$, we have the following observation functions:
\begin{equation}\label{eq:imperfect global_ob_c}
    \mathbf{y}_{G_t}^{C_j}=h(^{L}\mathbf{R}_{I_t}^{\top}(^{L}\mathbf{R}_{G_t}{}^{G}\mathbf{p}_{F_j,t}+^{L}\mathbf{p}_{G_t}-{}^{L}\mathbf{p}_{I_t}))+\boldsymbol{\gamma}_{G_{t}},
\end{equation}
\begin{equation}\label{eq:imperfect global_ob_kf}
    \mathbf{y}_{G_t}^{KF_ij}=h(^{G}\mathbf{R}_{KF_i,t}^{\top}({}^{G}\mathbf{p}_{F_j,t}-{^{G}\mathbf{p}_{KF_i,t}}))+\boldsymbol{\gamma}_{G_{i}},
\end{equation}
where the superscript ${C_j}$ means the observation of $C_t$ toward the map feature $F_j$, the superscript ${KF_ij}$ means the observation of the map keyframe $KF_i$ toward the map feature $F_j$, and $\boldsymbol{\gamma}_{G_t}$ and $\boldsymbol{\gamma}_{G_i}$ represent the observation noises of $C_t$ and $KF_i$, respectively.

Note that (\ref{eq:imperfect global_ob_c}) is almost identical to (\ref{eq:ob_global}), except that the feature ${}^{G}\mathbf{p}_{F_j,t}$ is variable whereas the feature ${}^{G}\mathbf{p}_{F}$ in (\ref{eq:ob_global}) is constant.

The Jacobian matrix of (\ref{eq:imperfect global_ob_c}) is given by
\begin{equation}\label{eq:imperfect jacob c}
\begin{aligned}
     \mathbf{H}_{G_t}^{C_j}&=-\nabla h^{\prime}|_{\hat{\mathfrak{q}}^{\prime}}{}^{L}\hat{\mathbf{R}}_{I_t}^{\top}
    \left[\begin{matrix}
    -(^{L}\hat{\mathbf{R}}_{G_t}{}^{G}\hat{\mathbf{p}}_{F_j,t})_{\times}&\mathbf{0}_3&\mathbf{I}_3&\mathbf{0}_3&-\mathbf{I}_3\end{matrix}\right.\\
    &\left.\begin{matrix}(^{L}\hat{\mathbf{R}}_{G_t}{}^{G}\hat{\mathbf{p}}_{F_j,t})_{\times}&\mathbf{0}_3&\mathbf{0}_3|
\cdots \mathbf{0}_3&\mathbf{0}_3\cdots -^{L}\hat{\mathbf{R}}_{G_t}\cdots
    \end{matrix}\right],
\end{aligned}
\end{equation}
where the elements before $|$ are identical to the elements of (\ref{eq:global_H}), and the elements after $|$ are the Jacobians for the map-related variables. Specifically, the Jacobian of  $\mathbf{y}_{G_t}^{C_j}$ with respect to $KF_i$ is $\begin{bmatrix}
\mathbf{0}_3&\mathbf{0}_3
\end{bmatrix}$, the Jacobian of $\mathbf{y}_{G_t}^{C_j}$ with respect to $F_j$ is $-^{L}\hat{\mathbf{R}}_{G_t}$. The other omitted Jacobians are zeros. The detailed derivation is in \textbf{Appendix B} of the supplementary material\cite{supplementary}.

The Jacobian matrix of (\ref{eq:imperfect global_ob_kf}) is given by
\begin{equation}\label{eq:imperfect jacob kf}
\begin{aligned}
     \mathbf{H}_{G_t}^{KF_ij}&=-\nabla h^{*}|_{\hat{\mathfrak{q}}^{*}}{}^{G}\hat{\mathbf{R}}_{KF_i,t}^{\top}
    \left[\begin{matrix}
    \mathbf{0}_3&\mathbf{0}_3&\mathbf{0}_3&\mathbf{0}_3&\mathbf{0}_3\end{matrix}\right.\\
    &\left.\begin{matrix}\mathbf{0}_3&\mathbf{0}_3&\mathbf{0}_3|
\cdots -({}^{G}\hat{\mathbf{p}}_{F_j,t})_{\times}&\mathbf{I}_{3}\cdots -\mathbf{I}_{3}\cdots
    \end{matrix}\right],
\end{aligned}   
\end{equation}
where $\mathfrak{q}^{*} \triangleq {}^{G}\mathbf{R}_{KF_i,t}^{\top}({}^{G}\mathbf{p}_{F_j,t}-{}^{G}\mathbf{p}_{KF_i,t})$, $\nabla h^{*} \triangleq \frac{dh}{d\mathfrak{q}^{*}}$.
Specifically, the Jacobian of $\mathbf{y}_{G_t}^{KF_ij}$ with respect to $KF_i$ is $\left[-({}^{G}\hat{\mathbf{p}}_{F_j,t})_{\times},\mathbf{I}_{3}\right]$, and the Jacobian of $\mathbf{y}_{G_t}^{KF_ij}$ with respect to $F_j$ is $-\mathbf{I}_{3}$. The other omitted Jacobians are zeros. The detailed derivation is in \textbf{Appendix B} of the supplementary material \cite{supplementary}.

\section{Observability analysis of map-based visual-inertial localization system}
In this section, we theoretically analyze the observability properties of the \textit{perfect augmented system} and the \textit{imperfect augmented system} from the perspectives of standard EKF and invariant EKF. Besides, intuitive explanations are also given.   

% In this section, we firstly analyze the observability of the \textit{perfect augemented system} with standard EKF in the ideal case (the nonlinear functions are linearized with the ground truth state), which reveals the dimension of the unobservable subspace that the \textit{perfect augmented system} should have. However, for the real case, due to the ever-changing linearization points, the unobservable subspace that the \textit{perfect augmented system} is supposed to have is broken. Then, the advantage that our proposed invariant EKF can naturally maintain the correct unobservable subspace is demonstrated. Moreover, for the \textit{imperfect augmented system}, as we introduce the map information into the state, the dimensions of the unobservable subspace expand. We give the theoretical derivation toward the observability of this kind of \textit{imperfect augmented system}. Besides, some intuitive explanations and related corollary are given.   

To analyze the unobservable subspace of the \textit{augmented system}, we need to compute the observability matrix of the \textit{augmented system}, and the unobservable subspace is spanned by the right null space of the observability matrix.

For simplicity, the analysis below neglects the IMU bias, assumes the extrinsic between the camera and the IMU is identity matrix, and sets the number of the local features, the map features and the map keyframes to be one, while the results can be extended to general cases.

\subsection{Observability of the perfect augmented system with standard EKF}\label{sec:ob_ekf}
For standard EKF, the state is represented in the Euclidean vector space as
\begin{equation}\label{eq:simple_state}
    \mathbf{x}_{s_t}=\begin{bmatrix}
    ^{L}\mathbf{q}_{I_t}^{\top}&^{L}\mathbf{v}_{I_t}^{\top}&^{L}\mathbf{p}_{I_t}^{\top}&^{L}\mathbf{p}_{f_t}^{\top}&^{L}\mathbf{q}_{G_t}^{\top}&^{L}\mathbf{p}_{G_t}^{\top}
    \end{bmatrix}^{\top},
\end{equation}
where $^{L}\mathbf{q}_{I_t}$ is the quaternion of $^{L}\mathbf{R}_{I_t}$,  and $^{L}\mathbf{q}_{G_t}$ is the quaternion of $^{L}\mathbf{R}_{G_t}$.

Denoting the state transition matrix of the \textit{perfect augmented system} from timestamp ${t-1}$ to $t$ as ${}^{st}\boldsymbol{\Phi}_{{t}|{t-1}}$, we have
\begin{equation}\label{ofvio}
{}^{st}\boldsymbol{\Phi}_{t|{0}} \triangleq {}^{st}\boldsymbol{\Phi}_{t|{t-1}}\ldots{}^{st}\boldsymbol{\Phi}_{2|1}{}^{st}\boldsymbol{\Phi}_{1|0}.
\end{equation}
Accordingly, denoting the Jacobian matrix of the local feature based observation function at timestamp $t$ as ${}^{st}\mathbf{H}_{L_t}$, and that of the map-based observation function as ${}^{st}\mathbf{H}_{G_t}$, we have the observability matrix following \cite{gpsvio} as
\begin{equation}\label{om}
{}^{st}\mathbf{M}\triangleq\begin{bmatrix}
{}^{st}\mathbf{H}_{L_0}\\
{}^{st}\mathbf{H}_{G_0}\\
{}^{st}\mathbf{H}_{L_1}{}^{st}\boldsymbol{\Phi}_{1|0}\\
{}^{st}\mathbf{H}_{G_1}{}^{st}\boldsymbol{\Phi}_{1|0}\\
\vdots\\
{}^{st}\mathbf{H}_{L_t}{}^{st}\boldsymbol{\Phi}_{t|0}\\
{}^{st}\mathbf{H}_{G_t}{}^{st}\boldsymbol{\Phi}_{t|0}\\
\end{bmatrix}
\triangleq
\begin{bmatrix}
{}^{st}\mathbf{M}_{L_0}\\
{}^{st}\mathbf{M}_{G_0}\\
{}^{st}\mathbf{M}_{L_1}\\
{}^{st}\mathbf{M}_{G_1}\\
\vdots\\
{}^{st}\mathbf{M}_{L_t}\\
{}^{st}\mathbf{M}_{G_t}\\
\end{bmatrix}.
\end{equation}
The left-superscript $st$ represents that the matrix is derived from standard EKF.

We first assume that the Jacobian matrices are evaluated at the ground truth, which is ideal but can demonstrate the theoretical implication.
\begin{lemma}
\label{lemma: ideal perfect ekf}
    (Ideal observability of \textit{perfect augmented system} with standard EKF) The right null space $^{st}\mathcal{N}_1$ of the observability matrix $^{st}\mathbf{M}$, where the Jacobian matrices are evaluated ideally, is spanned by four directions as
\begin{equation}
^{st}\mathcal{N}_1 = span\left(\begin{bmatrix}
    {}^{L}\hat{\mathbf{R}}_{I_0}^{\top}\mathbf{g}&\mathbf{0}_{3}\\
    -(^{L}\hat{\mathbf{v}}_{I_0})_{\times}\mathbf{g}&\mathbf{0}_{3}\\
    -(^{L}\hat{\mathbf{p}}_{I_0})_{\times}\mathbf{g}&\mathbf{I}_{3}\\
    -(^{L}\hat{\mathbf{p}}_{f})_{\times}\mathbf{g}&\mathbf{I}_{3}\\
    {}^{L}\mathbf{R}_{G}^{\top}\mathbf{g}&\mathbf{0}_{3}\\
    -(^{L}\mathbf{p}_{G})_{\times}\mathbf{g}&\mathbf{I}_{3}\\
    \end{bmatrix}\right).
\end{equation}
\end{lemma}
\begin{proof}
      See \textbf{Appendix C} of the supplementary material\cite{supplementary}.
\end{proof}

However, we cannot access the ground truth in practice, so the Jacobian matrices can only be evaluated at some estimated points. This causes spurious information and breaks the original observability of the \textit{augmented system} with standard EKF:

\begin{theorem}
\label{theorem: real perfect ekf}
(Real observability of \textit{perfect augmented system} with standard EKF) The right null space ${}^{st}\mathcal{N}_2$ of the observability matrix ${}^{st}\mathbf{M}$, where the Jacobian matrices are evaluated at changing estimated values, is spanned by three directions as

\begin{equation}
^{st}\mathcal{N}_2 = span\left(\begin{bmatrix}
    \mathbf{0}_{3}&
  \mathbf{0}_{3}&
    \mathbf{I}_{3}&
  \mathbf{I}_{3}&
    \mathbf{0}_{3}&
  \mathbf{I}_{3}
    \end{bmatrix}^{\top}\right).
\end{equation}

% \begin{equation}
% ^{st}\mathcal{N}_2 = span\left(\begin{bmatrix}
%     \mathbf{0}_{3}\\
%   \mathbf{0}_{3}\\
%     \mathbf{I}_{3}\\
%   \mathbf{I}_{3}\\
%     \mathbf{0}_{3}\\
%   \mathbf{I}_{3}\\
%     \end{bmatrix}\right).
% \end{equation}
\end{theorem}

\begin{proof}
      See \textbf{Appendix C} of the supplementary material \cite{supplementary}.
\end{proof}

\subsection{Observability of the perfect augmented system  with invariant EKF} \label{sec:ob perfect ikf}
For invariant EKF, the state is represented as $\mathbf{X}_{A_t}$ in (\ref{eq:state}). The state transition matrix, Jacobian matrix of the observation function, and the observability matrix derived from the invariant EKF are denoted similar to those in (\ref{ofvio}) and (\ref{om}), except that the left-superscript $st$ is replaced by $in$.

\begin{theorem}
\label{theorem: perfect iekf}
    (Observability of \textit{perfect augmented system} with invariant EKF) The right null space ${}^{in}\mathcal{N}$ of the observability matrix ${}^{in}\mathbf{M}$, is independent of the state values and spanned by four directions as
% \begin{equation}\label{eq:null_space_om_inekf}
%     {}^{in}\mathcal{N}=span\left(\begin{bmatrix}\mathbf{g}&\mathbf{0}_{3}\\
%     \mathbf{0}_{3\times1}&\mathbf{0}_{3}\\
%     \mathbf{0}_{3\times1}&\mathbf{I}_{3}\\
%     \mathbf{0}_{3\times1}&\mathbf{I}_{3}\\
%     \mathbf{0}_{3\times1}&\mathbf{I}_{3}\\
%     \mathbf{g}&\mathbf{0}_{3}\\
%     \end{bmatrix}\right).
% \end{equation}
\begin{equation}\label{eq:null_space_om_inekf}
    {}^{in}\mathcal{N}=span\left(\begin{bmatrix}\mathbf{g}^{\top}&\mathbf{0}_{1\times 3}&\mathbf{0}_{1\times 3}&\mathbf{0}_{1\times 3}&\mathbf{0}_{1\times 3}&\mathbf{g}^{\top}\\
    \mathbf{0}_{3}&\mathbf{0}_{3}&\mathbf{I}_{3}&\mathbf{I}_{3}&\mathbf{I}_{3}&\mathbf{0}_{3}
    \end{bmatrix}^{\top}\right).
\end{equation}
\end{theorem}

\begin{proof}
     See \textbf{Appendix D} of the supplementary material\cite{supplementary}.
\end{proof}

As ${}^{in}\mathcal{N}$ consists of constant values, the unobservable subspace of the \textit{perfect augmented system} will not be broken even the state values are ever-changing. Therefore, the correct observability of the \textit{perfect augmented system} can be maintained naturally with our proposed invariant EKF (cf. Sec. \ref{sec:IEKF perfect}).

\subsection{Observability of the imperfect augmented system with standard EKF}
\label{sec:ob imperfect std ekf}
As mentioned before, the map components always have more or less uncertainty. It is necessary to consider the uncertainty of the map components so that the system can be of good consistency. 

For the standard EKF, the state of the \textit{imperfect augmented system} is defined as
\begin{equation}\label{eq:simple_state_imperfect}
\begin{aligned}
    \mathbf{x}_{s_t}^{*}=&\left[\begin{matrix}
    ^{L}\mathbf{q}_{I_t}^{\top}&^{L}\mathbf{v}_{I_t}^{\top}&^{L}\mathbf{p}_{I_t}^{\top}&^{L}\mathbf{p}_{f_t}^{\top}&^{L}\mathbf{q}_{G_t}^{\top}&^{L}\mathbf{p}_{G_t}^{\top}|\end{matrix}\right.\\
    &\left.\begin{matrix}{}^{G}\mathbf{q}_{KF_t}^{\top}&{}^{G}\mathbf{p}_{KF_t}^{\top}&{}^{G}\mathbf{p}_{F_t}^{\top}\end{matrix} \right]^{\top},
\end{aligned}
\end{equation}
where the elements after $|$ are map-related. ${}^{G}\mathbf{q}_{KF_t}$ and ${}^{G}\mathbf{p}_{KF_t}$ form a map keyframe pose, and ${}^{G}\mathbf{p}_{F_t}$ is a map feature position.

The state transition matrix, Jacobian matrix of the observation function, and the observability matrix derived from the standard EKF are denoted similar to those in (\ref{ofvio}) and (\ref{om}), except that we use right-superscript $*$ to distinguish the \textit{(im)perfect augmented system}. 

\begin{theorem}
\label{theorem: ideal imperfect ekf}
    (Ideal observability of \textit{imperfect augmented system} with standard EKF) The right null space $^{st}\mathcal{N}_1^{*}$ of the observability matrix ${}^{st}\mathbf{M}^{*}$, where the Jacobian matrices are evaluated ideally, is spanned by ten (four plus six) directions as
\begin{equation}\label{eq:nullspace ideal std ekf imperfect}
\begin{aligned}
^{st}\mathcal{N}_1^{*} &= span\left(\begin{bmatrix}
    {}^{L}\hat{\mathbf{R}}_{I_0}^{\top}\mathbf{g}&\mathbf{0}_{3}&\mathbf{0}_{3}&\mathbf{0}_{3}\\
    -(^{L}\hat{\mathbf{v}}_{I_0})_{\times}\mathbf{g}&\mathbf{0}_{3}&\mathbf{0}_{3}&\mathbf{0}_{3}\\
    -(^{L}\hat{\mathbf{p}}_{I_0})_{\times}\mathbf{g}&\mathbf{I}_{3}&\mathbf{0}_{3}&\mathbf{0}_{3}\\
    -(^{L}\hat{\mathbf{p}}_{f})_{\times}\mathbf{g}&\mathbf{I}_{3}&\mathbf{0}_{3}&\mathbf{0}_{3}\\
    {}^{L}\mathbf{R}_{G}^{\top}\mathbf{g}&\mathbf{0}_{3}&\mathbf{0}_{3}&-^{G}\mathbf{R}_{KF}\\
    -(^{L}\mathbf{p}_{G})_{\times}\mathbf{g}&\mathbf{I}_{3}&^{L}\mathbf{R}_{G}&\mathbf{0}_3\\
    \mathbf{0}_3& \mathbf{0}_3&\mathbf{0}_{3}&\mathbf{I}_3\\
    \mathbf{0}_3& \mathbf{0}_3&\mathbf{I}_{3}&*_1\\
    \mathbf{0}_3& \mathbf{0}_3&\mathbf{I}_{3}&*_2
    \end{bmatrix}\right)\\
*_1&=-(^{G}\mathbf{p}_{KF})_{\times}{}^{G}\mathbf{R}_{KF}\\
*_2&=-(^{G}\mathbf{p}_{F})_{\times}{}^{G}\mathbf{R}_{KF}.
\end{aligned}
\end{equation}
\end{theorem}

\begin{proof}
     See \textbf{Appendix E} of the supplementary material\cite{supplementary}.
\end{proof}
We can see that $^{st}\mathcal{N}_{1}^{*}$ consists of two parts: the unobservable subspace of the original \textit{perfect augmented system} (the first four columns), and the unobservable subspace introduced by the map-related variables (the last six columns). Further, as the elements of the last six columns of $^{st}\mathcal{N}_{1}^{*}$ include ever-changing variables, the six dimensions of the unobservable subspace will be broken for the real case, which leads to the following theorem:

\begin{theorem}
\label{theorem: real imperfect ekf}
    (Real observability of \textit{imperfect augmented system} with standard EKF) The right null space $^{st}\mathcal{N}_{2}^{*}$ of the observability matrix ${}^{st}\mathbf{M}^{*}$, where the Jacobian matrices are evaluated at changing estimated values, is spanned by three directions as
% \begin{equation} 
% ^{st}\mathcal{N}_2^{*} = span\left(\begin{bmatrix}
%     \mathbf{0}_{3}\\
%   \mathbf{0}_{3}\\
%     \mathbf{I}_{3}\\
%   \mathbf{I}_{3}\\
%     \mathbf{0}_{3}\\
%   \mathbf{I}_{3}\\
%   \mathbf{0}_{3}\\
%   \mathbf{0}_{3}\\
%   \mathbf{0}_{3}
%     \end{bmatrix}\right).
% \end{equation}
\begin{equation} 
^{st}\mathcal{N}_2^{*} = span\left(\begin{bmatrix}
    \mathbf{0}_{3}&
   \mathbf{0}_{3}&
    \mathbf{I}_{3}&
   \mathbf{I}_{3}&
    \mathbf{0}_{3}&
   \mathbf{I}_{3}&
   \mathbf{0}_{3}&
   \mathbf{0}_{3}&
   \mathbf{0}_{3}
    \end{bmatrix}^{\top}\right).
\end{equation}
\end{theorem}
With the prior knowledge of Sec. \ref{sec:ob_ekf}, the proof is trivial and is omitted.

\subsection{Observability of the imperfect augmented system with invariant EKF}

The state of the \textit{imperfect augmented system} with invariant EKF is given by (\ref{eq:imperfect state}), and the state error is also defined by (\ref{eq:error_map_kf}) and (\ref{eq:error_map_landmark}). The state transition matrix, Jacobian matrix of the observation function, and the observability matrix derived from the invariant EKF are denoted similar to those in (\ref{ofvio}) and (\ref{om}).

\begin{theorem}
\label{theorem: ideal imperfect iekf}
    (Ideal observability of \textit{imperfect augmented system} with invariant EKF) The right null space ${}^{in}\mathcal{N}^{*}_{1}$ of the observability matrix ${}^{in}\mathbf{M}^{*}$, where the Jacobian matrices are evaluated ideally, is spanned by ten (four plus six) directions as
\begin{equation}\label{eq:null_space_om_inekf imperfect ideal}
    {}^{in}\mathcal{N}^{*}_{1}=span\left(\begin{bmatrix}\mathbf{g}&\mathbf{0}_{3}&\mathbf{0}_3&\mathbf{0}_3\\
    \mathbf{0}_{3\times1}&\mathbf{0}_{3}&\mathbf{0}_3&\mathbf{0}_3\\
    \mathbf{0}_{3\times1}&\mathbf{I}_{3}&\mathbf{0}_3&\mathbf{0}_3\\
    \mathbf{0}_{3\times1}&\mathbf{I}_{3}&\mathbf{0}_3&\mathbf{0}_3\\
    \mathbf{0}_{3\times1}&\mathbf{I}_{3}&-^{L}\mathbf{R}_{G}&\mathbf{0}_3\\
    \mathbf{g}&\mathbf{0}_{3}&\mathbf{0}_3&\mathbf{I}_3\\
    \mathbf{0}_{3\times1}&\mathbf{0}_{3}&\mathbf{0}_3&-^{L}\mathbf{R}_{G}^{\top}\\
    \mathbf{0}_{3\times1}&\mathbf{0}_{3}&\mathbf{I}_3&\mathbf{0}_3\\
    \mathbf{0}_{3\times1}&\mathbf{0}_{3}&\mathbf{I}_3&(^{G}\mathbf{p}_{F})_{\times}{}^{L}\mathbf{R}_{G}^{\top}\\
    \end{bmatrix}\right).
\end{equation}
\end{theorem}

\begin{proof}
     See \textbf{Appendix F} of the supplementary material\cite{supplementary}.
\end{proof}
Different from the unobservable subspace of the  \textit{perfect augmented system} with invariant EKF (cf. (\ref{eq:null_space_om_inekf})), the elements of the last six columns of $^{in}\mathcal{N}_{1}^{*}$ include ever-changing variables like the standard EKF system in Sec. \ref{sec:ob imperfect std ekf}. Therefore, these six dimensions of the unobservable subspace will be broken for the real case:

\begin{theorem}
\label{theorem: real imperfect iekf}
(Real observability of \textit{imperfect augmented system} with invariant EKF) The right null space $^{in}\mathcal{N}_{2}^{*}$ of the observability matrix ${}^{in}\mathbf{M}^{*}$, where the Jacobian matrices are evaluated at changing estimated values, is spanned by four directions as
% \begin{equation}\label{eq:null_space_om_inekf imperfect real}
%     {}^{in}\mathcal{N}^{*}_{2}=span\left(\begin{bmatrix}\mathbf{g}&\mathbf{0}_{3}\\
%     \mathbf{0}_{3\times1}&\mathbf{0}_{3}\\
%     \mathbf{0}_{3\times1}&\mathbf{I}_{3}\\
%     \mathbf{0}_{3\times1}&\mathbf{I}_{3}\\
%     \mathbf{0}_{3\times1}&\mathbf{I}_{3}\\
%     \mathbf{g}&\mathbf{0}_{3}\\
%     \mathbf{0}_{3\times1}&\mathbf{0}_{3}\\
%     \mathbf{0}_{3\times1}&\mathbf{0}_{3}\\
%     \mathbf{0}_{3\times1}&\mathbf{0}_{3}\\
%     \end{bmatrix}\right).
% \end{equation}
\begin{equation}\label{eq:null_space_om_inekf imperfect real}
\begin{aligned}
     &{}^{in}\mathcal{N}^{*}_{2}=\\
     &span\left(\left[\begin{array}{lllllllll}\mathbf{g}^{\top}\,\mathbf{0}_{1\times3}\,\mathbf{0}_{1\times 3}\,\mathbf{0}_{1\times3}\,\mathbf{0}_{1\times3}\,\mathbf{g}^{\top}\,\mathbf{0}_{1\times3}\,\mathbf{0}_{1\times3}\,\mathbf{0}_{1\times3}\\
    \mathbf{0}_{3}\ \ \; \mathbf{0}_{3}\ \ \ \; \mathbf{I}_{3}\ \ \ \; \mathbf{I}_{3}\ \ \ \, \mathbf{I}_{3}\ \ \, \mathbf{0}_{3}\ \ \, \mathbf{0}_{3}\ \ \ \, \mathbf{0}_{3}\ \ \ \,\mathbf{0}_{3}
    \end{array}\right]^{\top}\right).
\end{aligned}
\end{equation}
\end{theorem}

The proof of \textbf{Theorem \ref{theorem: real imperfect iekf}} is straightforward and is omitted. 

\subsection{Intuitive explanations toward observability}
In this subsection, combining Fig. \ref{fig:intuitive explanation}, intuitive explanations toward the observability of the \textit{(im)perfect augmented system} are given in order to give readers a better understanding of the unobservable subspace of the \textit{augmented systems}.

In Fig. \ref{fig:intuitive explanation}, there are four kinds of frames, i.e. $C,\,L,\,G$ and $KF$. The orange dots represent map features. The dotted-curved arrows represent the relative poses or positions between different frames (for example, the dotted-curved arrow between a map feature and frame $G$ means the position of the map feature in the frame $G$). The dotted-straight lines represent the observation of $KF$ or $C$ toward map features. The black arrows/lines mean that they are fixed variables and should be independent of a transformation (the purple curved-arrow) applied to the system, while the grey arrows represent the variables to be estimated and will be changed when applying a transformation to the system.

\begin{figure}[t!]
    \centering
    \setlength{\abovecaptionskip}{0cm}
    \includegraphics[width=1\linewidth]{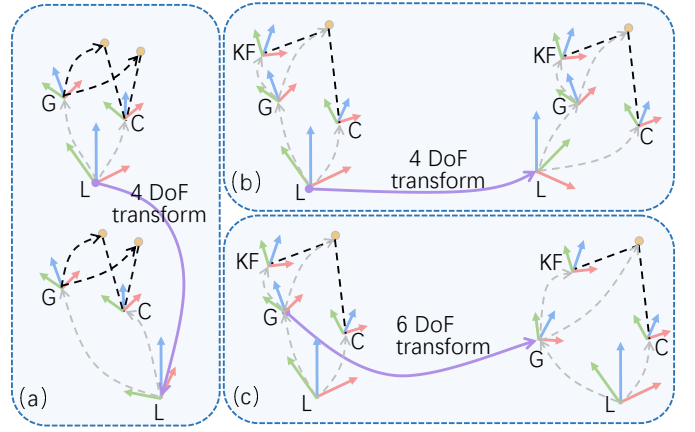}
     \caption{Intuitive explanation toward observability. (a) For the \textit{perfect augmented system}, applying a 4 DoF transformation to the frame $L$ does not change the observations of $C$ and map features positions in $G$. (b) For the \textit{imperfect augmented system}, applying a 4 DoF transformation to frame $L$ does not change the observations of $KF$ and $C$. (c) For the \textit{imperfect augmented system}, applying a 6 DoF transformation to frame $G$ does not change the observations of $KF$ and $C$.}
     \label{fig:intuitive explanation}
     \vspace{-0.6cm}
\end{figure}

\textbf{Perfect augmented system:} For the ideal case, the \textit{perfect augmented system} has the same unobservable dimensions as the standard visual-inertial system. This result is not surprising. We can understand the proposed result by regarding the \textit{augmented variable} as another 6 DoF pose feature in the odometry frame and never being marginalized. In this way, the \textit{perfect augmented system} is equivalent to a visual-inertial SLAM system, demonstrating similar properties. Fig. \ref{fig:intuitive explanation} (a) depicts the four dimensions of the \textit{perfect augmented system} that are not observable. When a 4 DoF transformation (one for the rotation along the gravity axis and three for the translation) is applied to the frame $L$, the pose of $C$ and $G$ in frame $L$ can be adjusted (see the changed gray dotted-curved arrows in Fig. \ref{fig:intuitive explanation} (a)) so that the observation of $C$ (the black dotted-straight lines) and the map feature positions (the black dotted-curved arrows) does not change. As the system is gravity aligned, the dimension of the applied transformation is four (three dimensions for translation, one dimension for rotation along the gravity axis), which is the dimension of the unobservable subspace. The similar conclusion can also be find in \cite{gpsvio}, where the authors argue that even with GPS-based global measurements, the unobservable subspace of the system (i.e., the \textit{perfect augmented system} in this paper) has 4 dimensions. 

Following this idea, for real case, when standard EKF is employed, the observability deficiency of the \textit{perfect augmented system} would be the same as the visual-inertial SLAM system, i.e., along one direction (rotation around gravity direction)\cite{consistent2}, which is consistent with \textbf{Theorem 2}. However, when invariant EKF is employed, the correct observability can be maintained naturally as indicated by \textbf{Theorem 3}, which is according with other invariant EKF based VIO systems\cite{consistent3,consistent4}.

\textbf{Imperfect augmented system:} For the \textit{imperfect augmented system}, the map-related parts can be treated as another visual SLAM system where keyframe poses and features are considered, and the dimensions of the unobservable subspace of the visual SLAM system are six. From this point of view, it is easy to understand why the \textit{imperfect augmented system} has another six unobservable directions compared with the \textit{perfect augmented system}. Fig. \ref{fig:intuitive explanation} (b) and (c) provide good insights toward this conclusion. In Fig. \ref{fig:intuitive explanation} (b), we apply a 4 DoF transformation (one for the rotation along the gravity axis and three for the translation) to frame $L$ while keeping the observations of $KF$ and $C$ unchanged. In Fig. \ref{fig:intuitive explanation} (c), a 6 DoF transformation is applied to frame $G$, while by adjusting the pose of $KF$ and the position of map features in frame $G$, the observations of $KF$ and $C$ are unchanged.

\section{\textcolor{black}{Schmidt invariant Kalman filter with multi-state and observability constraints}}

% \section{Multi-state constraint Schmidt invariant Kalman filter}
From Sec. \ref{sec:IEKF perfect} and Sec. \ref{sec:ob perfect ikf}, we have an invariant EKF algorithm that can naturally maintain the consistency of the \textit{perfect augmented system} with the theoretical guarantee. However, it is not enough for practical application. On the one hand, as map information is not perfect, it is crucial to fuse the map information consistently and efficiently --- the constructed framework should consider the uncertainty of the map information while keeping the cost of the computation and the storage at a low level. On the other hand, only maintaining the robot pose at the current timestamp is not sufficient, because in this way informative historical data fail to constrain the estimation of the current state --- a sliding window based (multi-state constrained) technique is preferred to get more accurate estimation.
From these two aspects, we propose a multi-state observability constrained Schmidt invariant Kalman filter (MSOC-S-IKF).

% From Sec. \ref{sec:IEKF perfect} and Sec. \ref{sec:ob perfect ikf}, we have an invariant EKF algorithm that can naturally maintain the consistency of the \textit{perfect augmented system} with the theoretical guarantee. However, it is not enough for practical application. On the one hand, only maintaining the robot pose at the current timestamp is not sufficient, because in this way informative historical data fail to constrain the estimation of the current state --- a sliding window based (multi-state constraint) technique is preferred to get more accurate estimation. On the other hand, as map information is not perfect, it is crucial to fuse the map information consistently and efficiently --- the constructed framework should consider the uncertainty of the map information while keeping the cost of the computation and the storage at a low level.
% From these two aspects, we propose a multi-state observability constraint Schmidt invariant Kalman filter (MSOC-S-IKF). 

\subsection{Observability constrained Schmidt invariant filter for imperfect map} \label{sec:oc}
In this part, to integrate the map information consistently and efficiently, Schmidt filter is introduced to consider the uncertainty of the map information efficiently, while a observability constrained technique is proposed to maintain the correct observability of the \textit{imperfect augmented system}. 

\textbf{Schmidt invariant filter:} To consider the uncertainty of the map information, we need to add map-related variables into the state. If we put all the map features, the number of which can be tens of thousands, into the state of the \textit{augmented system} like (\ref{eq:imperfect state}), it requires significant computation to update state and large storage to record covariance. Therefore, we keep the map keyframe poses and their observations toward features instead. In this way, we can perform the null space projection toward map features like (\ref{eq:ob_msckf_final}) so that the uncertainty of the map feature can be taken into consideration. 

Based on the above idea, the state of the \textit{augmented system} $\mathbf{X}_t$ in (\ref{eq:state}) will be augmented as
\begin{equation}\label{eq:state_sch}
    \mathbf{X}_{SCH_t}=(\mathbf{X}_{A_t},\mathbf{B}_{t},\mathbf{X}_{KF_t}),
\end{equation}
where $\mathbf{X}_{KF_t}$ contains matched map keyframe poses $\{\mathbf{X}_{KF_1,t},\cdots,\mathbf{X}_{KF_m,t}\}$ defined in (\ref{eq:imperfect state}).

For brevity, we divide (\ref{eq:state_sch}) into two parts: the active part $\mathbf{X}_{a_t}\triangleq (\mathbf{X}_{A_t},\mathbf{B}_t)$ and the nuisance part $\mathbf{X}_{n_t}\triangleq\mathbf{X}_{KF_t}$. This partition is very useful for Schmidt-EKF updating, as illustrated in the following.

When there is a matching between the current frame and a map keyframe, we have the observation function (\ref{eq:imperfect global_ob_c}) and (\ref{eq:imperfect global_ob_kf}). Linearizing these two functions, we have
\begin{equation}\label{eq:ob_error_map_1}
\begin{aligned}
     \mathbf{r}_{G_t}^{C_j}&=\mathbf{y}_{G_t}^{C_j}-\hat{\mathbf{y}}_{G_t}^{C_j}\\
     &=-\mathbf{H}_{a_t}^{C_j}\boldsymbol{\epsilon}_{a_t}-\mathbf{H}_{F_j,t}^{C_j}\boldsymbol{\epsilon}_{F_j,t}+\boldsymbol{\gamma}_{SCH_t}^{j},
\end{aligned}
\end{equation}
\begin{equation}\label{eq:ob_error_map_2}
\begin{aligned}
     \mathbf{r}_{G_t}^{KF_ij}&=\mathbf{y}_{G_t}^{KF_ij}-\hat{\mathbf{y}}_{G_t}^{KF_ij}\\
     &=-\mathbf{H}_{n_t}^{KF_ij}\boldsymbol{\epsilon}_{n_t}-\mathbf{H}_{F_j,t}^{KF_ij}\boldsymbol{\epsilon}_{F_j,t}+\boldsymbol{\gamma}_{SCH_i}^{j}.
\end{aligned}
\end{equation}
where $\boldsymbol{\epsilon}_{a_t}$ and $\boldsymbol{\epsilon}_{n_t}$ are the errors of the active part and the nuisance part, respectively,
$\boldsymbol{\epsilon}_{F_j,t}$ is the error of the map feature $F_j$, and $\boldsymbol{\gamma}_{SCH_i}^{j}$ and $\boldsymbol{\gamma}_{SCH_t}^{j}$ are observation noises.
The Jacobian matrices $\mathbf{H}_{a_t}^{C_j}$ and $\mathbf{H}_{F_j,t}^{C_j}$ form the Jacobian matrix defined by (\ref{eq:imperfect jacob c}). $\mathbf{H}_{n_t}^{KF_ij}$ and $\mathbf{H}_{F_j,t}^{KF_ij}$ form the Jacobian matrix defined by (\ref{eq:imperfect jacob kf}). For the sake of keeping the correct observability of the \textit{imperfect augmented system}, we employ the observability constrained (OC) technique to modify the value of these observation Jacobian matrices. The details of the OC technique will be given later.

By stacking (\ref{eq:ob_error_map_1}) and (\ref{eq:ob_error_map_2}), and performing null space projection toward $F_j$, we have
\begin{equation}\label{eq:stacking obs}
    \begin{aligned}
      \mathbf{r}_{G_t}^{j}&=-\mathbf{H}_{a_t}^{j}\boldsymbol{\epsilon}_{a_t}-\mathbf{H}_{n_t}^{j}\boldsymbol{\epsilon}_{n_t}-\mathbf{H}_{F_j,t}^{j}\boldsymbol{\epsilon}_{F_j,t}+\boldsymbol{\gamma}_{SCH}^{j}\\
      \mathcal{N}_{F_j}\mathbf{r}_{G_t}^{j}&=-\mathcal{N}_{F_j}\mathbf{H}_{a_t}^{j}\boldsymbol{\epsilon}_{a_t}-\mathcal{N}_{F_j}\mathbf{H}_{n_t}^{j}\boldsymbol{\epsilon}_{n_t}-\mathcal{N}_{F_j}\mathbf{H}_{F_j,t}^{j}\boldsymbol{\epsilon}_{F_j,t}\\&\quad+\mathcal{N}_{F_j}\boldsymbol{\gamma}_{SCH}^{j}\\
      \mathbf{r}_{G_t}^{j\prime}&=-\mathbf{H}_{a_t}^{j\prime}\boldsymbol{\epsilon}_{a_t}-\mathbf{H}_{n_t}^{j\prime}\boldsymbol{\epsilon}_{n_t}+\boldsymbol{\gamma}_{SCH}^{j\prime}\\
      &=-\begin{bmatrix}
      \mathbf{H}_{a_t}^{j\prime}&\mathbf{H}_{n_t}^{j\prime}
      \end{bmatrix}\begin{bmatrix}
      \boldsymbol{\epsilon}_{a_t}\\
      \boldsymbol{\epsilon}_{n_t}
      \end{bmatrix}+\boldsymbol{\gamma}_{SCH}^{j\prime}\\
      &\triangleq -\mathbf{H}_{X_{SCH_t}}^{j}\boldsymbol{\epsilon}_{X_{SCH_t}}+\boldsymbol{\gamma}_{SCH}^{j\prime},
    \end{aligned}
\end{equation}
where $\mathbf{r}_{G_t}^{j}$ is the stacking vector of $\mathbf{r}_{G_t}^{C_j}$ and $ \mathbf{r}_{G_t}^{KF_ij}$, $\mathbf{H}_{a_t}^{j}$, $\mathbf{H}_{n_t}^{j}$, $\mathbf{H}_{F_j,t}^{j}$ and $\boldsymbol{\gamma}_{SCH}^{j}$ are with the similar formulations, and $\mathcal{N}_{F_j}$ is the left null space of $\mathbf{H}_{F_j,t}^{j}$. Note that (\ref{eq:stacking obs}) only considers one map feature. In practice, at timestamp $t$, there would be more than one map feature that can be observed. By stacking these observation functions together, we have:
\begin{equation}\label{eq: final schmidt ob}
     \mathbf{r}_{G_t}^{\prime}= -\mathbf{H}_{X_{SCH_t}}\boldsymbol{\epsilon}_{X_{SCH_t}}+\boldsymbol{\gamma}_{SCH}^{\prime},
\end{equation}
which is used to perform Schmidt updating.

Since the state $\mathbf{X}_{SCH_t}$ can be divided as $\mathbf{X}_{SCH_t}=(\mathbf{X}_{a_t},\mathbf{X}_{n_t})$, its covariance can also be partitioned as
\begin{equation}
    \mathbf{P}_t=\begin{bmatrix}
    \mathbf{P}_{aa_t}&\mathbf{P}_{an_t}\\
    \mathbf{P}_{na_t}&\mathbf{P}_{nn_t}\\
    \end{bmatrix}.
\end{equation}
 Then, the following equations are employed to update state, which is so-called Schmidt updating:
 \begin{align}
     \mathbf{S}_{t}&=\mathbf{H}_{X_{SCH_t}}\mathbf{P}_t\mathbf{H}_{X_{SCH_t}}^{\top}+\mathbf{V}_{SCH_t}^{\prime},\\ 
     \mathbf{K}_t&=\left[\begin{matrix}\mathbf{K}_{\mathbf{X}_{{A}_{t}}}\\\mathbf{K}_{\mathbf{B}_t}\\\mathbf{K}_{\mathbf{X}_{{n}_{t}}}\end{matrix}\right]=\left[\begin{matrix}\mathbf{K}_{\mathbf{X}_{a_t}}\\\mathbf{K}_{\mathbf{X}_{n_t}}\end{matrix}\right]\notag\\ &= \left[\begin{matrix}\mathbf{P}_{aa_{t|t-1}}\mathbf{H}_{a_t}^{\prime\top} + \mathbf{P}_{an_{t|t-1}}\mathbf{H}_{n_t}^{\prime\top}\\\mathbf{P}_{na_{t|t-1}}\mathbf{H}_{a_t}^{\prime\top} + \mathbf{P}_{nn_{t|t-1}}\mathbf{H}_{n_t}^{\prime\top}\end{matrix}\right]\mathbf{S}_{t}^{-1},
    %  =\left[\begin{matrix}\bar{\mathbf{K}}_{\mathbf{X}_{a_t}}\\\bar{\mathbf{K}}_{\mathbf{X}_{n_t}}\end{matrix}\right]\mathbf{S}_{t}^{-1},
\end{align}
\vspace{-0.5cm}
\begin{align}
    &\mathbf{P}_{t}=\mathbf{P}_{t|t-1}- \notag\\
      &\left[\begin{matrix} \mathbf{K}_{\mathbf{X}_{a_t}}\mathbf{S}_t\mathbf{K}_{\mathbf{X}_{a_t}}^{\top}& \mathbf{K}_{\mathbf{X}_{a_t}}\mathbf{H}_{X_{SCH_t}}\left[\begin{matrix} \mathbf{P}_{an_{t|t-1}}\\\mathbf{P}_{nn_{t|t-1}}
    \end{matrix}\right]\\
    \left[\begin{matrix} \mathbf{P}_{an_{t|t-1}}\\\mathbf{P}_{nn_{t|t-1}}\\
    \end{matrix}\right]^{\top} \mathbf{H}_{X_{SCH_t}}^{\top} \mathbf{K}_{\mathbf{X}_{a_t}}^{\top}& \mathbf{0}\end{matrix}\right],\label{eq:sch_p}\\
     &\hat{\mathbf{X}}_{a_t}=\begin{bmatrix}
     \hat{\mathbf{X}}_{A_{t}}\\
     \hat{\mathbf{B}}_{t}
     \end{bmatrix}=\begin{bmatrix}
     \mathrm{exp}(\mathbf{K}_{\mathbf{X}_{A_t}}\mathbf{r}_{SCH_t}^{\prime})\hat{\mathbf{X}}_{A_{t|t-1}}\\
      \hat{\mathbf{B}}_{t|t-1}+\mathbf{K}_{\mathbf{B}_t}\mathbf{r}_{SCH_t}^{\prime}
     \end{bmatrix},\\
    &\qquad\qquad\hat{\mathbf{X}}_{n_t}=\hat{\mathbf{X}}_{n_{t|t-1}},
 \end{align}
where $\mathbf{V}_{SCH_t}^{\prime}$ is the covariance of the noise $\boldsymbol{\gamma}_{SCH}^{\prime}$ in (\ref{eq: final schmidt ob}).

We can see that the active part is updated in the form of invariant EKF as shown in Algorithm \ref{alg:invariant_EKF}, while the nuisance part (map keyframes) remains unchanged.

% Besides, it is worth noting that the updated covaraince from (\ref{eq:sch_p}), denoted as $\mathbf{P}_{SCH}$, is not overconfident because $\mathbf{P}_{SCH}=\mathbf{P}_{KF}+\left[\begin{matrix}\mathbf{0}\\ \bar{\mathbf{K}}_{n}\end{matrix}\right] \mathbf{S}^{-1} \left[\begin{matrix}\mathbf{0}& \bar{\mathbf{K}}_{n}^{\top}\end{matrix}\right]\geqslant\mathbf{P}_{KF}$, where $\mathbf{P}_{KF}$ is the covariance derived from the standard Kalman filter update. In this way, the uncertainty of the map keyframe pose and the map feature is considered, so that the map information is integrated consistently.

\textbf{Observability maintenance:} Reviewing \textbf{Theorem \ref{theorem: ideal imperfect iekf}} and \textbf{Theorem \ref{theorem: real imperfect iekf}}, we can find that due to the ever-changing estimated values of $^{L}\mathbf{R}_{G_t}$ and $^{G}\mathbf{p}_{F_t}$, the last six columns of (\ref{eq:null_space_om_inekf imperfect ideal}) vanish for the real case. In the framework of Schmidt updating, the value of $^{G}\mathbf{p}_{F_t}$ stays the same, whereas the value of $^{L}\mathbf{R}_{G_t}$ is not the case. An simple and intuitive way to maintain the correct dimension of the system unobservable subspace is to define the null space of the observability matrix (\ref{eq:null_space_om_inekf imperfect ideal}) based on the first-estimated value of $^{L}\mathbf{R}_{G_t}$, i.e., $^{L}\hat{\mathbf{R}}_{G_0}$, and compute the Jacobian matrix (\ref{eq:imperfect jacob c}) with  $^{L}\hat{\mathbf{R}}_{G_0}$. However, as indicated by \cite{oc}, the performance of this solution relies heavily on $^{L}\hat{\mathbf{R}}_{G_0}$ because $^{L}\mathbf{R}_{G_0}$ is used at all time steps to compute Jacobians. If $^{L}\hat{\mathbf{R}}_{G_0}$ is far from the true value, the system's performance would be degraded. Therefore, we design an OC technique following the idea of \cite{oc} to search the optimal values of Jacobians while maintaining the correct unobservable subspace. 

The key idea of the OC technique is to select a suitable null space $^{in}\mathcal{N}^{*}_3$ of the observability matrix $^{in}\mathbf{M}^{*}$ and suitable Jaocbians $^{in}\mathbf{H}_{L_i}^{*}$, $^{in}\mathbf{H}_{G_i}^{*}$ and $^{in}\boldsymbol{\Phi}_{i|0}^{*}$ (cf. (\ref{om})) so as to fulfill the following conditions\cite{oc design}:
\begin{equation}\label{eq:condition1}
     {}^{in}\mathbf{H}_{L_i}^{*}{}^{in}\mathcal{N}^{*}_3=\mathbf{0},
\end{equation}
\begin{equation}\label{eq:condition2}
       {}^{in}\mathbf{H}_{G_i}^{*}{}^{in}\mathcal{N}^{*}_3=\mathbf{0},
\end{equation}
\begin{equation}\label{eq:condition3}
        {}^{in}\mathcal{N}^{*}_3={}^{in}\boldsymbol{\Phi}_{i|0}^{*}{}^{in}\mathcal{N}^{*}_3.
\end{equation}

A natural and realizable choice of $^{in}\mathcal{N}^{*}_3$ (cf. (\ref{eq:null_space_om_inekf imperfect ideal})) is given by
\begin{equation}\label{eq:null_space_om_inekf imperfect real oc}
    {}^{in}\mathcal{N}^{*}_{3}=span\left(\begin{bmatrix}\mathbf{g}&\mathbf{0}_{3}&\mathbf{0}_3&\mathbf{0}_3\\
    \mathbf{0}_{3\times1}&\mathbf{0}_{3}&\mathbf{0}_3&\mathbf{0}_3\\
    \mathbf{0}_{3\times1}&\mathbf{I}_{3}&\mathbf{0}_3&\mathbf{0}_3\\
    \mathbf{0}_{3\times1}&\mathbf{I}_{3}&\mathbf{0}_3&\mathbf{0}_3\\
    \mathbf{0}_{3\times1}&\mathbf{I}_{3}&-{}^{L}\hat{\mathbf{R}}_{G_0}&\mathbf{0}_3\\
    \mathbf{g}&\mathbf{0}_{3}&\mathbf{0}_3&\mathbf{I}_3\\
    \mathbf{0}_{3\times1}&\mathbf{0}_{3}&\mathbf{0}_3&-{}^{L}\hat{\mathbf{R}}_{G_0}^{\top}\\
    \mathbf{0}_{3\times1}&\mathbf{0}_{3}&\mathbf{I}_3&\mathbf{0}_3\\
    \mathbf{0}_{3\times1}&\mathbf{0}_{3}&\mathbf{I}_3&(^{G}\mathbf{p}_{F})_{\times}{}^{L}\hat{\mathbf{R}}_{G_0}^{\top}\\
    \end{bmatrix}\right).
\end{equation}

\textcolor{black}{It is worth mentioning that the OC technique is a general and useful way to make the estimator consistent. Even for the standard EKF, as long as the conditions (\ref{eq:condition1})--(\ref{eq:condition3}) are satisfied, the whole system can be consistent. However, we argue that a good estimator should introduce as few artificial-designed constraints as possible. Therefore, we still choose to formulate the state error in the nonlinear form as (\ref{eq:nonlinear_error}) instead of the standard error form. In this way, the special observability constraint design for the transformation matrix $^{in}\boldsymbol{\Phi}_{i|0}^{*}$ and the local observation Jacobian matrix $^{in}\mathbf{H}_{L_i}^{*}$ can be avoided. Readers can verify that, 
with (\ref{eq:local_H}) and $^{in}\boldsymbol{\Phi}_{i|0}^{*}$ (the detailed expression is given in \textbf{Appendix F} of the supplementary material \cite{supplementary} ), (\ref{eq:condition1}) and (\ref{eq:condition3}) are naturally satisfied. Therefore, we only need to design an OC technique satisfying (\ref{eq:condition2}).} Further, since ${}^{in}\mathbf{H}_{G_i}^{*}$ consists of (\ref{eq:imperfect jacob c}) and (\ref{eq:imperfect jacob kf}), and (\ref{eq:imperfect jacob kf}) automatically meets the condition (\ref{eq:condition2}), our problem is simplified to modifying (\ref{eq:imperfect jacob c}) such that the linearization error toward (\ref{eq:imperfect global_ob_c}) is minimum while the observability of the \textit{imperfect augmented system} is correctly maintained:
\begin{equation}\label{eq:simplified condition}
    \min_{\mathbf{H}_{G_i}^{C^{*}}} ||\mathbf{H}_{G_i}^{C^{*}}-\mathbf{H}_{G_i}^{C}||_{\mathcal{F}}^{2}, \text{ s.t. }  \mathbf{H}_{G_i}^{C^{*}}{}^{in}\mathcal{N}^{*}_3=\mathbf{0},
\end{equation}
where $||\cdot||_{\mathcal{F}}$ denotes the Frobenius matrix norm. $\mathbf{H}_{G_i}^{C}$ is calculated by (\ref{eq:imperfect jacob c}) with only one map feature and one map keyframe being considered. The optimal $\mathbf{H}_{G_i}^{C^{*}}$ that fulfills (\ref{eq:simplified condition}) is given by
\begin{equation}\label{eq:optimal imperfect jacob c}
    \mathbf{H}_{G_i}^{C^{*}} = \mathbf{H}_{G_i}^{C}-\mathbf{H}_{G_i}^{C}{}^{in}\mathcal{N}^{*}_3({}^{in}\mathcal{N}^{*^{\top}}_3{}^{in}\mathcal{N}^{*}_3)^{-1}{}^{in}\mathcal{N}^{*^{\top}}_3.
\end{equation}
Replacing (\ref{eq:imperfect jacob c}) with (\ref{eq:optimal imperfect jacob c}) while computing Jacobians, the correct unobservable subspace ${}^{in}\mathcal{N}^{*}_3$ of $^{in}\mathbf{M}^{*}$ is preserved.

\textbf{Complexity analysis:}
Suppose a map has $m$ keyframes and $n$ features ($m \ll n$). In terms of storage, if we maintain all the features of the map, then the complexity of the storage will be $\mathcal{O}(n^{2})$. Instead, if we maintain the map keyframe poses as mentioned before, the complexity of the storage will be $\mathcal{O}(m^{2}+n)$, which is much more storage-saving. 

In terms of computation, the main difference between the invaraint EKF and the invaraint Schmidt EKF is the covariance update (\ref{eq:sch_p}). For invariant EKF, $\mathbf{0}$ in (\ref{eq:sch_p}) is replaced by $\mathbf{K}_{\mathbf{X}_{n_t}}\mathbf{S}_{t}\mathbf{K}_{\mathbf{X}_{n_t}}^{\top}$, whose computational complexity is quadratic of the size of nuisance part ($\mathcal{O}(m^2)$). This limits the real-time performance of the \textit{augmented system}. In contrast, the computational complexity of (\ref{eq:sch_p}) is linear ($\mathcal{O}(m)$).
\vspace{-2mm}

\subsection{Multi-state constrained invariant filter}\label{sec:msc}
Multi-state constrained filter is designed for improving the performance of visual-inertial odometry\cite{msckf}, and in our system it is applied to the local feature based observation function (cf.(\ref{eq:ob_local})). 

\textbf{State with multi-state constraint:} Suppose we maintain a sliding window with the size of $s$, where the historical IMU poses are saved. Then, the state of the \textit{augmented system} $\mathbf{X}_{t}$ in (\ref{eq:state}) will be augmented by cloned poses as
\begin{equation}
    \label{eq:state_with_clone}
    \mathbf{X}_{MSC_t}=(\mathbf{X}_{A_t},\mathbf{B}_{t}, {}^{I}\mathbf{T}_{C_t},\mathbf{X}_{Clone_t}),
\end{equation}
where $^{I}\mathbf{T}_{C_t}$ is the extrinsic between the camera and the IMU composed of $\{{}^{L}\mathbf{R}_{C_t},{}^{L}\mathbf{p}_{C_t}\}$. We estimate this variable online for better localization results. $\mathbf{X}_{Clone_t}$ contains cloned $s$ IMU poses $\{\mathbf{X}_{CP_t},\cdots,\mathbf{X}_{CP_t-s+1}\}$. For each cloned IMU pose $\{\mathbf{X}_{CP_i}\}$, it consists of  $\{{}^{L}\mathbf{R}_{I_i}, {}^{L}\mathbf{p}_{I_i}\}$. Following the definition of (\ref{eq:nonlinear_error}), the \textit{right-invariant error} of $^{I}\mathbf{T}_{C_t}$ and $\mathbf{X}_{CP_i}$ are given as
\begin{equation}
    \begin{aligned}
     \boldsymbol{\xi}_{\boldsymbol{\theta}_{IC_t}}&=\tilde{\boldsymbol{\theta}}_{IC_t},\; \boldsymbol{\xi}_{\boldsymbol{\theta}_{LI_i}}=\tilde{\boldsymbol{\theta}}_{LI_i},\\
     \boldsymbol{\xi}_{\mathbf{p}_{IC_t}}&={}^{I}\hat{\mathbf{p}}_{C_t}-(\mathbf{I}_{3}+(\tilde{\boldsymbol{\theta}}_{IC_t})_{\times}){}^{I}\mathbf{p}_{C_t},\\
     \boldsymbol{\xi}_{\mathbf{p}_{LI_i}}&={}^{L}\hat{\mathbf{p}}_{I_i}-(\mathbf{I}_{3}+(\tilde{\boldsymbol{\theta}}_{LI_i})_{\times}){}^{L}\mathbf{p}_{I_i}.\\
    \end{aligned}
\end{equation}

\textbf{Observation function with multi-state constraint:} Supposing there is a local feature $f$ tracked by multiple cloned poses being used to perform invariant EKF update, for each observation from cloned poses, we can formulate the observation function as
\begin{equation} \label{eq:msckf_ob}
\begin{aligned}
     \mathbf{y}\!_{MSC_i} \!\!=\!\! h[({}^{L}\mathbf{R}_{I_i}{}^{I}\mathbf{R}_{C_t})\!\!^{\top}\!({}^{L}\mathbf{p}_{f}\!-\!({}^{L}\mathbf{R}_{I_i}{}^{I}\mathbf{p}_{C_t}\!+\!{}^{L}\mathbf{p}_{I_i}))]
     \!+\!\boldsymbol{\gamma}\!_{MSC_i},
\end{aligned}
\end{equation}
where $\mathbf{y}_{MSC_i}$ is the measurement of feature $f$ in the image captured by $\mathbf{X}_{CP_i}$, and $\boldsymbol{\gamma}_{MSC_i}$ is the observation noise. Linearizing above function like (\ref{eq:general_linear_ob_function}), we can get
\begin{equation}\label{eq:msckf_reproject_error}
    \begin{aligned}
     \mathbf{r}_{MSC_i}&=\mathbf{y}_{MSC_i}-\hat{\mathbf{y}}_{MSC_i}\\
     &=-\mathbf{H}_{MSC_{x_i}}\boldsymbol{\epsilon}_{MSC_{x_i}} -\mathbf{H}_{f}\boldsymbol{\epsilon}_{f}+\boldsymbol{\gamma}_{MSC_i}.
    \end{aligned}
\end{equation}
where $\mathbf{r}_{MSC_i}$ is the re-projection error of $f$ under $\mathbf{X}_{CP_i}$. $\mathbf{H}_{MSC_{x_i}}$ is the Jacobian matrix of the observation function with respect to all related state except the observed feature $f$. $\mathbf{H}_{f}$ is the Jacobian matrix of the observation function with respect to the observed feature $f$. $\boldsymbol{\epsilon}_{MSC_{x_i}}$ and $\boldsymbol{\epsilon}_{f}$ are the \textit{right-invariant error} of the state and feature, respectively.

Noting that in (\ref{eq:msckf_reproject_error}), although the error of state, $\boldsymbol{\epsilon}_{MSC_{x_i}}$, is related to different cloned poses, the error of the same feature $f$, $\boldsymbol{\epsilon}_{f}$, should be identical. Recall (\ref{eq:nonlinear_error}) that the feature error is defined as $\boldsymbol{\epsilon}_{f}=\boldsymbol{\xi}_{f}= {}^{L}\hat{\mathbf{p}}_{f} - \hat{\mathbf{R}}\mathbf{R}^{\top}{}^{L}\mathbf{p}_{f}$, where $\mathbf{R}$ is a rotation matrix. This means that we need to bind a rotation matrix to each feature's error. In this paper, we assume this rotation matrix is the rotation part of the current IMU cloned pose $\mathbf{X}_{CP_t} \triangleq \{^{L}\mathbf{R}_{I_t},^{L}\mathbf{p}_{I_t}\}$, and we call this current IMU cloned pose as the \textit{anchored frame}. Then, there are two cases need to be considered:

\noindent \textit{a)} The feature $f$ is observed by the \textit{anchored frame} $\mathbf{X}_{CP_t}$:
\begin{equation}\label{eq:reproject_anchor}
\begin{aligned}
     \mathbf{r}_{MSC_t} &=\mathbf{y}_{MSC_t}-\hat{\mathbf{y}}_{MSC_t}\\
    &\approx \nabla h |_{\hat{\mathfrak{q}}} 
     \left[({}^{L}\hat{\mathbf{R}}_{I_t} {}^{I}\hat{\mathbf{R}}_{C_t})^{\top}
     \left[
     -\boldsymbol{\xi}_{f}
     +{}^{L}\hat{\mathbf{R}}_{I_t}\boldsymbol{\xi}_{\mathbf{p}_{IC_t}}
     +\boldsymbol{\xi}_{\mathbf{p}_{LI_t}}\right.\right.\\&
     \left.\left.
     +({}^{L}\hat{\mathbf{p}}_{I_t}-{}^{L}\hat{\mathbf{p}}_{f})_{\times}{}^{L}\hat{\mathbf{R}}_{I_t}\boldsymbol{\xi}_{\boldsymbol{\theta}_{IC_t}}
     \right]
     \right]+\boldsymbol{\gamma}_{MSC_t}.\\ 
\end{aligned}
\end{equation}

\noindent \textit{b)} The feature $f$ is observed by the other cloned frame $\mathbf{X}_{CP_i}$, $i=t-1,\cdots,t-s+1$:
\begin{equation}\label{eq:reproject_other}
    \begin{aligned}
     \mathbf{r}_{MSC_i} &=\mathbf{y}_{MSC_i}-\hat{\mathbf{y}}_{MSC_i}\\
    &\approx \nabla h |_{\hat{\mathfrak{q}}} 
     \left[({}^{L}\hat{\mathbf{R}}_{I_i} {}^{I}\hat{\mathbf{R}}_{C_t})^{\top}
     \left[-\boldsymbol{\xi}_{f}+({}^{L}\hat{\mathbf{p}}_{f})_{\times}\boldsymbol{\xi}_{\boldsymbol{\theta}_{LI_t}}\right.\right.\\&-
     ({}^{L}\hat{\mathbf{p}}_{f})_{\times}\boldsymbol{\xi}_{\boldsymbol{\theta}_{LI_i}}+
     ({}^{L}\hat{\mathbf{p}}_{I_i}-{}^{L}\hat{\mathbf{p}}_{f})_{\times}{}^{L}\hat{\mathbf{R}}_{I_i}\boldsymbol{\xi}_{\boldsymbol{\theta}_{IC_t}}\\
     &\left.\left.
     +{}^{L}\hat{\mathbf{R}}_{I_i}\boldsymbol{\xi}_{\mathbf{p}_{IC_t}}
     +\boldsymbol{\xi}_{\mathbf{p}_{LI_i}}
     \right]
     \right]+\boldsymbol{\gamma}_{MSC_i}.\\
    \end{aligned}
\end{equation}
Combining (\ref{eq:reproject_anchor}) and (\ref{eq:reproject_other}), and stacking them together, we have the following expression:

    \begin{align}
    \mathbf{r}_{MSC_f}&=-\mathbf{H}_{MSC_x}\boldsymbol{\epsilon}_{MSC_x}-\mathbf{H}_{f}\boldsymbol{\epsilon}_{f}+\boldsymbol{\gamma}_{MSC}\notag\\
     \begin{bmatrix}
     \mathbf{r}_{MSC_t}\\
     \vdots\\
     \mathbf{r}_{MSC_i}\\
     \vdots\\
     \end{bmatrix}&=
     \begin{bmatrix}
        \boldsymbol{\nabla}_t[
        \mathbf{0}\quad\mathbf{I}\cdots\mathbf{0}\quad\mathbf{0}\cdots\mathbf{H}1_t\quad\mathbf{H}2_t
        ]\\
        \vdots\\
         \boldsymbol{\nabla}_i[
        \mathbf{H}3\quad\mathbf{0}\cdots\mathbf{H}4 \quad \mathbf{0}\cdots\mathbf{H}1_i\quad\mathbf{H}2_i
        ]\notag\\
        \vdots
     \end{bmatrix}\!\!\!
     \begin{bmatrix}
     \boldsymbol{\xi}_{\boldsymbol{\theta}_{LI_t}}\\
     \boldsymbol{\xi}_{\mathbf{p}_{LI_t}}\\
     \vdots\\
     \boldsymbol{\xi}_{\boldsymbol{\theta}_{LI_i}}\\
     \boldsymbol{\xi}_{\mathbf{p}_{LI_i}}\\
     \vdots\\
     \boldsymbol{\xi}_{\boldsymbol{\theta}_{IC_t}}\\
     \boldsymbol{\xi}_{\mathbf{p}_{IC_t}}
     \end{bmatrix}\\
     &+
     \begin{bmatrix}
     -\boldsymbol{\nabla}_t\\
     \vdots\\
     -\boldsymbol{\nabla}_i\\
     \vdots\\
     \end{bmatrix}
    \boldsymbol{\xi}_{f}
     +\begin{bmatrix}
     \boldsymbol{\gamma}_{MSC_t}\\
     \vdots\\
     \boldsymbol{\gamma}_{MSC_i}\\
     \vdots
     \end{bmatrix},
    \end{align}
where $i=t-1,\ldots,t-s+1$, $\boldsymbol{\nabla}_{t/i}= \nabla h |_{\hat{\mathfrak{q}}}({}^{L}\hat{\mathbf{R}}_{I_{t/i}} {}^{I}\hat{\mathbf{R}}_{C_t})^{\top}$, $\mathbf{H}1_{t/i}=({}^{L}\hat{\mathbf{p}}_{I_{t/i}}-{}^{L}\hat{\mathbf{p}}_{f})_{\times}{}^{L}\hat{\mathbf{R}}_{I_{t/i}}$, $\mathbf{H}2_{t/i}={}^{L}\hat{\mathbf{R}}_{I_{t/i}}$, $\mathbf{H}3=({}^{L}\hat{\mathbf{p}}_{f})_{\times}=-\mathbf{H}4$.

Following the method in \cite{msckf}, we do not maintain the feature $f$ in the state vector. Instead, we perform null space projection to merge the uncertainty of the feature into the uncertainty of the maintained state as
\begin{equation}\label{eq:ob_msckf_final}
    \begin{aligned}
      \mathbf{r}_{MSC_f} &= -\mathbf{H}_{MSC_x}\boldsymbol{\epsilon}_{MSC_x}-\mathbf{H}_{f}\boldsymbol{\epsilon}_f+\boldsymbol{\gamma}_{MSC}\\
      \mathcal{N}_f\mathbf{r}_{MSC_f} &= -\mathcal{N}_f\mathbf{H}_{MSC_x}\boldsymbol{\epsilon}_{MSC_x}-\mathcal{N}_f\mathbf{H}_{f}\boldsymbol{\epsilon}_f+\mathcal{N}_f\boldsymbol{\gamma}_{MSC}\\
      \mathbf{r}_{MSC_f}^{\prime} &= -\mathbf{H}_{MSC_x}^{\prime}\boldsymbol{\epsilon}_{MSC_x}+\boldsymbol{\gamma}_{MSC}^{\prime},\\
    \end{aligned}
\end{equation}
where $\mathcal{N}_f$ is the left null space of $\mathbf{H}_{f}$. With (\ref{eq:ob_msckf_final}), we can perform invariant EKF udpate.

It is worth mentioning that although cloned poses and the extrinsic are introduced into the state, the observability of the \textit{augmented system} remains unchanged, and the dimension of the unobservable subspace is still four. Readers can verify it following the definition of the observability matrix (\ref{om}).

\subsection{The flow of the whole algorithm}
The whole procedure of our proposed algorithm is summarized in Fig. \ref{fig:flow_chart}. The local VIO processes IMU and camera data to get the local pose (cf. Sec. \ref{sec:prop perfect} and Sec. \ref{sec:msc}). When there are feature matches between the pre-built visual map and the camera, the system will initialize the \textit{augmented variable} with the techniques of 3D-2D pose estimation e.g., EPnP\cite{epnp}, after which the \textit{augmented variable} and the poses of the matched map keyframes will be added into the state. Then, the system will utilize the measurements from the matched features to perform the map-based update in the manner of the Schmidt updating (cf. Sec. \ref{sec:oc}). The system will output the local pose and the \textit{augmented variable} at a frequency similar to the camera's frame rate. By multiplying the local pose and the \textit{augmented variable}, we can get the pose in the map (global) frame.
\vspace{-2mm}

\begin{figure}[t!]
    \centering
    \setlength{\abovecaptionskip}{0cm}
    \includegraphics[width=1\linewidth]{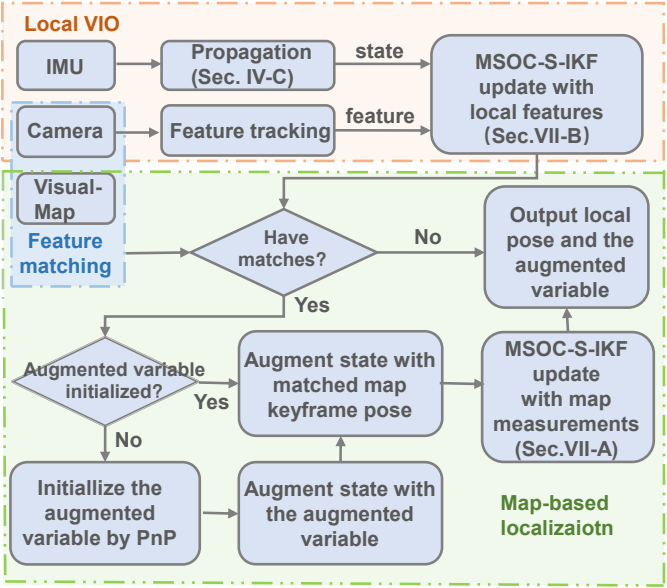}
     \caption{The flow chart of MSOC-S-IKF.}
     \label{fig:flow_chart}
     \vspace{-0.5cm}
\end{figure}

\section{Experimental results}
In this section, we perform extensive experiments with simulations and real world data. Before the details of the experiments, the evaluation metrics will be given first. 

We utilize the evaluation tools of Open-VINS \cite{openvins} to evaluate the performance of different algorithms. The accuracy of the localization can be measured by root mean squared error (RMSE)
or absolute trajectory error (ATE), the drift of the localization can be reflected by relative pose error (RPE)\cite{openvins,evaluation}, and the consistency of the system can be evaluated by Normalized Estimation Error Squared (NEES). For a given dataset with $N$ runs of the same algorithm and $K$ time steps for each run\footnote{With a slight abuse toward symbols, the definitions of $N$ and $K$ are different from those mentioned in earlier sections. These new definitions are only valid for this section.}, we define the estimated state value at time step $k$ of $i^{th}$ run as $\hat{\mathbf{x}}_{k,i}$, its corresponding true value as $\mathbf{x}_{k,i}$ and its corresponding covariance matrix as $\mathbf{P}_{k,i}$. Following \cite{openvins}, the detailed definitions of these 
four metrics are as follow:

\noindent \textbf{RMSE}:
 \begin{equation}
     e_{rmse,k}=\sqrt{\frac{1}{N}\sum_{i=1}^{N}||\mathbf{x}_{k,i} \boxminus \hat{\mathbf{x}}_{k,i}||^{2}_{2}},
 \end{equation}
 where $e_{rmse,k}$ is the RMSE of time step $k$, $\boxminus$ is the generalized subtraction. The RMSE of the whole trajectories can be calculated by:
 \begin{equation}
     e_{rmse}=\frac{1}{K}\sum_{k=1}^{K}\sqrt{\frac{1}{N}\sum_{i=1}^{N}||\mathbf{x}_{k,i} \boxminus \hat{\mathbf{x}}_{k,i}||^{2}_{2}}.
 \end{equation}
 
 \noindent \textbf{ATE}:
 \begin{equation}
     e_{ate}=\frac{1}{N}\sum_{i=1}^{N}\sqrt{\frac{1}{K}\sum_{k=1}^{K}||\mathbf{x}_{k,i} \boxminus \hat{\mathbf{x}}_{k,i}||^{2}_{2}}.
 \end{equation}
 
 \noindent \textcolor{black}{\textbf{RPE}:}
 \begin{equation}\label{eq:RPE}
     \begin{aligned}
          \tilde{\mathbf{x}}_{r,i}^{s}&=\mathbf{x}_{b,i}^{s}-\mathbf{x}_{e,i}^{s},\\
          e_{rpe,l}&=\frac{1}{N\times S}\sum_{i=1}^{N}\sum_{s=1}^{S}\sqrt{||\hat{\tilde{\mathbf{x}}}_{r,i}^{s}\boxminus\tilde{\mathbf{x}}_{r,i}^{s}||^{2}_{2}},
     \end{aligned}
 \end{equation}
 where we assume that the trajectory is split into $S$ segments, and each segment has the length of $l$. For each segment, it  corresponds to a pair of state $\{\mathbf{x}_{b,i}^{s},\mathbf{x}_{e,i}^{s}\}$ for $s=1\cdots S$, where $\mathbf{x}_{b,i}^{s}$ and $\mathbf{x}_{e,i}^{s}$ mean the beginning and the end states of this segment, respectively.
 
 \noindent \textbf{NEES}:
 \begin{equation}
 e_{nees,k} = \frac{1}{N\times d}\sum_{i=1}^{N}{
 (\mathbf{x}_{k,i} \boxminus \hat{\mathbf{x}}_{k,i})^{\top}\mathbf{P}_{k,i}^{-1}(\mathbf{x}_{k,i} \boxminus \hat{\mathbf{x}}_{k,i})},
 \end{equation}
 where $d$ is the dimension of $\mathbf{x}_{k,i}$. The NEES of the whole trajectories can be calculated by:
  \begin{equation}
 e_{nees} =\frac{1}{K\times N\times d}\sum_{k=1}^{K} \sum_{i=1}^{N}{
 (\mathbf{x}_{k,i} \boxminus \hat{\mathbf{x}}_{k,i})^{\top}\mathbf{P}_{k,i}^{-1}(\mathbf{x}_{k,i} \boxminus \hat{\mathbf{x}}_{k,i})}.
 \end{equation}
 This metric indicates the consistency of the system. If the system is consistent, NEES should be approximately 1 for large $N$. If NEES is much larger than (less than) 1, this means the system underestimates (overestimates) the covariance of the state. \textcolor{black}{It is worth noting that, to compute NEES, the estimated covariance derived from invariant-related algorithms is corresponding to the nonlinear error defined in (\ref{eq:nonlinear_error}) instead of the standard error in the vector space. However, for a fair comparison, RMSE, ATE, and RPE are computed with the standard error for all algorithms.} 
 
 \vspace{-2mm}
\subsection{Simulations}

\begin{figure}[t!]
    \centering
    \setlength{\abovecaptionskip}{0cm}
    \includegraphics[width=0.8\linewidth]{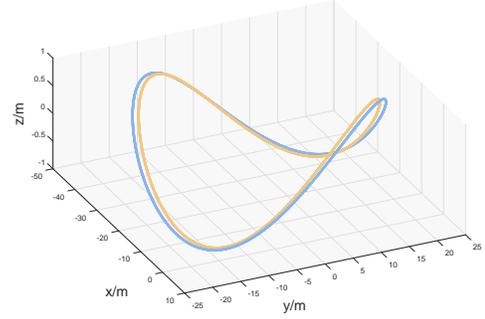}
     \caption{The trajectories used for simulation. The blue one is used to generate map information and the orange one is used to test algorithms.}
     \label{fig:sim_traj}
     \vspace{-0.5cm}
\end{figure}
In this part, we make some adaptations to the simulator of Open-VINS \cite{openvins} to evaluate the accuracy and consistency of \textit{augmented system} with different algorithms. We feed the simulator with a ground truth trajectory shaped like a ``saddle" whose length is about 630m (cf. orange trajectory of Fig. \ref{fig:sim_traj}). The map keyframe poses come from another trajectory, which has a similar shape as the previous trajectory (cf. blue trajectory of Fig. \ref{fig:sim_traj}). These two trajectories are denoted as S1 (the blue one) and S2 (the orange one), respectively. \textcolor{black}{To demonstrate the necessity of considering the uncertainty of the map, we conduct simulation experiments on a perfect map and an imperfect map. For the perfect map, we use the ground truth of map keyframe poses.} For the imperfect map, we artificially add the noise to the ground truth of S1 to simulate the imperfect but real map keyframes. The position of the ground truth is perturbed with the Gaussian white noise $\boldsymbol{\gamma}_{t_p} \sim \mathcal{N}(\mathbf{0},\sigma_p^{2}\mathbf{I}_{3\times3}), \sigma_p=0.1\,m$, and the orientation is perturbed with $\boldsymbol{\gamma}_{t_o} \sim \mathcal{N}(\mathbf{0},\sigma_o^2\mathbf{I}_{3\times3}),\sigma_o=0.9\,degree$. After the perturbation, the RMSE of the map trajectory S1 is $0.172\,m/1.566\,degree$.

For the map matching information, we first randomly generate 3D features, then we utilize the map-based observation functions (\ref{eq:imperfect global_ob_c}) and (\ref{eq:imperfect global_ob_kf}) to reproject 3D features into the frames obtained from the running trajectory (S2) and the perfect map keyframes so that the 2D-2D matching features are obtained. After that, we add Gaussian white noise to the 2D features. Finally, for each 3D map feature, we utilize the noisy 2D features and the noisy map keyframe poses to triangulate and optimize its 3D position as the estimated 3D feature. \textcolor{black}{Noting that for the perfect map, we directly use the ground truth of the generated 3D features instead of estimated ones.}

To test the consistency of our proposed algorithm (MSOC-S-IKF), we compare it with the standard EKF version (MSC-S-EKF), and the ones without considering the uncertainty of the map (MSC-IKF, MSC-EKF). We regard Open-VINS\cite{openvins} as the pure visual-inertial odometry baseline method to show that the consistent map-based algorithm can effectively alleviate the drift of odometry. \textcolor{black}{The overall results of the five algorithms with a perfect map and an imperfect map are given in Table \ref{tab:sim_res_2} and Table \ref{tab:sim_res}, respectively. As MSC-S-EKF and MSOC-S-IKF require uncertainties of the map keyframe poses, for the perfect map, we assign a small uncertainty for each map keyframe pose: the standard deviation of the position part is $0.1mm$ and the standard deviation of the orientation part is $0.01degree$.} In the tables, results of three types of variables are given, the pose of the robot in the local reference system (local-pose), the relative transformation between the local reference system and the map reference system (relative-trans), and the pose of the robot in the map reference system (map-pose). For each type of variables, its RMSE and NEES of orientation ($degree$) / position ($m$) are given. All the results are derived from ten Monte Carlo simulations. It should be noted that as Open-VINS is a VIO, it does not estimate the relative-trans, therefore only results of local pose are given. Besides, the other four algorithms estimate local-pose and relative-trans instead of map-pose. The results of map-pose are derived from the estimated local-pose and relative-trans. Therefore, we only give the NEESs of local-pose and relative-trans, to demonstrate the consistency of the algorithms.

% \usepackage{multirow}

\begin{table}[t] 
\caption{\textcolor{black}{RMSEs ($degree/m$) of different algorithms on simulation data with a perfect map}}
\vspace{-2mm}
\setlength{\tabcolsep}{2mm}{
\begin{tabular}{c|c|c|c|c}
\hline
\multicolumn{2}{c|}{Algorithms}&local-pose&relative-trans&map-pose\\
\hline
\multirow{2}{*}{Open-VINS}&RMSE&0.240/0.136&\textbf{—}&\textbf{—}\\
                          &NEES&1.304/0.832&\textbf{—}&\textbf{—}\\
\hline
\multirow{2}{*}{MSC-EKF}&RMSE&0.969/0.440&0.954/0.046&0.098/\textbf{0.019}\\
&NEES&7.110/5.132&6.404/0.566&\textbf{—}\\
\hline
\multirow{2}{*}{MSC-S-EKF}&RMSE&0.986/0.449&0.976/0.056&0.105/0.037\\
&NEES&6.496/4.446&6.154/0.550&\textbf{—}\\
\hline
\multirow{2}{*}{MSC-IKF}&RMSE&0.152/\textbf{0.052}&0.135/\textbf{0.040}&\textbf{0.096}/0.020\\
&NEES&1.548/0.536&0.807/0.404&\textbf{—}\\
\hline
\multirow{2}{*}{\textbf{MSOC-S-IKF}}&RMSE&\textbf{0.150}/0.064&\textbf{0.115}/\textbf{0.040}&0.104/0.037\\
&NEES&1.409/0.600&0.834/0.340&\textbf{—}\\
\hline
\multicolumn{5}{l}{\textbf{—} means there is no such variable to evaluate.}
\end{tabular}}
\label{tab:sim_res_2}
\vspace{-0.5cm}
\end{table}

\begin{table}[t] 
\caption{RMSEs ($degree/m$) of different algorithms on simulation data with an imperfect map}
\vspace{-2mm}
\setlength{\tabcolsep}{2mm}{
\begin{tabular}{c|c|c|c|c}
\hline
\multicolumn{2}{c|}{Algorithms}&local-pose&relative-trans&map-pose\\
\hline
\multirow{2}{*}{Open-VINS}&RMSE&0.246/0.137&\textbf{—}&\textbf{—}\\
                          &NEES&1.027/0.741&\textbf{—}&\textbf{—}\\
\hline
\multirow{2}{*}{MSC-EKF}&RMSE&0.922/0.442&0.968/0.102&0.498/0.249\\
&NEES&5.631/7.725&103.499/2.308&\textbf{—}\\
\hline
\multirow{2}{*}{MSC-S-EKF}&RMSE&1.157/0.521&1.179/0.107&0.174/\textbf{0.106}\\
&NEES&7.611/5.265&7.800/0.901&\textbf{—}\\
\hline
\multirow{2}{*}{MSC-IKF}&RMSE&0.209/0.179&0.454/0.126&0.475/0.234\\
&NEES&1.194/5.507&95.108/3.565&\textbf{—}\\
\hline
\multirow{2}{*}{\textbf{MSOC-S-IKF}}&RMSE&\textbf{0.175}/\textbf{0.113}&\textbf{0.175}/\textbf{0.102}&\textbf{0.170}/0.112\\
&NEES&0.980/1.005&0.597/0.839&\textbf{—}\\
\hline
\multicolumn{5}{l}{\textbf{—} means there is no such variable to evaluate.}
\end{tabular}}
\label{tab:sim_res}
\vspace{-0.5cm}
\end{table}

\textbf{Trajectory accuracy:} \textcolor{black}{From Table \ref{tab:sim_res_2}, we can find that both MSC-IKF and MSOC-S-IKF have good performance in local-pose and relative-trans. This is because these two algorithms correctly maintain the observability of the \textit{perfect augmented system} (cf. \textbf{Theorem \ref{theorem: perfect iekf}}). However, for the imperfect map, as shown in Table \ref{tab:sim_res}, the situation is different.}

From Table \ref{tab:sim_res}, we can find that for local-pose and relative-trans, MSOC-S-IKF has the minimum RMSE. And for map-pose, which is the result of a direct multiplication of local-pose and relative-trans, MSOC-S-IKF has the best orientation accuracy and good position accuracy. It is worth noting that for the \textit{augmented system}, it online estimates local-pose and relative-trans instead of map-pose. Therefore, the performance on estimating local-pose and relative-trans is what we are concerned about. And listing the results of map-pose is just for illustrating that with a good estimation on local-pose and relative-trans, we can get a competitive or even the best result of map-pose. In contrast, MSC-IKF has poor performance for local-pose, relative-trans, and map-pose, even though it correctly maintains the observability of the system. It is because MSC-IKF does not consider the uncertainty of the map. Similarly, MSC-EKF, which neither considers the uncertainty of the map nor correctly maintains the observability of the system, and MSC-S-EKF, which does not correctly maintain the observability of the system, have (much) worse performance than MSOC-S-IKF, which not only considers the uncertainty of the map but maintains the correct observability of the system. Fig. \ref{fig:sim_traj_compare} gives an intuitive comparison among Open-VINS, MSC-EKF, MSC-IKF, and MSOC-S-IKF for the local trajectory on the x-y plane with the imperfect map. From the zoom-in area of Fig. \ref{fig:sim_traj_compare}, we can find that the trajectory derived from MSOC-S-IKF (the green one) is the closest to the ground truth.

% From Table \ref{tab:sim_res}, we can find that for local-pose and relative-trans, MSOC-S-IKF has the minimum RMSE, and for map-pose, which is the result of a direct multiplication of local-pose and relative-trans, MSOC-S-IKF has the best orientation accuracy and good position accuracy. It is worth noting that for the \textit{augmented system}, it online estimates local-pose and relative-trans instead of map-pose. Therefore, the performance on estimating local-pose and relative-trans is what we concern about, and listing the results of map-pose is just for illustrating that with a good estimation on local-pose and relative-trans, we can get a competitive or even the best result of map-pose. 

% From Table \ref{tab:sim_res}, MSC-EKF, which neither considers the uncertainty of the map information nor correctly maintains the observability of the system, MSC-S-EKF, which does not correctly maintains the observability of the system, and MSC-IKF which does not consider the uncertainty of the map information, have (much) worse performance than MSOC-S-IKF, which not only considers the uncertainty of the map information but maintains the correct observability of the system. Fig. \ref{fig:sim_traj_compare} gives an intuitive comparison among Open-VINS, MSC-EKF, MSC-IKF and MSOC-S-IKF toward the local trajectory. From the zoom-in area of Fig. \ref{fig:sim_traj_compare} (a), we can find that the trajectory derived from MSOC-S-IKF (the green one) best fits the ground truth.

\begin{figure}[t]
    \centering
    \includegraphics[width=1\linewidth]{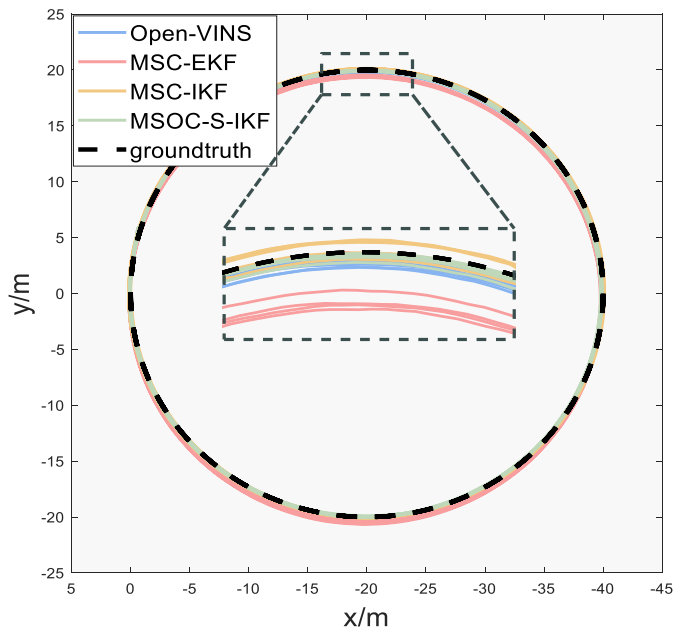}
     \vspace{-0.2cm}
     \caption{The trajectory comparison of different methods on x-y plane}
     \label{fig:sim_traj_compare}
    %  \vspace{-0.5cm}
\end{figure}

% \begin{figure}[t]
%     \centering
%     \subfigure[trajectories on x-y plane]{
%     \begin{minipage}{0.7\linewidth}\centering
%     \includegraphics[width=1\textwidth]{sim_traj_comp_xy.pdf}
%     \end{minipage}
%     } \vspace{-0.2cm}
%     \subfigure[trajectories along z axis v.s. timestamps]{
%     \begin{minipage}{0.7\linewidth}\centering
%     \includegraphics[width=1\textwidth]{sim_traj_comp_z.pdf}
%     \end{minipage}
%     }
%      \caption{The trajectory comparison of different methods}
%      \label{fig:sim_traj_compare}
%     %  \vspace{-0.5cm}
% \end{figure}

\textbf{Consistency:} To show the consistency properties of the different algorithms, the NEESs for local-pose and relative-trans are also given in Table \ref{tab:sim_res_2} and Table \ref{tab:sim_res}. When the map is perfect, we can find that both MSC-IKF and MSOC-S-IKF have the NEES around 1, which indicates the good consistency of these two algorithms. However, when the map is imperfect, MSC-IKF has large values of NEES because it neglects the uncertainty of the map, while MSOC-S-IKF has good consistency property (the values of NEES are around 1). For both situations, MSC-EKF and  MSC-S-EKF have large values of NEES, which indicates that the systems are overconfident to the estimation uncertainties.

Fig. \ref{fig:sim_3sigma} and Fig. \ref{fig:sim_nees} show the position error with 3$\sigma$ bounds of local-pose and the NEES of each timestep of local-pose derived from MSC-S-EKF, MSC-IKF, and MSOC-S-IKF under the situation that the map is imperfect. For MSC-S-EKF, since the correct unobservable subspace is broken (cf. \textbf{Theorem \ref{theorem: real perfect ekf}} and \textbf{Theorem \ref{theorem: real imperfect ekf}}), its estimation shows a trend of divergence. For MSC-IKF, although it correctly maintains the observability of the system (cf. \textbf{Theorem \ref{theorem: perfect iekf}}), the uncertainty of the map information is not considered, which makes the estimator over-reliant on the imperfect map information and leads to inaccurate and inconsistent results. On the contrary, MSOC-S-IKF considers the above two factors to get an accurate and consistent estimation.

\begin{figure}[t]
\centering
\setlength{\abovecaptionskip}{0cm}
    \includegraphics[width=1\linewidth]{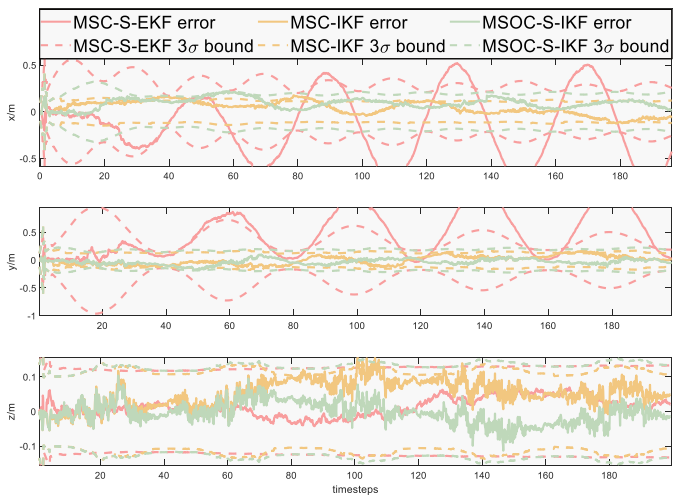}
   \caption{The error of the position part of the local-pose with 3$\sigma$ bound}
   \label{fig:sim_3sigma}
   \vspace{-0.5cm}
\end{figure}

\begin{figure}[t]
\centering
\setlength{\abovecaptionskip}{0cm}
    \includegraphics[width=1\linewidth]{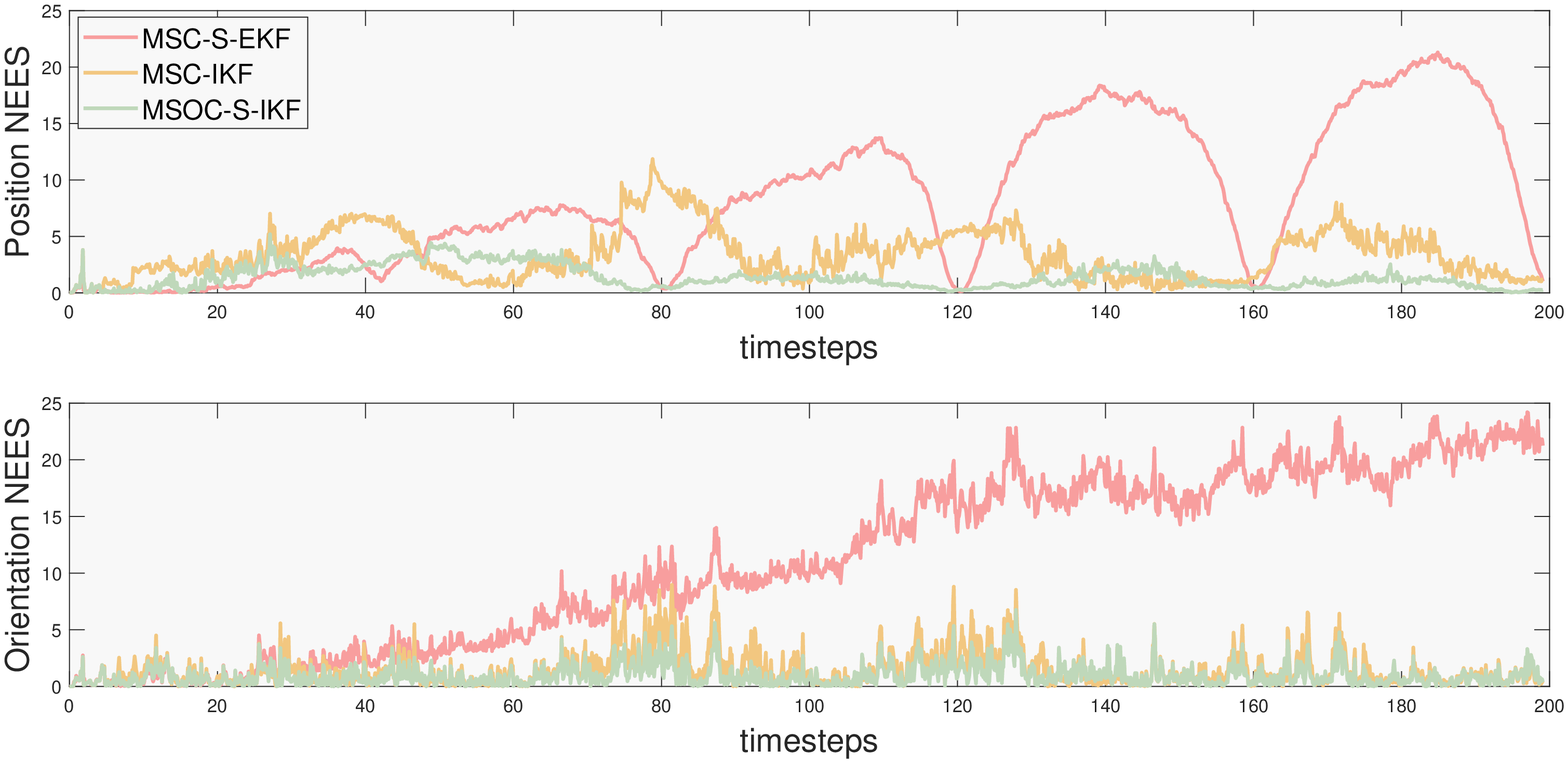}
   \caption{The NEES of the local-pose}
   \label{fig:sim_nees}
   \vspace{-0.5cm}
\end{figure}

\vspace{-2mm}
\subsection{Real world experiments}
In this part, we will validate our proposed algorithm in four kinds of real world datasets, which cover the scenarios of aerial vehicle (EuRoC \cite{euroc}), ground vehicles in urban areas (Kaist \cite{kaist}, 4Seasons \cite{4seasons}) and in a campus (YQ \cite{YQ}).

\begin{figure}
    \centering
    \subfigure[EuRoC]{
    \includegraphics[width=1\linewidth]{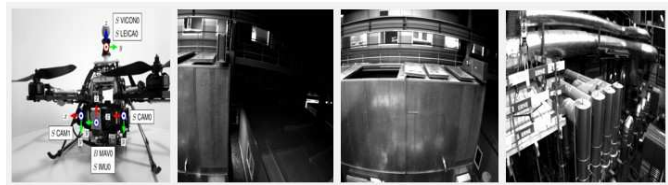}
    }\\
    \vspace{-2mm}
    \subfigure[Kaist]{
    \includegraphics[width=1\linewidth]{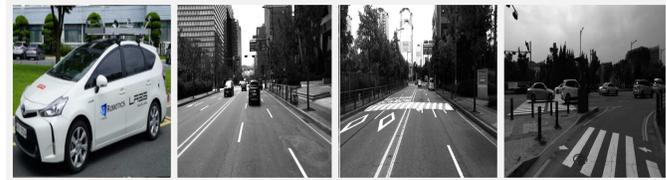}
    }\\
    \vspace{-2mm}
    \subfigure[4Seasons]{
    \includegraphics[width=1\linewidth]{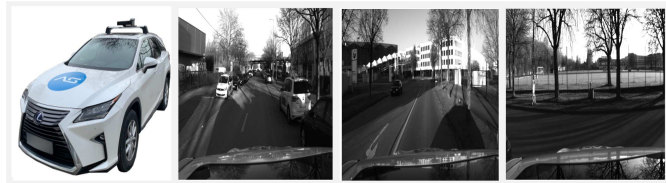}
    }\\
    \vspace{-2mm}
    
    \subfigure[YQ]{
    \includegraphics[width=1\linewidth]{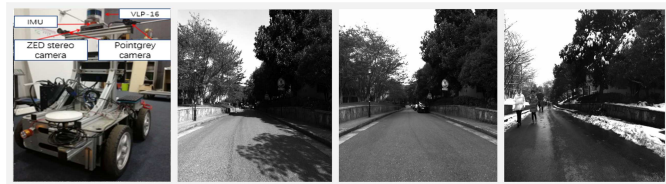}
    }
    \vspace{-3mm}
     \caption{The vehicles and images sampled from the four datasets.}
     \label{fig:dataset}
     \vspace{-0.5cm}
\end{figure}

\textbf{EuRoC:} EuRoC\cite{euroc} is a visual-inertial dataset collected on-board a Micro Aerial Vehicle (MAV). This dataset contains scenarios of vicon rooms and a machine hall, where there are three scenarios, vicon room 1 (V1), vicon room 2 (V2), and machine hall (MH). For our experiments, the sequences V101, V201 and MH01 are used to build the map, while the sequences V102-V103, V202-V203 and MH02-MH05 are used for localization. The map contains keyframes and 3D features. As shown in dataset documents, the ground truth poses have measurement error in millimeters, so we regard the ground truth poses as the noisy map keyframes with the standard deviation of $1\,cm$ and $1\,degree$. The positions of 3D features are triangulated and optimized across multiple adjacent map images and the corresponding poses. The first row of Fig. \ref{fig:dataset} gives some selected images from EuRoC.

\textbf{Kaist:} Kaist\cite{kaist} is a dataset collected in urban on-the-road environments with many moving vehicles (cf. the second row of Fig. \ref{fig:dataset}), so this is a challenging dataset causing drift in visual-inertial odometry. In the dataset, there are two sequences, Urban38 and Urban39, having a large overlap. We employ Urban38 to build the map and Urban39 for localization testing. The length of the trajectory is 10.67km. The map keyframe and features are calculated following that in EuRoC part. As the ground truth in this dataset is built upon Virtual Reference Station (VRS)–GPS, we set the standard deviation of map keyframe pose as $0.1\,m$ for position parts and $2.87\, degree$ ($0.05 \, rad$) for orientation parts.

\textbf{4Seasons:} 4Seasons \cite{4seasons} is a dataset collected in different scenarios and under various weather conditions and illuminations. This dataset is for the research on VIO, global place recognition, and map-based re-localization, suitable to test our proposed algorithm. In this dataset, we select the first two sequences ($2020\text{-}03\text{-}24\_17\text{-}36\text{-}22$, denoted as Office-Loop1, $2020\text{-}03\text{-}24\_17\text{-}45\text{-}31$, denoted as Office-Loop2) from the collection called ``Office Loop", where the vehicle loops around an industrial area of the city (cf. the third row of Fig. \ref{fig:dataset}). We pick these two sequences because the vehicle at the beginning of the sequences is static, which is necessary for the IMU initialization of our system. Besides, due to lots of overlaps of their trajectories and the similar illuminations, these two sequences are easier to detect matching images. We employ Office-Loop-1 to build the map and Office-Loop-2 for localization testing. The map keyframe and features are generated following the procedure mentioned 
in EuRoC part. As the ground truth in this dataset is built upon the Real Time Kinematic-Global Navigation Satellete System (RTK-GNSS) that provide global positioning with up-to-centimeter accuracy \cite{4seasons}, the standard deviation of the map keyframe pose is set the same as that in Kaist.

\textbf{YQ:} YQ\cite{dxq,YQ} is our own collected data in Yuquan Campus, Zhejiang University, China, by a four-wheel ground vehicle. This dataset aims at testing in off-the-road campus environments. It contains four sequences, YQ1-YQ4, where YQ1-YQ3 were recorded on three separate days with different weather in summer and YQ4 was collected in winter after snowing (cf. the fourth row of Fig. \ref{fig:dataset}). The changing weather and season severely degenerate the feature matching, causing the long absence of map-based measurements. For this dataset, we use YQ1 to build the map and YQ2-YQ4 to test the performance of the algorithms. The standard deviation of the map keyframe pose is set the same as that in Kaist.

\textbf{Feature matching:} For all real world datasets, the matching procedure is conducted as follows: We first utilize R2D2\cite{r2d2} to extract new features on the current query frame and match with features in map keyframes. The initial matching pairs based on feature descriptors are then fed into a robust pose estimator in \cite{2entity} to generate accurate feature matching pairs.

\textbf{Benchmark with comparative methods:} 
In this part, we make comparisons between our proposed methods with the benchmark from Open-VINS\cite{openvins} and VINS-Fusion\cite{vinsmono,vinsgps} on EuRoC, Kaist, 4Seasons and YQ. We regard Open-VINS as a pure odometry baseline with drift, and validate the correction of map-based measurements. For VINS-Fusion, its localization mode is used. We keep its map setting the same as the ones in our method. The only difference is that VINS-Fusion ignores the uncertainty of the map, so we do not set the map covariance for VINS-Fusion. We record the localization result of VINS-Fusion by concatenating the odometry $^{L}\mathbf{T}_{k}$ and the estimated $^{G}\mathbf{T}_{L}$ instead of the optimized trajectory due to the real-time causality. 

\begin{figure}[t]
\centering
\setlength{\abovecaptionskip}{0cm}
    \includegraphics[width=1\linewidth]{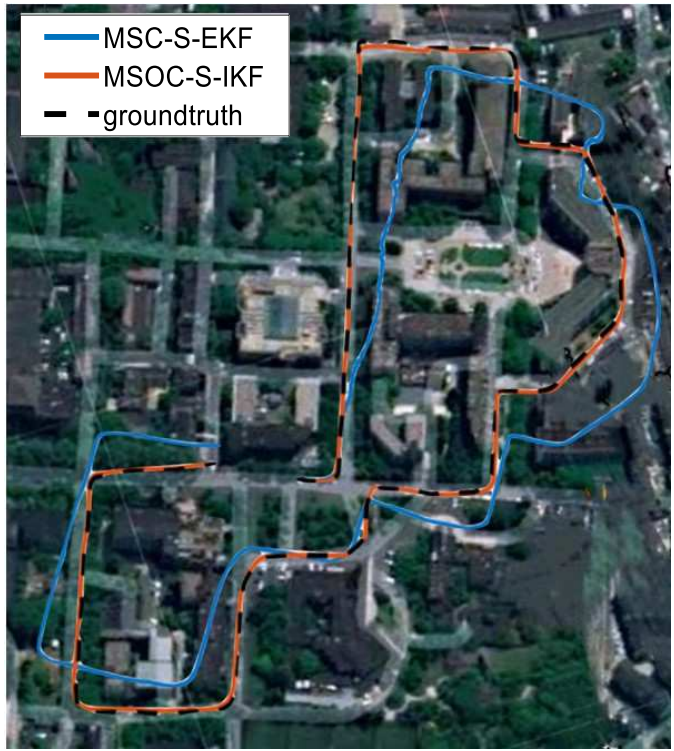}
   \caption{The trajectory comparison (YQ2) between MSC-S-EKF and MSOC-S-IKF on the satellite map}
   \label{fig:YQ2_comp}
   \vspace{-0.5cm}
\end{figure}

\begin{table*}[t]
    \centering
    \caption{The ATE results ($m$) of different algorithms on different datasets}
    \vspace{-2mm}
    \begin{tabular}{c|c|c|c|c|c|c|c}
    \hline
    \hline
    \multicolumn{2}{c|}{\diagbox{Datasets}{Algorithms}}
         &Open-VINS&VINS-Fusion&MSC-S-EKF&MSC-EKF&MSC-IKF&\textbf{MSOC-S-IKF} \\
         \hline
    \multirow{2}{*}{\makecell[c]{V102\\(75.89m)}}&\makecell[c]{local position}&0.059&0.120&0.044&0.062&0.053&\textbf{0.036}\\
    &\makecell[c]{map position}&\textbf{—}&0.105&\textbf{0.039}&0.043&0.072&\textbf{0.039}\\
    \cline{2-8}
     \multirow{2}{*}{\makecell[c]{V103\\(78.98m)}}&\makecell[c]{local position}&0.090&0.168&0.047&0.053&\textbf{0.046}&0.047\\
    % \cdashline{2-6}
    &\makecell[c]{map position}&\textbf{—}&0.145&\textbf{0.052}&0.063&0.062&0.057\\
    \hline
    \multirow{2}{*}{\makecell[c]{V202\\(86.23m)}}&\makecell[c]{local position}&0.065&0.147&0.056&0.094&0.062&\textbf{0.049}\\
    &\makecell[c]{map position}&\textbf{—}&0.081&0.049&0.093&0.092&\textbf{0.046}\\
    \cline{2-8}
    \multirow{2}{*}{\makecell[c]{V203\\(86.13m)}}&\makecell[c]{local position}&0.112&0.209&0.117&0.100&0.089&\textbf{0.086}\\
    &\makecell[c]{map position}&\textbf{—}&0.229&0.150&0.154&0.133&\textbf{0.112}\\
    \hline
    \multirow{2}{*}{\makecell[c]{MH02\\(73.70m)}}&\makecell[c]{local position}&0.131&0.107&0.065&0.115&0.048&\textbf{0.045}\\
    &\makecell[c]{map position}&\textbf{—}&0.059&0.030&0.044&0.045&\textbf{0.027}\\
    \cline{2-8}
     \multirow{2}{*}{\makecell[c]{MH03\\(130.93m)}}&\makecell[c]{local position}&0.149&0.172&0.088&0.149&0.110&\textbf{0.080}\\
    &\makecell[c]{map position}&\textbf{—}&0.247&0.081&0.192&0.188&\textbf{0.078}\\
    \cline{2-8}
    \multirow{2}{*}{\makecell[c]{MH04\\(91.75m)}}&\makecell[c]{local position}&0.192&\textbf{0.172}&0.236&0.864&0.267&0.186\\
    &\makecell[c]{map position}&\textbf{—}&0.428&0.350&0.728&0.584&\textbf{0.255}\\
    \cline{2-8}
    \multirow{2}{*}{\makecell[c]{MH05\\(97.59m)}}&\makecell[c]{local position}&0.246&\textbf{0.178}&0.289&0.309&0.219&0.198\\
    &\makecell[c]{map position}&\textbf{—}&0.401&0.526&0.824&0.508&\textbf{0.280}\\
    \hline
    \multirow{2}{*}{\makecell[c]{Urban39\\(10.67km)}}&\makecell[c]{local position}&10.491&\XSolidBrush&42.873&\XSolidBrush&\XSolidBrush&\textbf{2.875}\\
    &\makecell[c]{map position}&\textbf{—}&\XSolidBrush&18.030&\XSolidBrush&\XSolidBrush&\textbf{3.107}\\
    \hline
    \multirow{2}{*}{\makecell[c]{YQ2\\(1.30km)}}&\makecell[c]{local position}&5.159&\XSolidBrush&5.285&\XSolidBrush&10.826&\textbf{1.635}\\
    &\makecell[c]{map position}&\textbf{—}&\XSolidBrush&1.751&\XSolidBrush&20.940&\textbf{1.504}\\
    \cline{2-8}
     \multirow{2}{*}{\makecell[c]{YQ3\\(1.28km)}}&\makecell[c]{local position}&4.198&\XSolidBrush&3.425&\XSolidBrush&\XSolidBrush&\textbf{1.348}\\
    &\makecell[c]{map position}&\textbf{—}&\XSolidBrush&1.814&\XSolidBrush&\XSolidBrush&\textbf{1.406}\\
    \cline{2-8}
     \multirow{2}{*}{\makecell[c]{YQ4\\(0.93km)}}&\makecell[c]{local position}&5.267&\XSolidBrush&3.721&\XSolidBrush&\XSolidBrush&\textbf{1.633}\\
    &\makecell[c]{map position}&\textbf{—}&\XSolidBrush&3.384&\XSolidBrush&\XSolidBrush&\textbf{2.502}\\
    \hline
     \multirow{2}{*}{\makecell[c]{Office-Loop2\\(3.859km)}}&\makecell[c]{local position}&12.987&31.389&33.586&\XSolidBrush&\XSolidBrush&\textbf{3.666}\\
    &\makecell[c]{map position}&\textbf{—}&19.479&25.698&\XSolidBrush&\XSolidBrush&\textbf{4.048}\\
    \hline
    \hline
    \multicolumn{8}{l}{\textbf{—} means there is no such variable to evaluate; \XSolidBrush means this algorithm is failed to run the corresponding sequence.}
    \end{tabular}
    \label{tab:all_ate_comp}
    \vspace{-0.5cm}
\end{table*}

\textbf{Trajectory accuracy:} Table \ref{tab:all_ate_comp} lists ATE results of position ($m$) in the local frame (local position) and in the map frame (map position), which are derived from different algorithms on different dataset sequences. The ATE results of the local position are derived by aligning the local trajectories with the ground truths, whereas for the map position, we directly compare the computed trajectories in the map frame with the ground truths without alignments. All the results are the average of three runs. Since the data in Kaist, 4Seasons and YQ are collected from ground vehicles, the ATE results on these datasets are calculated in 2D (x-y plane). We can find that our consistent algorithm MSOC-S-IKF has great (or the best for most cases) performance for both local and map (global) positions, whereas VINS-Fusion, MSC-EKF, and MSC-IKF are failed to survive in the long-time and large-scale scene (Kaist and YQ). VINS-Fusion diverges due to the significant drifts and the overconfident belief in the noisy map, while MSC-EKF and MSC-IKF diverge mainly due to ignoring the uncertainty of the map information so that the estimators maintain incorrect covariance matrices which leads to wrong state updating. It is worth mentioning that although MSC-S-EKF, which cannot maintain the correct unobservable subspace, can successfully run through all of these dataset sequences and sometimes has competitive results for some evaluated variables (e.g., map position of YQ), the results of the variable (local-position) that is actually estimated by the estimator turn out to be bad. \textcolor{black}{For some sequences, the accuracy of MSC-S-EKF of local-pose is worse than the pure odometry (Open-VINS), like V203, MH04, Urban39, and Offie-Loop2.} 
% (cf. Fig. \ref{fig:YQ2_comp}). 
Besides, the local trajectories of MSC-S-EKF as well as MSOC-S-IKF are plotted in Fig. \ref{fig:YQ2_comp}, where the two trajectories are aligned with the ground truth by the \textit{first pose} such that we can see how the local trajectories change over time. From Fig. \ref{fig:YQ2_comp}, we can see the trajectory of MSC-S-EKF wildly deviates from the ground truth after a short term running. This phenomenon stems from the inconsistent estimation, which introduces spurious information to the system. This makes the estimator produce wrong results even though sometimes the results of map position (derived though the estimated local pose and the \textit{augmented variable}) are not bad. 

In particular, we can find an interesting phenomenon from Table \ref{tab:all_ate_comp}: For EuRoC datasets, the values of the map keyframe poses are set the same as the ground truth, and the reprojection error of triangulated 3D features is around 0.45 pixel, which means the pre-build map have pretty good accuracy. For MSC-IKF that can naturally maintain the correct observability of the \textit{augmented system}, it is still inferior to MSOC-S-IKF in the vast majority of cases. This demonstrates the necessity to consider the uncertainty of the map information.

\begin{figure*}[t]
    \centering
    \subfigure[Urban39 local]{\includegraphics[width=0.24\textwidth]{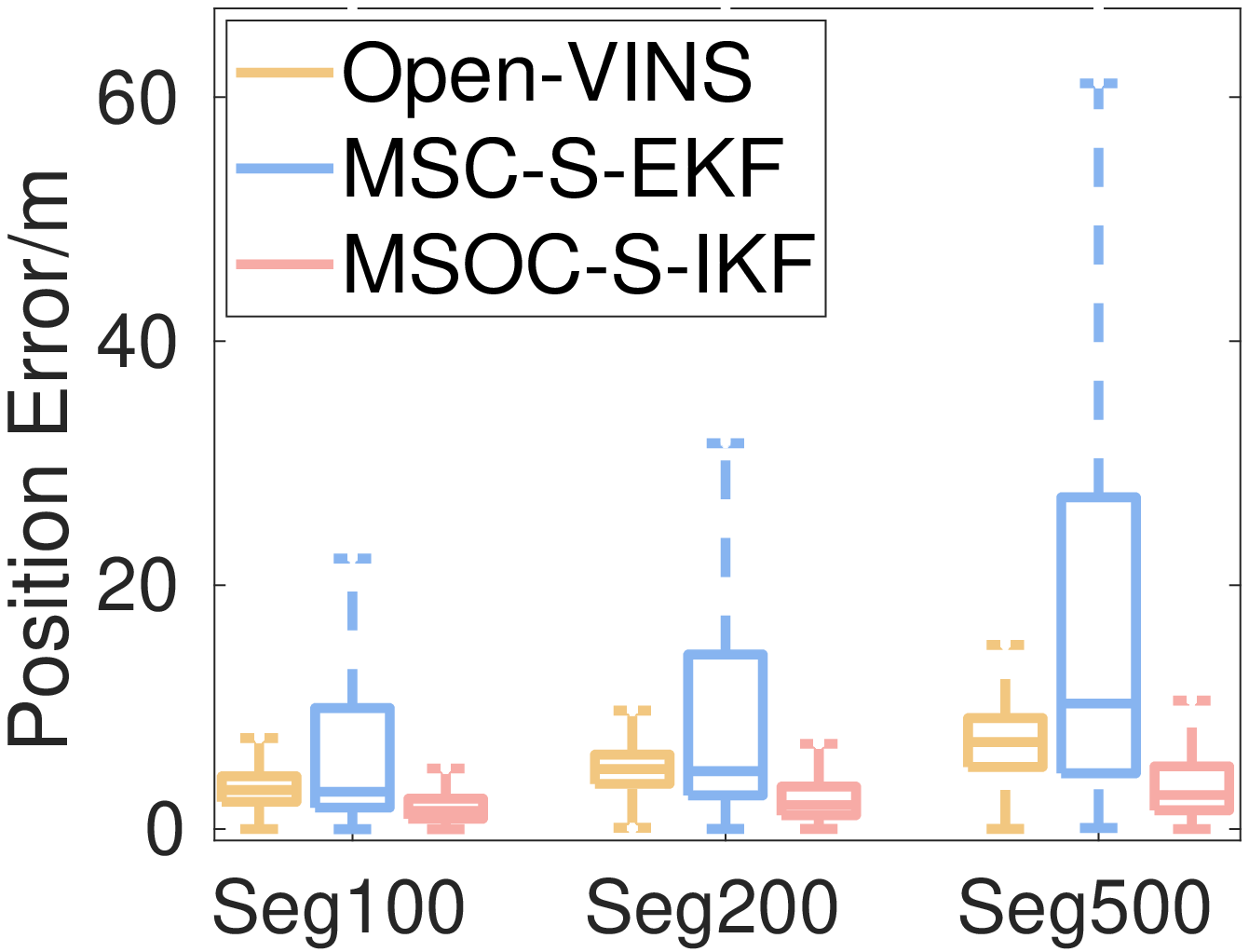}}
    \vspace{-0.12cm}
    \subfigure[YQ2 local]{\includegraphics[width=0.24\textwidth]{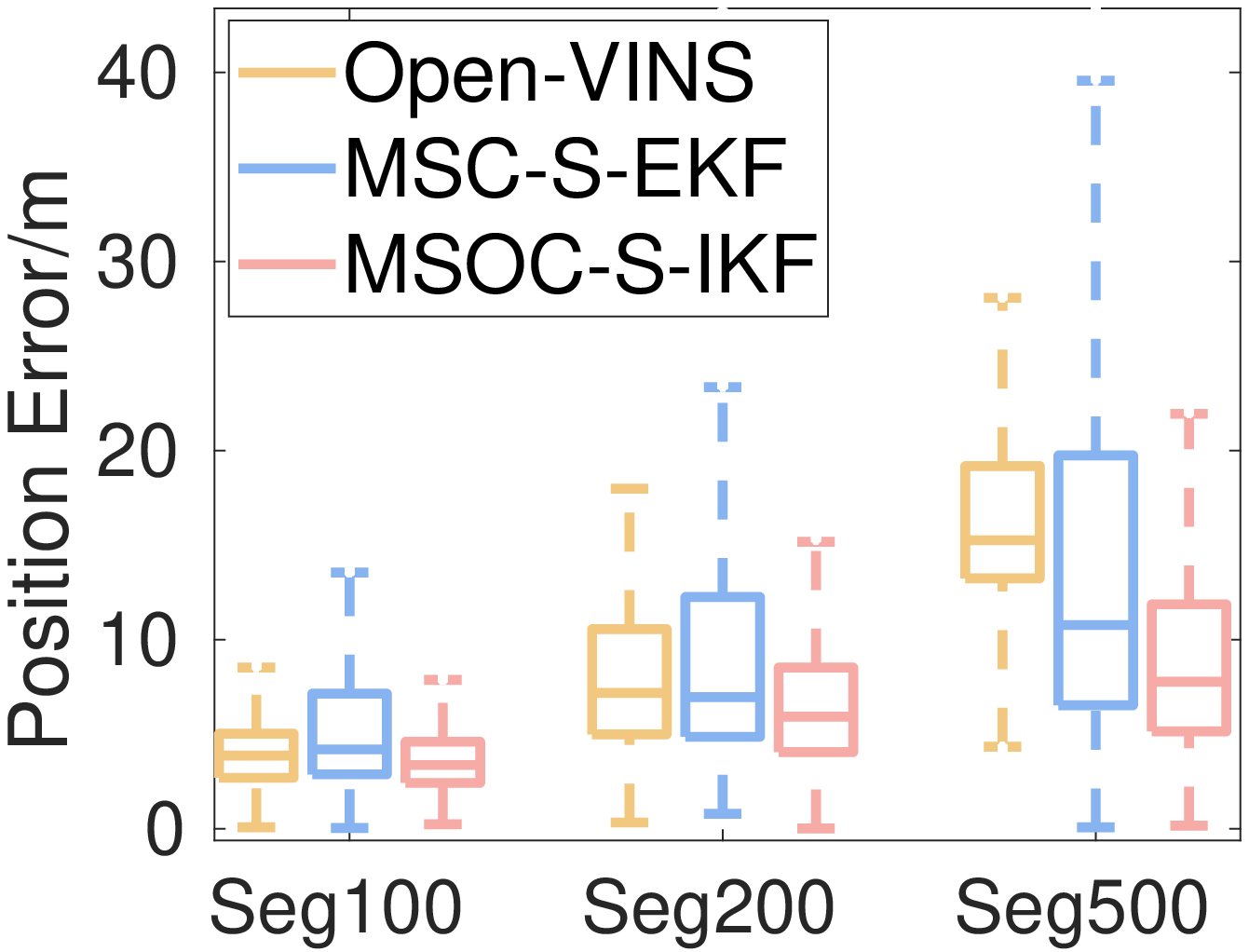}}
    \vspace{-0.12cm}
    \subfigure[YQ3 local]{\includegraphics[width=0.24\textwidth]{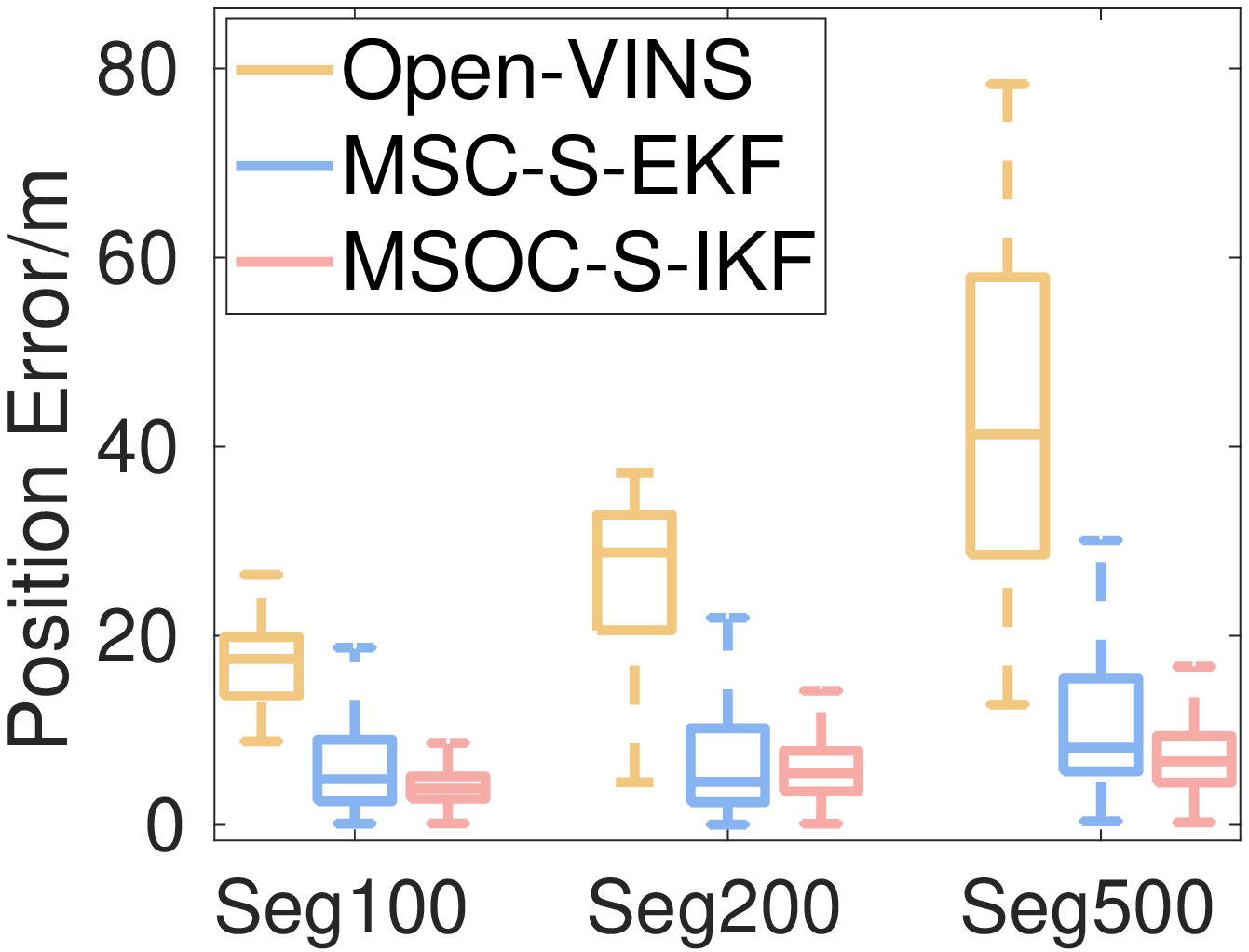}}
    \vspace{-0.12cm}
    \subfigure[YQ4 local]{\includegraphics[width=0.24\textwidth]{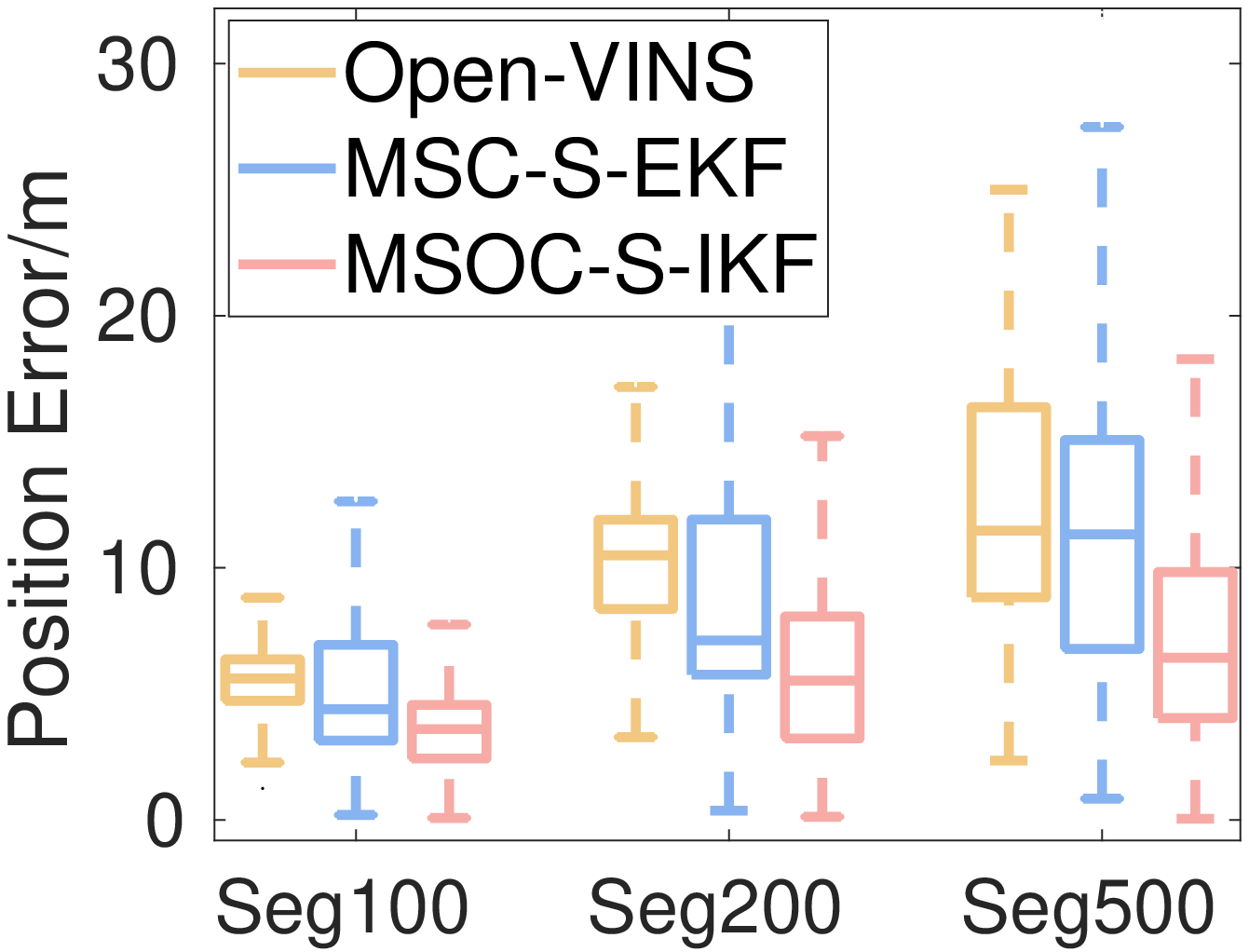}}
    \vspace{-0.12cm}
    \subfigure[Urban39 global]{\includegraphics[width=0.24\textwidth]{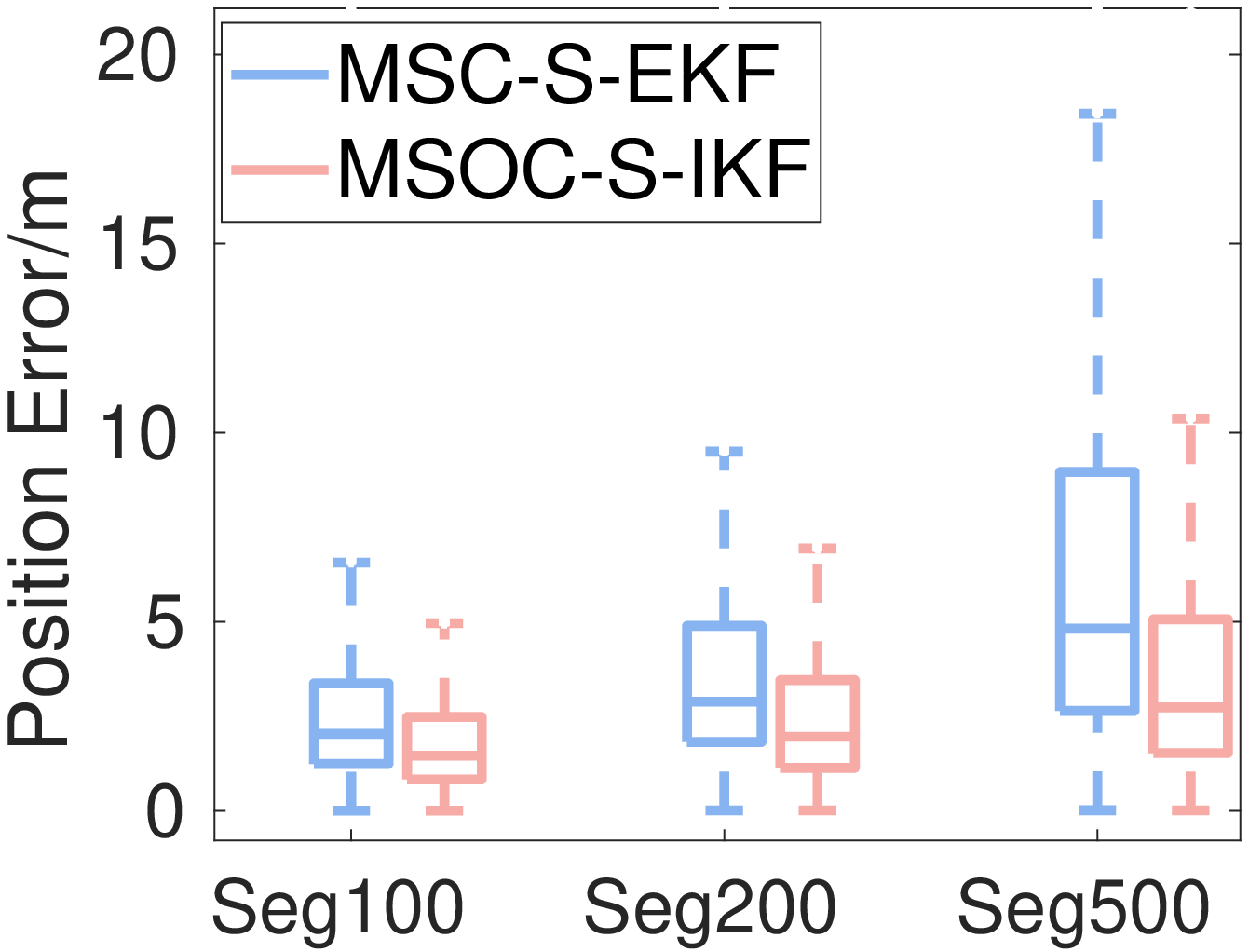}}
    \vspace{-0.0cm}
    \subfigure[YQ2 global]{\includegraphics[width=0.24\textwidth]{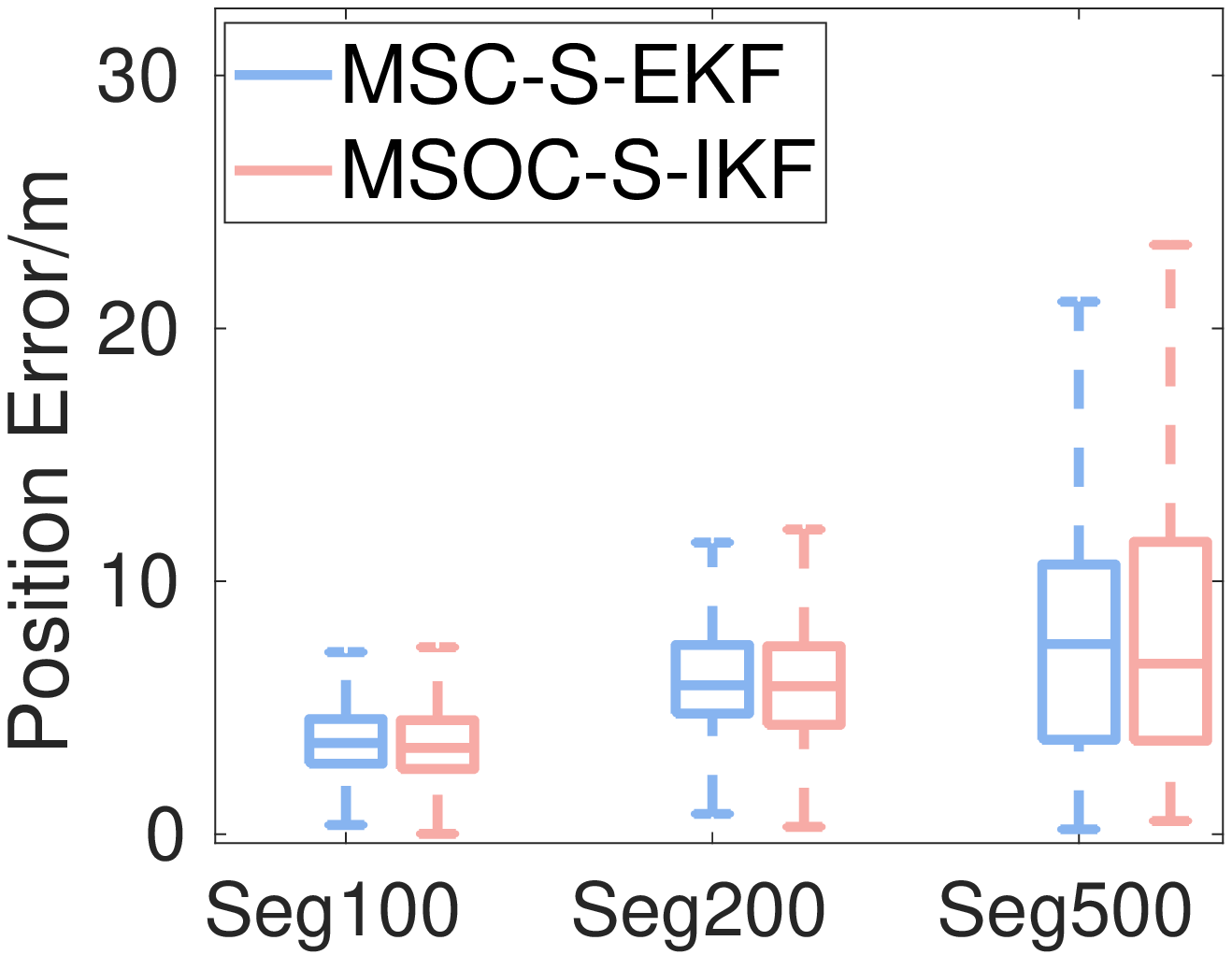}}
    \vspace{-0.0cm}
  \subfigure[YQ3 global]{\includegraphics[width=0.24\textwidth]{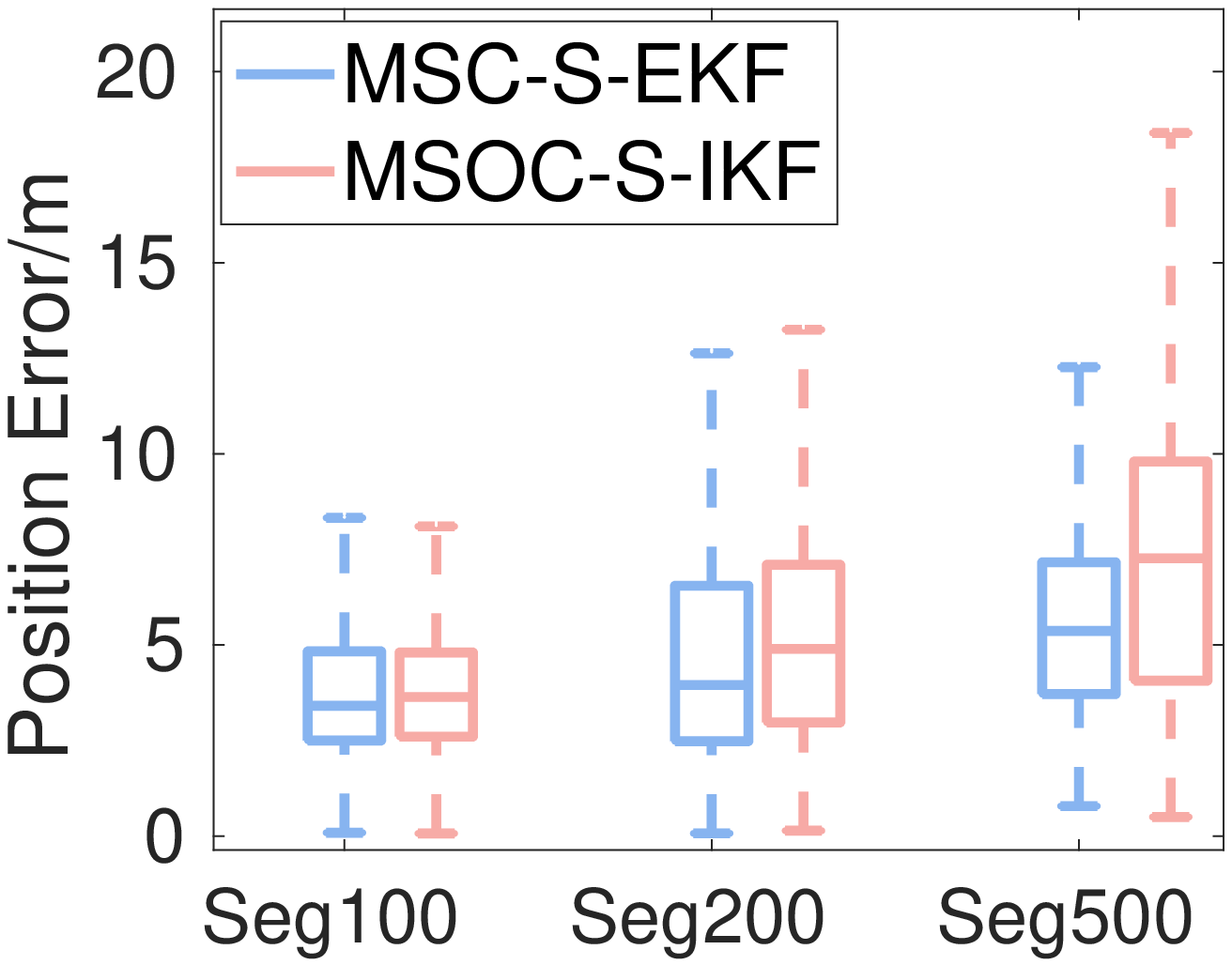}}
  \vspace{-0.0cm}
    \subfigure[YQ4 global]{\includegraphics[width=0.24\textwidth]{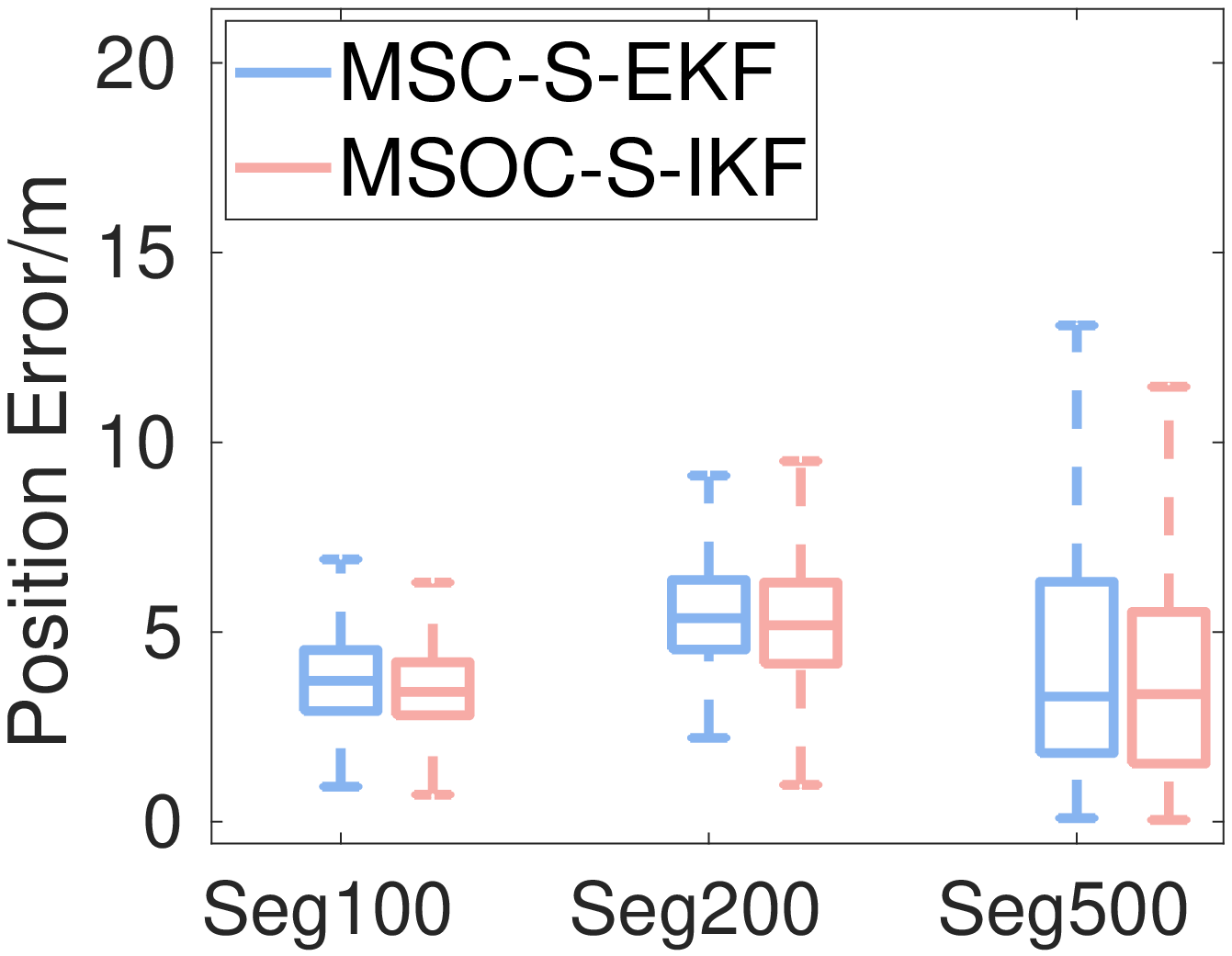}}
    \vspace{-0.0cm}
    \caption{The RPE results of different algorithms on different datasets}
    \label{fig:RPE}
    \vspace{-0.5cm}
\end{figure*}

\textcolor{black}{To see the drift of the localization and the smoothness of the trajectory, in Fig. \ref{fig:RPE}, we give the box plots of RPEs of Open-VINS, MSC-S-EKF, and MSOC-S-IKF for the position in both local and map reference frames. As defined in (\ref{eq:RPE}), we select the segment length $l$ as $100m, 200m$ and $500m$. 
% The median errors and related standard deviations (given in brackets) for these three segment lengths are given. 
We can find that, with the segment length increasing, the position errors also increase, which indicates the existence of the drift. However, as MSOC-S-IKF consistently fuses the map information, the drift is alleviated, which can be illustrated by comparing the local position errors of Open-VINS and MSOC-S-IKF. On the contrary, MSC-S-EKF has apparent drift, and the RPE for each segment length is also large. This is due to the fact that MSC-S-EKF is inconsistent, and spurious information is introduced to the system. Moreover, we can find that, for the local positions, the deviation of MSC-S-EKF is much larger than that of MSOC-S-IKF, which indicates that the RPE of MSC-S-EKF changes much more apparently than that of MSOC-S-IKF, i.e., the smoothness of the trajectory of MSC-I-EKF in the local reference system is much better than that of MSC-S-EKF. This conclusion can also be illustrated in Fig. \ref{fig:YQ2_comp}, where the trajectory of MSC-S-EKF is severely distorted. }

\textbf{Real-time efficiency:} The computing efficiency of our proposed algorithm is also validated. The average time consumptions per step of VINS-Fusion and MSOC-S-IKF on EuRoC sequences are given in Fig. \ref{fig:time_comp}, with the pure odometry Open-VINS as a baseline. One can see that our proposed filter based solution is more time-saving than VINS-Fusion, thus is more appropriate for on-board deployment.
\vspace{-2mm}

\begin{figure}[t]
\centering
    \includegraphics[width=1\linewidth]{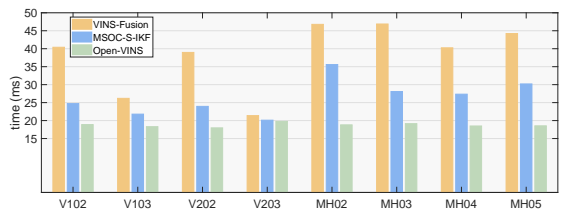}
    \vspace{-7mm}
   \caption{The average time consumptions (ms) per step of different methods on EuRoC}
   \label{fig:time_comp}
   \vspace{-0.5cm}
\end{figure}
\subsection{Discussion}

\textcolor{black}{Through the simulations and the real world experiments, we validate that MSC-IKF is consistent for the \textit{perfect augmented system}, but inconsistent for the \textit{imperfect augmented system} due to ignoring the uncertainty of the \textit{imperfect map}. On the contrary, because MSOC-S-IKF not only correctly maintains the observability of the \textit{imperfect augmented system}, but also considers the uncertainty of the map, the trajectory accuracy in both the local odometry frame and the map frame is the best in the vast majority of situations.}

\textcolor{black}{From experiments, there is an interesting phenomenon. Although MSC-S-EKF is inconsistent, it can successfully run on all the datasets. And for many sequences, like YQ2 and YQ3, MSC-S-EKF has competitive trajectory accuracy in the map frame. This is probably because there is only one relative transformation between the local odometry frame and the map (\textit{augmented variable}) is considered in our problem, which makes the impact of inconsistency caused by the \textit{augmented variable} less significant on the trajectory accuracy in the map frame. If there are multiple types of sensors with their corresponding maps, there will be multiple \textit{augmented variables}, then the impact of inconsistency caused by these \textit{augmented variables} might be more significant. For example, if a vehicle is equipped with a camera  and a GPS sensor, and a pre-built visual map and GPS signals are employed to provide global information, then there would be a relative transformation between the local odometry frame and the visual map frame and a relative transformation between the local odometry frame and the GPS measurements frame. Under this kind of situation, the system needs to estimate multiple \textit{augmented variables}. Besides, another situation where multiple \textit{augmented variables} are considered can also occur in multi-robot localization scenarios.}

% In addition, when we need to consider loop-closure, the consistent estimation of the local pose is necessary. The inconsistent estimation will make the estimation error of the system larger when a loop is closed\cite{consist ekf slam,consist ikf slam}. Therefore, MSC-S-EKF will inevitably produce wrong results when there is a loop-closure. On the contrary, as our proposed MSOC-S-IKF can maintain the correct observability properties of the system, its estimated state can be safely used to close the loop.

\textcolor{black}{Moreover, as shown in Table \ref{tab:all_ate_comp}, Fig. \ref{fig:YQ2_comp} and Fig. \ref{fig:RPE}, the accuracy and the smoothness of the local trajectory of MSC-S-EKF are much worse than MSOC-S-IKF. This could make the trajectory planning unfeasible for MSC-S-EKF, because in most practical applications, trajectory planning is carried out in a local odometry frame, making localization accuracy in the local odometry frame matter. Meanwhile, localization in the map (global) frame is used to carry out global path planning. Therefore, we need good localization accuracy in both local and global frames, and our proposed consistent algorithm, MSOC-S-IKF, can give promising results. 
}

\vspace{-2mm}
\section{Conclusion}
In this paper, a general situation of map-based VIL is considered, where the pre-built map is imperfect and gravity-unaligned. The map-based VIL problem is formulated as maintaining the local VIO and estimating the \textit{augmented variable} simultaneously. Based on a special Lie group, Schmidt filter, and the multi-state observability constarint techinique, we propose a consistent and efficient map-based VIL algorithm, MSOC-S-IKF. By observability-based analysis and computational complexity analysis, we theoretically demonstrate the consistency and efficiency of our proposed algorithm, respectively. Results of simulations and real-data experiments further validate the correctness of the theory and the effectiveness of our proposed algorithm. Future works include extending the MSOC-S-IKF to the distributed system to perform multiple robots localization. Besides, how to add loop closure into this algorithm consistently and efficiently will also be our future research direction.   

\vspace{-2mm}

\ifCLASSOPTIONcaptionsoff
  \newpage
\fi



%\bibliographystyle{IEEEtran}

\vspace{-5mm}
\begin{IEEEbiography}[{\includegraphics[width=1in,clip,keepaspectratio]{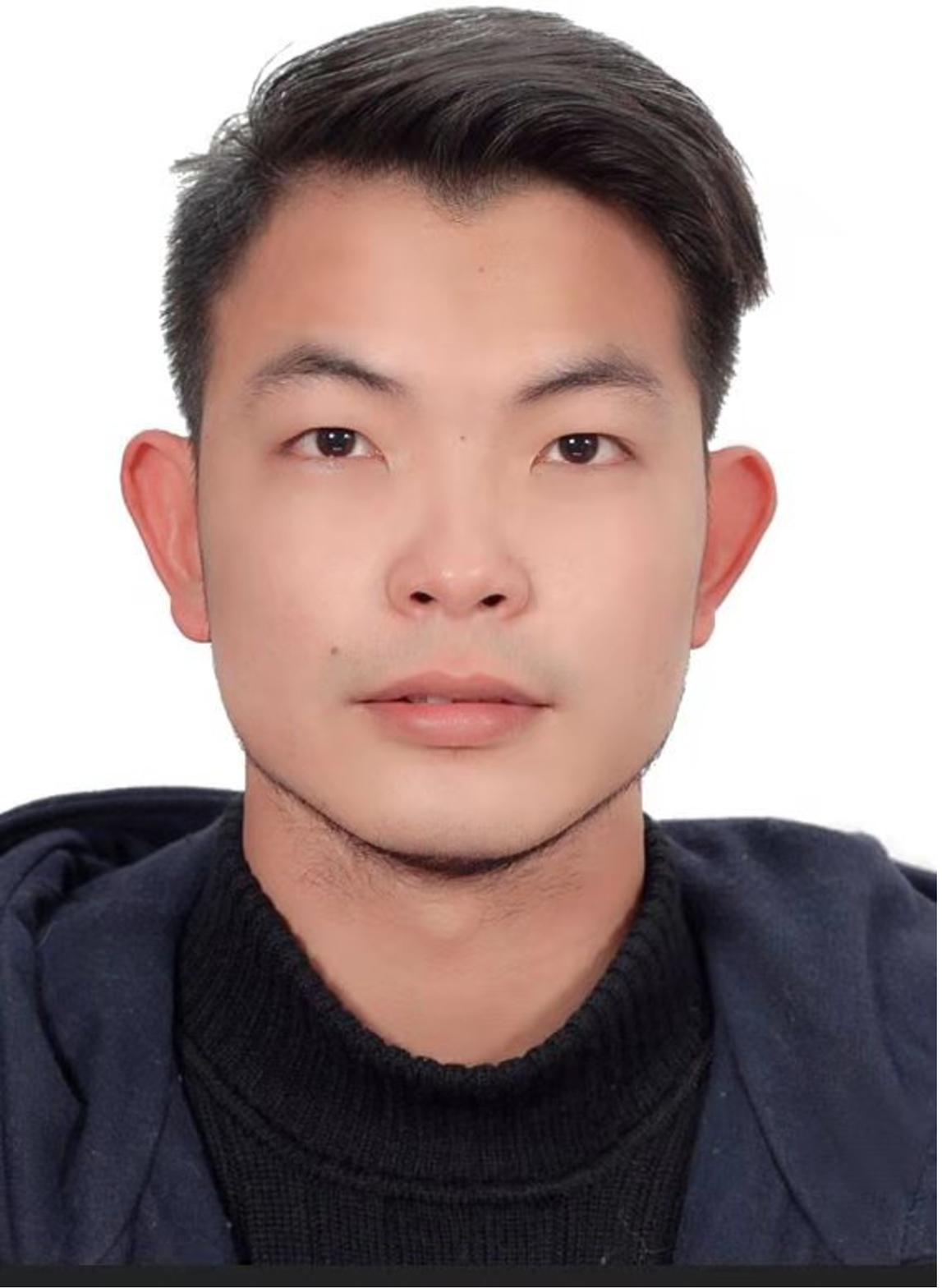}}]{Zhuqing Zhang} received his B.S from School of Astronautic, Northwestern Polytechnical University in 2017 and received his M.S. from School of Aeronautics and Astronautics, Shanghai Jiao Tong University in 2020. \\
\quad He is currently a Ph.D. candidate in the Department of Control Science and Engineering, Zhejiang University, Hangzhou, P.R. China. His current research interests include multi sensor fusion localization and SLAM.
\end{IEEEbiography}
\vspace{-5mm}
\begin{IEEEbiography}[{\includegraphics[width=1in,clip,keepaspectratio]{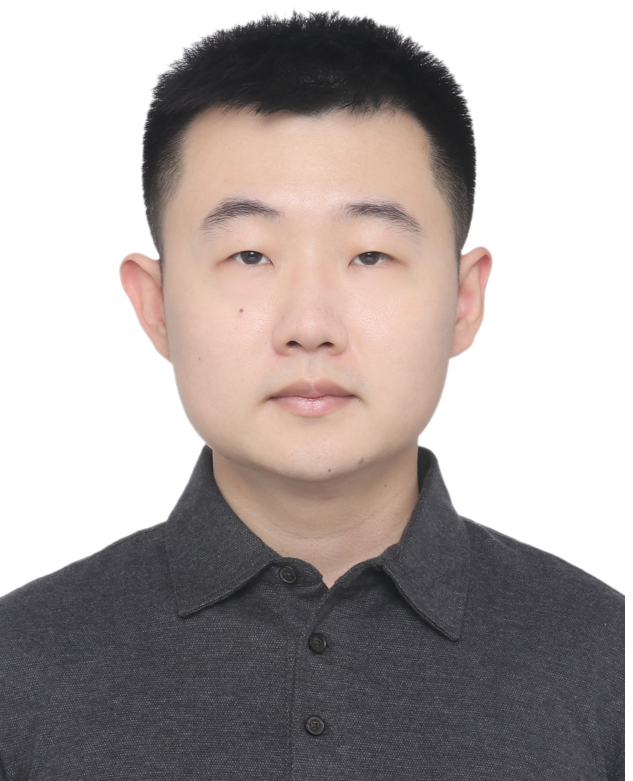}}]{Yang Song} received his B.S. and M.S. from School of Mathematics and Statistics, Beijing Institute of Technology in 2017 and 2020, respectively.\\
\quad He is currently a Ph.D. candidate in the Robotics Institute, University of Technology Sydney, Australia. His current research interests are EKF, SLAM and optimization.
\end{IEEEbiography}
\vspace{-5mm}
%\vspace{-0.05cm}

\begin{IEEEbiography}[{\includegraphics[width=1in,clip,keepaspectratio]{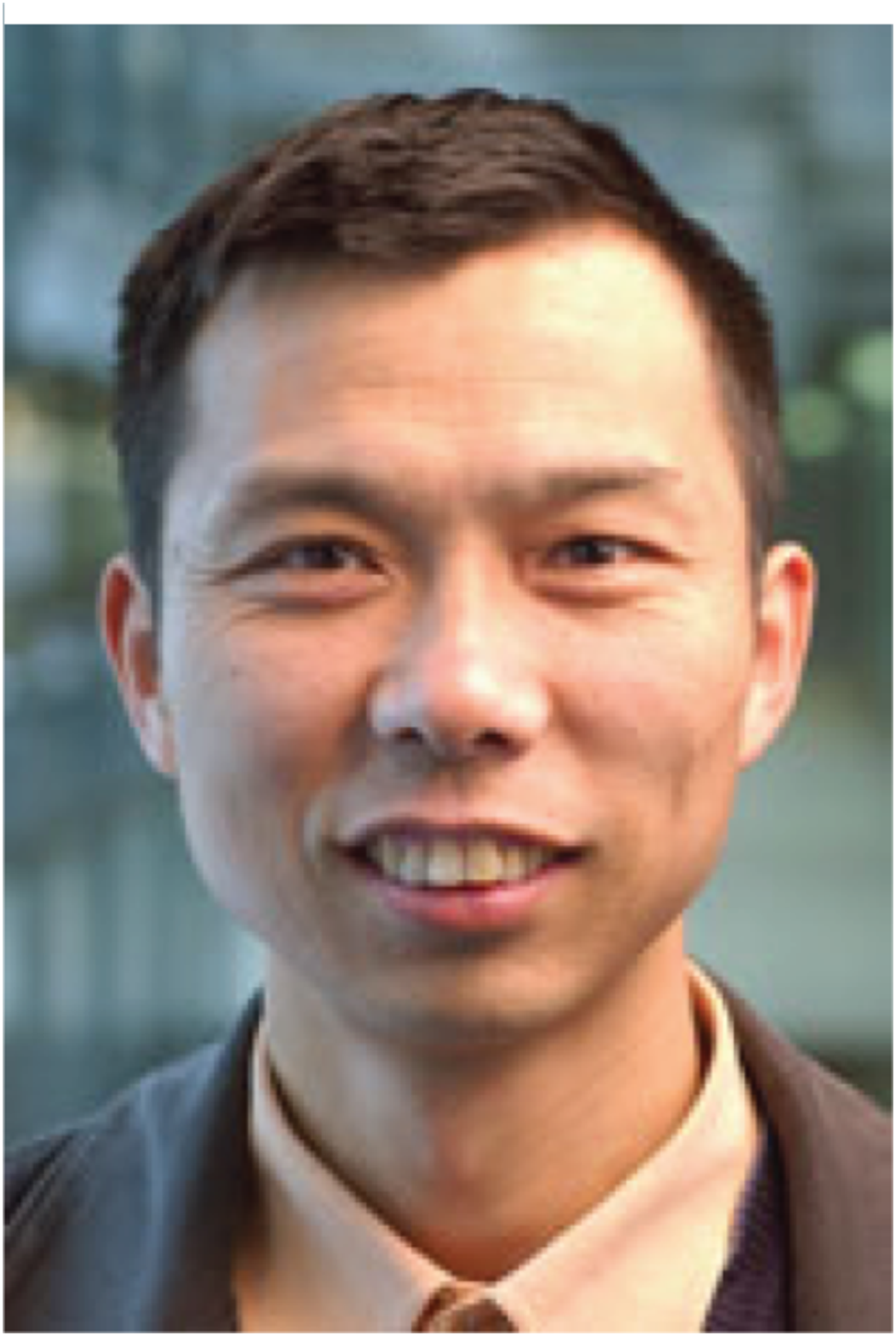}}]{Shoudong Huang} received the Bachelor and Master degrees in Mathematics, Ph.D. in Automatic Control from Northeastern University, PR China in 1987, 1990, and 1998, respectively. He is currently an Associate Professor in the Robotics Institute, University of Technology Sydney, Australia. His research interests include nonlinear control systems and mobile robots simultaneous localization and mapping (SLAM), exploration and navigation.
\end{IEEEbiography}
\vspace{-5mm}
%\vspace{-0.01cm}
\begin{IEEEbiography}[{\includegraphics[width=1in,clip,keepaspectratio]{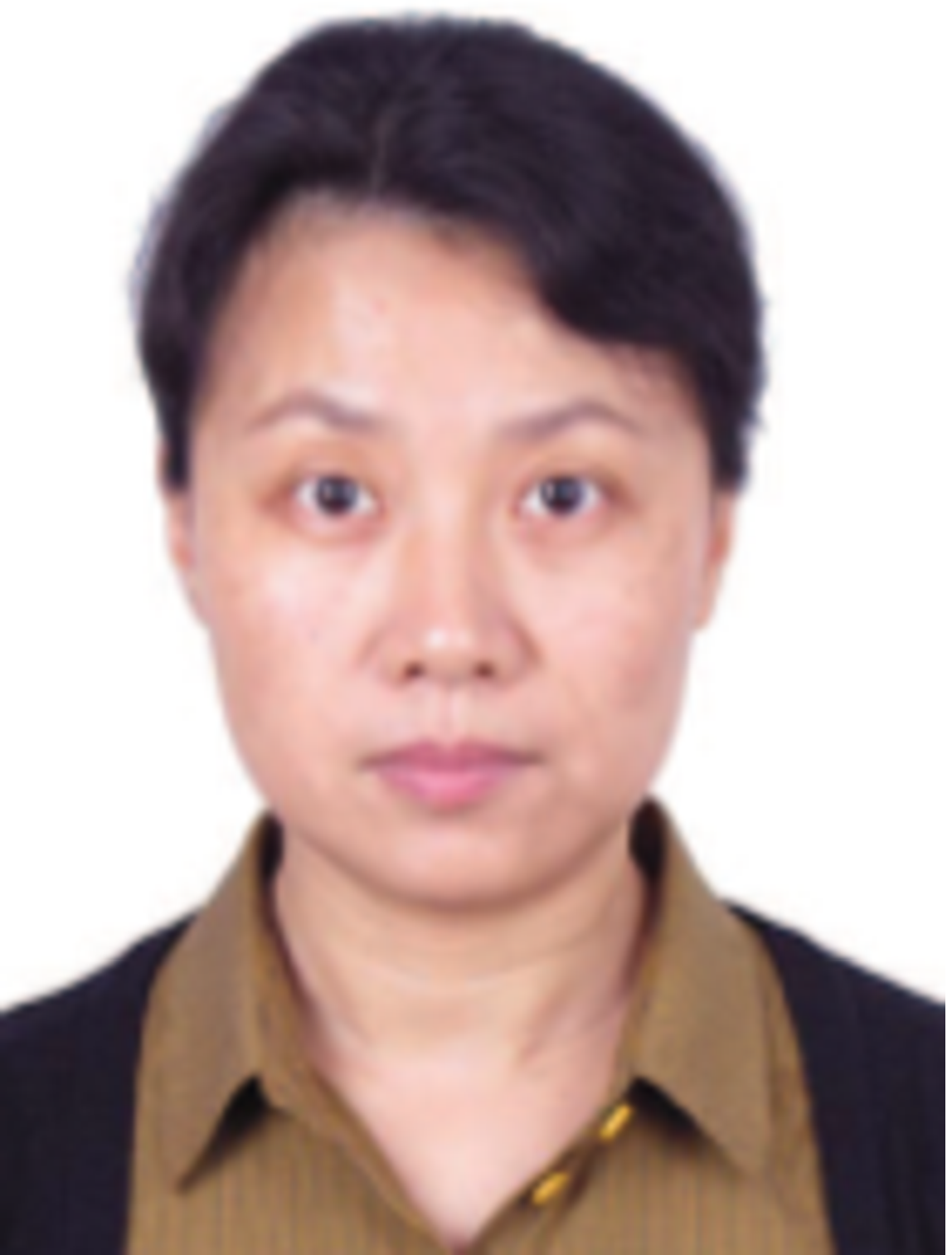}}]{Rong Xiong} received her PhD in Control Science and Engineering from the Department of Control Science and Engineering, Zhejiang University, Hangzhou, P.R. China in 2009. \\
\quad She is currently a Professor in the Department of Control Science and Engineering, Zhejiang University, Hangzhou, P.R. China. Her latest research interests include motion planning and SLAM.
\end{IEEEbiography}
\vspace{-5mm}
\begin{IEEEbiography}[{\includegraphics[width=1in,clip,keepaspectratio]{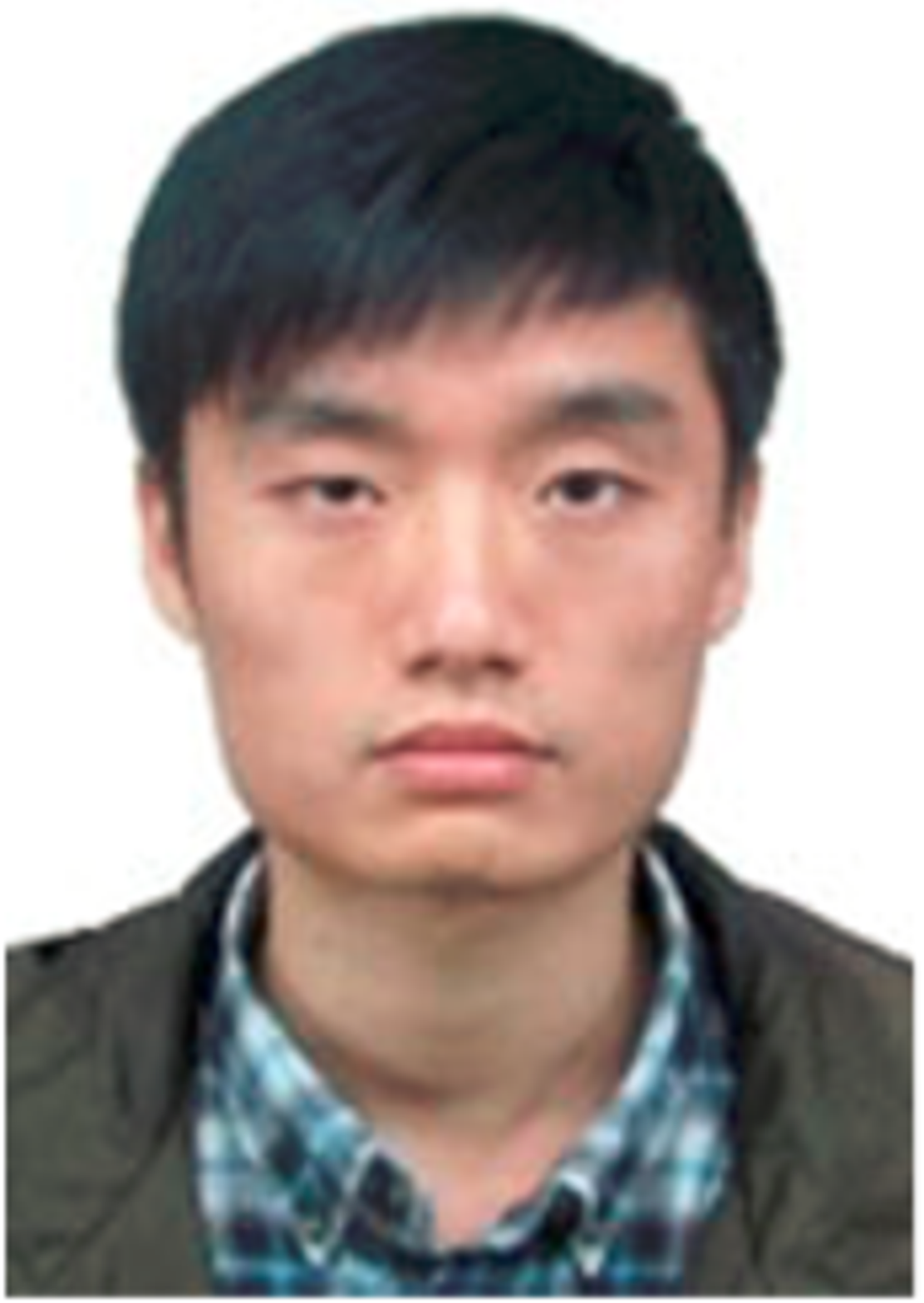}}]{Yue Wang} received his PhD in Control Science and Engineering from Department of Control Science and Engineering, Zhejiang University, Hangzhou, P.R. China in 2016. \\
\quad He is currently an Associate Professor in the Department of Control Science and Engineering, Zhejiang University, Hangzhou, P.R. China. His latest research interests include mobile robotics and robot perception.
\end{IEEEbiography}

%\newpage



%\vfill


\end{document}